\documentclass[aihp]{imsart}

\RequirePackage{amsthm,amsmath,amsfonts,amssymb}
\RequirePackage[numbers]{natbib}

\usepackage[backref=page]{hyperref}

\startlocaldefs

\theoremstyle{plain}
\newtheorem{hyp}{Assumption}[section]
\newtheorem{theorem}{Theorem}[section]
\newtheorem{prop}[theorem]{Proposition}

\newtheorem{lem}[theorem]{Lemma}
\newtheorem{corollary}[theorem]{Corollary}

\theoremstyle{remark}

\newtheorem{example}[theorem]{Example}
\newtheorem{remark}[theorem]{Remark}
\newtheorem{rem}[theorem]{Remark}

\usepackage[shortlabels]{enumitem}
\newlist{propenum}{enumerate}{1} 
\setlist[propenum]{label=(\roman*)}

\newcommand{\ca}{{\mathcal A}}

\newcommand{\cc}{{\mathcal C}}

\newcommand{\ck}{{\mathcal K}}
\newcommand{\cK}{{\mathcal K}}

\newcommand{\cq}{{\mathcal Q}}

\newcommand{\sparse}{{ s }}

\newcommand{\E}{{\mathbb E}}

\newcommand{\N}{{\mathbb N}}
\renewcommand{\P}{{\mathbb P}}

\newcommand{\R}{{\mathbb R}}

\newcommand{\Z}{{\mathbb Z}}

\newcommand{\rd}{{\rm d}}

\newcommand{\Leb}{{\rm Leb}}

\newcommand{\ind}{{\bf 1}}
\newcommand{\scale}{{\overline{\sigma}}}

\newcommand{\Supp}{{\rm Supp}\;}

\newcommand{\inv}[1]{\mathop{\frac{1}{ #1}}\nolimits}
\newcommand{\expp}[1]{\mathop {\mathrm{e}^{ #1}}}

\newcommand{\dT}{\mathfrak{d}_T}
\newcommand{\dTh}{\mathfrak{d}_T^{\rm h}}
\newcommand{\dI}{\mathfrak{d}_\infty}
\newcommand{\dK}{\mathfrak{d}_\cK}
\newcommand{\dKT}{\mathfrak{d}_{\cK_T}}
\newcommand{\DT}{\mathcal{V}_T}
\newcommand{\RT}{\rho_T}
\newcommand{\tD}{\tilde D}
\newcommand{\op}{\mathrm{op}}
\newcommand{\ve}{{\ell_\infty}}
\newcommand{\El}{e_\ell}

\newcommand{\Var}{{\rm Var}}
\newcommand{\supp}{{\mathrm{Supp}}}
\newcommand{\argmax}{\mathop{\mathrm{argmax}}}
\newcommand{\coeff}{{\xi}}
\newcommand{\norm}[1]{{\left\lVert #1 \right\rVert}}

\newcommand{\Put}{with }

\endlocaldefs

\begin{document}

\begin{frontmatter}

\title{Off-the-grid learning of mixtures from a continuous dictionary}
\runtitle{Off-the-grid learning of sparse mixtures from a continuous dictionary}

\begin{aug}
\author[A]{\inits{C.}\fnms{Cristina}~\snm{Butucea}\ead[label=e1]{cristina.butucea@ensae.fr}},
\author[B]{\inits{J-F.}\fnms{Jean-François}~\snm{Delmas}\ead[label=e2]{jean-francois.delmas@enpc.fr}}
\author[C]{\inits{A.}\fnms{Anne}~\snm{Dutfoy}\ead[label=e3]{anne.dutfoy@edf.fr}}
\and
\author[B,C]{\inits{C.}\fnms{Cl\'ement}~\snm{Hardy}\ead[label=e4]{clement.hardy@enpc.fr}}
\address[A]{CREST, ENSAE, IP Paris, France\printead[presep={,\ }]{e1}}

\address[B]{CERMICS, \'{E}cole des Ponts, France\printead[presep={,\ }]{e2,e4}}
\address[C]{EDF R\&D, Palaiseau, France\printead[presep={,\ }]{e3}}

\end{aug}

\begin{abstract}
	We consider a  general non-linear model where the signal  is a finite
	mixture of an unknown, possibly increasing, number of features issued
	from  a  continuous dictionary  parameterized  by  a real  non-linear
	parameter.    The  signal   is  observed   with  Gaussian   (possibly
	correlated) noise  in either  a continuous or  a discrete  setup.  We
	propose an off-the-grid optimization method, that is, a method which
	does not  use any  discretization scheme on  the parameter  space, to
	estimate  both the  non-linear  parameters of  the  features and  the
	linear parameters of the mixture.

	We use  recent results on  the geometry of off-the-grid  methods
	to give minimal  separation on the true underlying non-linear parameters  such that interpolating  certificate functions
	can be  constructed. Using also  tail bounds for suprema  of Gaussian
	processes we  bound the  prediction error  with high  probability. Assuming that the certificate functions can be constructed, our prediction error bound is up  to $\log$-factors similar to the rates  attained by the  Lasso predictor  in the linear  regression model.  We  also establish
	convergence rates that quantify with  high probability the quality of
	estimation for both the linear and the non-linear parameters.
 
 We develop in full details  our main results for two applications: the Gaussian spike deconvolution and the scaled exponential model. 
	
\end{abstract}




\end{frontmatter}


\section{Introduction}

\subsection{Model and method}

Assume  we observe  a random  element  $y$ of  an Hilbert  space and  we
consider a  signal-plus-noise structure  for the observation  $y$, where
the noise is  distributed according to a centered  Gaussian process. The
signal is modeled as a mixture model, by a linear combination of at most
$K$   features of the form  $\varphi(\theta)$ for 
some parameters  $\theta \in \Theta$,  where $\Theta\subseteq \R$  is an
interval of  parameters and  $\varphi$ is a  smooth function  defined on
$\Theta$  and  taking  values  in  the  Hilbert  space.   We  denote  by
$(\varphi(\theta), \theta \in \Theta)$ the continuous
dictionary.

In order to  capture a great variety of examples,  we shall assume there
exists  a  Hilbert  space  $H_T$,   endowed  with  the  scalar  product
$\langle \cdot, \cdot \rangle_T$ and the norm $\norm{\cdot}_T$, where $T$ is
a  parameter  belonging  to $\N$, 
such  that:  the
observed process $y$ belongs to $H_T$;  for all $\theta \in \Theta$, the
feature $\varphi_T(\theta)$ (which may depend  on $T$) belongs to $H_T $
and  is non  degenerate, {\it  i.e.} $  \norm{\varphi_T (\theta)}_T$  is
finite and non zero; the noise process $w_T$, which might also depend on
the  parameter  $T$   is  a  centered  Gaussian   process  belonging  to
$H_T$.

We consider the model with unknown parameters $\beta^{\star}$ in $\R^K$ and $\vartheta^{\star}$ in $\Theta^K$:
\begin{equation}
\label{eq:model}
y = \beta^{\star}\Phi_{T}(\vartheta^{\star}) + w_T \quad  \text{in $H_T$},
\end{equation}
where the multivariate
function $\Phi_T$ is defined on $\Theta^K$ by:
\[
\Phi_{T}(\vartheta)= (	\phi_{T}(\theta_{1}), \ldots,
\phi_{T}(\theta_{K}) )^\top 
\quad\text{for}\quad
\vartheta = \left ( \theta_1,\cdots,\theta_K\right ) \in \Theta^K
\]
and the function $\phi_T$ defined on $\Theta$ is the normalized feature $\varphi_T(\theta)$ : 
\begin{equation}
\label{eq:def-phi_T}
{\phi_{T}(\theta)=\frac{\varphi_T(\theta)}{\norm{\varphi_T(\theta)}_T}}\cdot
\end{equation} 

We assume from now on that the unknown $K$ dimensional vector $\beta^{\star}$ is sparse, \emph{i.e} it has $\sparse$ non zero entries or, equivalently, $\beta^{\star} \in \mathcal{B}_0({\sparse})=  \left \{ \beta\in \mathbb{R}^{K},  \, \|\beta\|_{\ell_0} = \sparse \right  \}$, where $\|\beta\|_{\ell_0}$ counts the number  of non zero entries of the vector $\beta$. Let $S^{\star}$ be the support of $\beta^{\star}$:
\begin{equation*}
S^{\star} = \supp(\beta^\star)=\{k \in \{1,\cdots,K\}, \, \beta^{\star}_k \neq 0 \},
\end{equation*} 
and call $\sparse = \operatorname{Card}  S^\star$ the sparsity parameter.
We  are interested  in  predicting observations  and  in recovering  the
unknown parameters. Let us denote in  general by $u_S$ the vector $u$ in
$\R^K$  restricted to  the  coordinates  in $S$  for  any non-empty  set
$S\subseteq\{1,...,K\}$.      We     estimate    both     the     vector
$\beta^\star_{S^\star}$ with  unknown  $\sparse$ and  the vector
$\vartheta^\star_{S^\star}   $ with entries in some compact set $\Theta_T$  containing  the  parameters  of  those  functions  from  our  continuous
dictionary that  appear in the  mixture model.  Note that  when applying
the  same permutation  on  the coordinates  of  $\beta^{\star}$ and  the
coordinates of $\vartheta^{\star}$, we obtain  the same model. Thus, the
vectors $\beta^{\star}$ and  $\vartheta^{\star}$ are defined
up    to    such   a    joint    permutation.     Moreover,   we    have
$\beta^{\star}\Phi_{T}(\vartheta^{\star})                              =
\beta^{\star}_{S^{\star}}\Phi_{T}(\vartheta^{\star})_{S^{\star}}$,
where,                           by                          definition,
$\Phi_{T}(\vartheta^{\star})_{S^{\star}}                               =
\Phi_{T}(\vartheta^{\star}_{S^{\star}})$.   Our  model   is  linear  and
sparse  in  $\beta^\star$ but  it  is  non-linear in  $\vartheta^\star$.
\medskip

We make the following  assumption on the noise  process $w_T$, where the decay rate  $\Delta_T>0$ controls the noise variance decay as the parameter $T$ grows and $\sigma>0$ is the
intrinsic noise level.
\begin{hyp}[Admissible  noise] \label{hyp:bruit}  Let $T  \in \N$.   The
	noise process  $w_T$ belongs to  $H_T$ a.s.,  and there exist  a noise
	level  $\sigma>0$ and  a decay  rate  $\Delta_T>0$ such  that for  all
	$f\in  H_T$,  the  random  variable $\langle  f,w_T  \rangle_T$  is  a
	centered Gaussian random variable satisfying:
	\begin{equation}
	\Var \left( \langle f,w_T  \rangle_T \right)\leq \sigma^2 \,
	\Delta_T\,  \norm{f}_T^2.
	\end{equation}
\end{hyp}
In  our model,  the parameter  $T$ may  be understood  as the  amount of
information that we have  on the underlying signal.

\medskip

In order to recover the sparse vector $\beta^{\star}$ as well as the associated parameters $\vartheta^\star_{S^{\star}}$ (up to a permutation), we solve the following regularized optimization problem with a real tuning parameter $\kappa>0$:
\begin{equation}
\label{eq:generalized_lasso}
(\hat{\beta},\hat{\vartheta}) \in \underset{\beta \in
	\mathbb{R}^{K}, \vartheta \in \Theta_{T}^K}{\text{argmin}} \quad
\frac{1}{2}\norm{y - \beta\Phi_{T}(\vartheta)}_T^{2} +\kappa
\norm{\beta}_{\ell_1}, 
\end{equation}
where the smooth function $\Phi_T$ is defined on the set $\Theta^K_T$, with $\Theta_T$ a compact interval. Therefore
the existence of at least a  solution is guaranteed. The functional that
we minimize in  this problem is composed  of a data fidelity  term and a
penalty  term. The  penalty is  expressed  with a  $\ell_1$-norm on  the
vector  $\beta=(\beta_1, \ldots,  \beta_K)$, \emph{i.e}  the sum  of the
absolute          values           of          its          coordinates:
$\norm{\beta}_{\ell_1}=\sum_{i=1}^K  |\beta_i|$.    This  penalization  is
similar to that of the Lasso problem (also referred to as Basis pursuit)
introduced  in \cite{tibshirani1996regression}  and extensively  studied
since  then  (see   \cite{buhlmann2011statistics}  for  a  comprehensive
survey). The optimization  of the non-linear parameters  is not performed
on the whole  set of parameters $\Theta$ but rather  on a compact subset
$\Theta_T$ indexed by the parameter $T$.  Indeed, it may be necessary to
restrict the set of parameters, \emph{e.g.}  in a finite mixture model where we
consider  a location  parameter  we can  only  recover those  parameters
within the support of the observations.

In the more  general Beurling Lasso (BLasso) framework,  one can rewrite
the  problem~\eqref{eq:generalized_lasso}  in  a  measure  setting.  The
actual                                                          solution
$\big(\hat   \beta=(\hat   \beta_1,    \ldots,   \hat   \beta_K),   \hat
\vartheta=(\hat      \theta_1,     \ldots,      \hat     \theta_K)\big)$
of~\eqref{eq:generalized_lasso}  is  then  seen as  the  atomic  measure
$\hat \mu=\sum_{k=1}^K  \hat \beta_k \,\delta _{\hat  \theta_k}$, where
the  amplitudes  and  the  locations  of  the  Dirac  masses  correspond
respectively  to  the  linear  coefficients   in  the  mixture  and  the
parameters of the features. The measure $\hat \mu$ is also a solution of
the BLasso problem  when the latter admits atomic  solutions composed of
less  than $K$  atoms.  This is  in  particular the  case  in the discrete-time model, with $T$ design points, 
presented  in  Section~\ref{sec:exple-lambda}   where  $K\geq  T$ according  to~\cite{boyer2019representer}.   However, to the best of our
knowledge, there  are no such  results when  $H_T$ is a  general Hilbert
space.

\bigskip
\subsection{Examples}
\label{sec:examples}
In this section we give examples of both discrete and continuous-time models that are covered by our general setup. We discuss how $T$ indicates the amount of information that the data contain on the unknown underlying signal. Indeed, in  the discrete  case, the amount  of information grows as  the number $T$ of the  design points over which the process is observed increases, while the largest step-size decreases; in the continuous case, it grows as the decay rate $\Delta_T$ of the noise variance decreases.

We emphasize the various structures of noise processes that are admissible by  giving several examples of discrete or continuous-time noise processes that satisfy our assumptions. They are frequently used in discrete regression models or continuous models like the Gaussian white noise model, see \cite{tsybakov} or \cite{GineNickl}.

\subsubsection{Discrete-time models} \label{sec:exple-lambda}


	Consider  a  real-valued process  $y$  observed  over the points 
	$t_1<\cdots< t_T$ on $[0, 1]$, with $T\in \N^*$.  Let $H_T=L^2(\lambda_T)$ be  the  Hilbert space of real  valued  functions  defined  on $[0,1]$  and
	square  integrable with  respect  to some  probability  measure
	$\lambda_T$         on         $\{t_1,         \ldots,         t_T\}$. 
 Let the  noise $w_T\in H_T$ be given by $ w_T(t)=\sum_{j=1}^T G_j \ind_{\{t_j\}}(t)$, where 	$G_1, \ldots, G_T$ are  centered   Gaussian   random    variables and $\ind_A$ 
	denotes the indicator function of an arbitrary set $A$. Thus, the observations are:
 \begin{align} \label{DTM}
 y\left( t_j \right) = \sum_{k \in S^\star} \beta^\star_k \cdot \phi_T \left(\theta_k^\star,  t_j \right) + G_j,\quad j=1,\ldots, T.
 \end{align}
 The risk is measured by:
 $$
\|y - \beta \Phi_T(\vartheta) \|_T^2 = \sum_{j=1}^T \left( y(t_j) - \sum_{k=1}^K \beta_k \cdot \phi_T \left(\theta_k,  t_j \right) \right)^2 \lambda_T(t_j).
 $$
Now,  let	$\lambda_T=\Delta_T \sum_{j=1}^T \delta_{t_j}$, where $\delta_x$ denotes the Dirac mass  at $x$. In the particular case where $\Delta_T = 1/T$, one can approximate the measure $\lambda_T$ for $T$ large by the Lebesgue measure  on $[0, 1]$, say  $\Leb$. In various examples,  it is
	also easier to compute  the norms of the features and  of their derivatives in  the  Hilbert  space  $L^2(\Leb)$. This amounts  to  seeing  $H_T$  as
	approximating Hilbert  spaces of the fixed  Hilbert space $L^2(\Leb)$.

Let us now see that, if  the noise variables $G_1,\ldots, G_T$ are independent   centered   Gaussian   random   variables   with   variance
$\sigma^2$, then   Assumption~\ref{hyp:bruit}  holds  with  an
equality:
\[
\Var( \langle f , w_T\rangle_T)=\sigma^2  \Delta_T \, \norm{f}^2_T  . 
\]
If  $(G_1,\ldots, G_T)$ is a centered Gaussian vector of dimension $T$ and covariance matrix  with each diagonal entry $\sigma^2$, then
Assumption~\ref{hyp:bruit} holds with  
$\Delta_T$ multiplied by the spectral radius $\varrho_T\in [1, T]$ of the  correlation  matrix:
\[
\Var( \langle f , w_T\rangle_T)\leq \sigma^2\,  \Delta_T\, \varrho_T\, \norm{f}^2_{T}. 
\]

\subsubsection{Continuous-time models with truncated white noise or colored noise}
\label{sec:continuous_noise}

Consider  the  set  $\cc=\cc([0,  1],  \R)$  of  $\R$-valued  continuous
functions defined on $[0,1]$, an orthonormal base $(\psi_j, j\in \N)$ of
$L^2=L^2([0,   1],   \mathrm{Leb})$   of  elements   of   $\cc$,   where
$\mathrm{Leb}$ is the Lebesgue measure on $[0, 1]$.  We simply denote by
$\langle \cdot,  \cdot \rangle_{L^2}$ the corresponding  scalar product.
Let $p=(p_j,  j\in \N)$ be a  sequence of non-negative real  numbers and
set $\supp (p)=\{j\in  \N\, \colon\ p_j>0\}$ its support.   Let $H_T$ be
the   completion   of  the   vector   space   generated  by   the   base
$(\psi_j, j\in \supp(p))$ (which is also  the completion of $\cc$ if $p$
is positive and bounded), with respect to the scalar product:
\[
\langle f,g \rangle_T = \sum_{j\in \N} p_j \, \langle f, \psi_j \rangle
_{L^2}\, \langle g, \psi_j \rangle_{L^2}. 
\]
Notice that the Hilbert space $H_T$ does not depend on the parameter $T$
unless $p$ depends on $T$. Let us recall that if $p\equiv 1$, that is, the
sequence  $p$  is  constant  equal   to  $1$, then 
$H_T=L^2$.  In this model we observe a continuous path:
\begin{align}\label{CTM}
y(t) = \sum_{k \in S^\star} \beta_k^\star \phi_T(\theta_k^\star, t) + w_T(t), \quad t \in [0,1].
\end{align}
The risk is measured by:
$$
\|y - \beta \Phi_T (\theta)\|_T^2 = \sum_{j \in \mathbb N} p_j \left( \int_0^1 (y(t) - \beta \Phi_T (\theta,t)) \cdot  \psi_j(t) \,  \rd t \right)^2.
$$

Let $\xi=(\xi_j,  j\in \N)$ be  a weight sequence of non-negative  real numbers
such that the sequence $p\circ \,\xi : =(p_j\, \xi_j, j\in \N)$ is summable.
Consider the  noise $w_T=\sum_{j \in \Supp(p)} \sqrt{\xi_j} \, G_j\,  \psi_j$, where
$(G_j, j\in \N)$ are independent centered Gaussian random variables with
variance              $\sigma^2$.               Notice
Assumption~\ref{hyp:bruit} holds as 
$\norm{w_T}^2_T=\sum_{j\in  \N}  p_j \,\xi_j  \,  G_j^2$ is  a.s. finite
and, with $\Delta_T=  \sup_\N p \circ \, \xi $:
\[
\Var( \langle f , w_T\rangle_T)=\sigma^2 \sum_{j\in \N} p_j^2\,\xi_j\,
\langle f, \psi_j \rangle_{L^2}^2  
\leq \sigma^2 \,\Delta_T\,   \norm{f}^2_T.
\]
Notice that the noise $w_T$ does  not depend on the parameter $T$ unless
$p$ or $\xi$ depends on $T$.
\medskip

The  truncated  white  noise  model   corresponds  to  $p\equiv  1$  and
$\xi=(\xi_j=\ind_{\{j\leq T\}}, j\in \N)$. In this case $\Delta_T=1$ and
$\norm{w_T}_T^2$ is a.s.  of order $\sigma^2\,T$  by the strong law  of large
numbers. The white noise corresponds to the limit case $T=+\infty $, which
does  not   satisfy  the   hypothesis  as   a.s.  its   $L^2$-norm  is
infinite. Let us  mention that the bounds given in  the main theorems in
Section~\ref{sec:main_result} rely on  $\norm{w_T}_T$ being finite and
not on its value.

Consider again $p\equiv 1$. Thanks to the Karhunen-Loève's decomposition, the scaled Brownian motion
$w_T= C_T \,  B$, with $B$ the Brownian motion  on $[0, 1]$ and $C_T$ a
positive constant, corresponds
to                  the        orthonormal           base                  functions
$\psi_k(t)    =   \sqrt{2}\,    \sin\left((2k+1)\pi   t/2\right)$    for
$t\in [0,  1]$ and the  weights $  \xi_k = 4C_T^2  / (2k +  1)^2\pi^2$ for
$k\in    \N$,   and    $\sigma^2=1$.    In   this    case,    we    have
$\langle  f,   w_T  \rangle_T= C_T  \int_0^1   f(s)  B(s)  \,  \rd   s$ for
$f\in L^2$  and
Assumption~\ref{hyp:bruit} holds with $\sigma^2=1$ and $\Delta_T=\sup_\N
p\circ \, \xi= 4C_T^2/\pi^2$.

\subsection{Previous work}

The   model    (\ref{eq:model})   in    the   particular    case   where
$\vartheta^\star$ is supposed given and the observations depend linearly
on a  vector $\beta^{\star}$  has long been  studied in  the literature.
Assume for simplicity  that $H_T = \R^T$ is the $T$-dimensional
Euclidean space, so that 
$\Phi_T\in \R^{K\times T}$  is a matrix whose entries are  known and can
be either random or deterministic, $y  \in \R^{T}$ is an observed vector
and $w_T  \in \R^{T}$  is a  vector of  noise (often  assumed to be Gaussian).
Even when  $K$ is larger than  $T$ the estimation of  $\beta^{\star}$ is
still consistent  provided the  vector $\beta^{\star}$  is sparse  and a
null  space  property  is  verified  by the  matrix  $\Phi_T$,  or  some
sufficient  condition saying  that the  lines  of $\Phi_T$  are not  too
colinear (see  \cite{van2016estimation} for  a complete  overview).  The
Lasso estimator \cite{tibshirani1996regression}  or the Dantzig selector
\cite{candes2007dantzig} are efficient to perfom such estimation and the
quality of the  estimation with respect to the dimension  of the problem
is  now well  known. The  authors of  \cite{bickel2009simultaneous} have
given bounds for the prediction error for both estimators.

We  consider  here  a  highly non-linear extension  of  this  model  that
consists in assuming that  the matrix $\Phi_T = \Phi_T(\vartheta^\star)$
depends   non-linearly  on   a  parameter   $\vartheta^{\star}$  to   be
estimated. In  our model \eqref{eq:model},  $\Phi_T$ is composed  of $K$
row  vectors  belonging  to  a  parametric family  or  by  $K$  features
belonging to  a continuous dictionary and  the observed data $y$  may be
either a vector or  a function. This model has proven  to be relevant in
many  fields such  as  microscopy, astronomy,  spectroscopy, imaging  or
signal processing.

When the observation  $y$ belongs to a  finite-dimensional Hilbert space
and the  dimension $K$  is fixed  and small compared  to $T$,  the model
received attention several decades ago  and gave rise to separable least
square problems and resolution methods  such as variable projection (see
\cite{kaufman1975variable,golub1973differentiation}).     These   papers
mainly provided  numerical methods  but let  us mention  the consistency
result in  \cite{kneip1988convergence} for non-linear  regression models.
\medskip

On  the contrary,  when $K$  is arbitrarily  large many  problems remain
open. One  of the  natural ideas to  estimate the  underlying parameters
could be  to discretize the parameter  space $\Theta$ and return  to the
study of a linear model. It would amount to considering a finite subfamily
of $(\varphi(\theta), \, \theta \in \Theta)$ as in \cite{tang2013sparse}
and deal with overcomplete dictionary learning techniques (also referred
to as sparse  coding, see \cite{olshausen1997sparse, donoho2005stable}).
In this case, sparse estimators for  linear models such as the Lasso are
available.   However, in  sparse  spike deconvolution  where the  family
$(\varphi(\theta),   \theta  \in   \Theta)$  is   a  family   of  spikes
parametrized    by    a    location   parameter,    the    authors    of
\cite{duval2017thingrid}  have  shown  that  in the  presence  of  noise
discretizing the space  of parameters and solving a  Lasso problem tends
to  produce  clusters   of  spikes  around  the  spikes   one  seeks  to
locate. That  is why it  is preferable  to use off-the-grid  methods. By
off-the-grid,   we  mean   that  the   methods  employed   do  not   use
discretization   schemes   on   the    parameter   set   $\Theta$.    In
\cite{duval2015exact},  the authors  show that  in presence  of a  small
noise,  the BLasso only  induces a  slight perturbation  of the
spikes  locations and  amplitudes and  does not  produce clusters.   The
BLasso  was  introduced in \cite{de2012exact} and has
been studied in many papers since  then mostly by the compressed sensing
and      super-resolution      communities      (see \cite{candes2013super},
\cite{azais2015spike}   among  many   others).   It   is  basically   an
off-the-grid extension of the  classical Lasso for continuous dictionary
learning.   The   optimization  problem   is  formulated  as   a  convex
minimization over the space of  Radon measures. In the BLasso framework,
the  dimension  $K$  in  \eqref{eq:model} is  infinite  and  the  linear
coefficients and non-linear  parameters are encoded by  an atomic measure
made of weighted Dirac functions. By solving a minimization problem over
Radon measures, the  aim is to recover an atomic  measure. It raises the
question     of    whether     such    a     solution    exists.      In
\cite{boyer2019representer} the question is  answered by the affirmative
when the observed data $y$ belongs to a finite-dimensional Hilbert space
$H_T$. When  this is not the  case, i.e. $H_T$ is  infinite dimensional,
the question is open. In this paper,  we avoid the problem by assuming a
bound $K$ on the number of  functions in the mixture and restricting the
space over which  the BLasso is perfomed to the  atomic measures with at
most $K$ atoms.  The numerical methods  used to solve the BLasso such as
the  Sliding  Frank-Wolfe   algorithm  (see \cite{denoyelle2019sliding}  and
\cite{butucea2021,golbabaee2020off} for applications in spectroscopy and
imaging),  also  called  the alternating  descent  conditional  gradient
method  (see \cite{boyd2017alternating}), and  the  conic particle  gradient
descent (see \cite{chizat2021sparse}), seek a solution directly in the space
of Dirac mixtures.   Hence, our formulation \eqref{eq:generalized_lasso}
is  closer to  the way  algorithms proceed.  Let us  mention that  other
methods such  as Orthogonal Matching Pursuit  (see \cite{elvira21}) exist to
tackle   the   problem   of    sparse   learning   from   a   continuous
dictionary. Typically, the case of  sparse spike deconvolution where the
dictionary consists of Gaussian functions continuously parametrized by a
location parameter is not included.
\medskip

The  study  of  the  regression  over a  continuous  dictionary  in  the
framework  of the  BLasso  has  been quite  specific  to the  dictionary
considered. The  literature first focused  on the dictionary  of complex
exponential     functions     parametrized    by     their     frequency
$(\varphi(\theta): t  \mapsto \expp{i  2\pi \langle t,  \theta \rangle},
\theta \in  \Theta )$ where  $\Theta$ is the $d$-dimensional  torus (see
\cite{candes2014towards}).   In  \cite{boyer2017adapting},  a  bound  is
given for the prediction error for this dictionary.  The proof extends a
previous  result   obtained  in  \cite{tang2014near}  for   atomic  norm
denoising.  What is particularly interesting  is that the rates obtained
for the prediction  error almost reach the minimax  rates achievable for
linear models  (see \cite{MR2882274, candes2013well}) provided  that the
frequencies  are  sufficiently   separated.   The  separation  condition
between the non-linear parameters to  estimate is inherent to the BLasso
unless  we  assume  the  positivity  of  the  linear  parameters  as  in
\cite{schiebinger2018superresolution}.

For results on a wider range  of dictionaries, let us highlight the work
of  \cite{duval2015exact} that  gives recovery  and robustness  to noise
results for spike deconvolution. Let us  also mention the recent work of
\cite{bernstein2020sparse} that generalizes  some exact recovery results
for   a  broader   family  of   dictionaries  as   well  as   the  paper
\cite{bernstein2019deconvolution}   that  gives   robustness  to   noise
guarantees      for     a      family      of     shifted      functions
$(\varphi(\theta)=k(\cdot  - \theta),  \theta  \in \Theta)$  of a  given
specific function $k$.  In a density  model that is a mixture of shifted
functions, \cite{de2019supermix}  studies a modification of  the BLasso by
considering a weighted $L^2$ prediction error.  

The case of non-translation invariant families remained for long
intractable without very pessimistic separation conditions. In
\cite{poon2018geometry} the authors set a natural geometric framework to
analyse the estimation problem. The separation condition between the
parameters appears naturally in terms of a  metric. In their
paper, the design over which the observation are made is distributed
according to a probability distribution. Their main result shows that in
presence of noise the BLasso recovers a measure close to the one to be
estimated with respect to a Wasserstein metric.

\subsection{Contributions}
This  paper addresses  the problem  of  learning sparse  mixtures from  a
continuous dictionary for  a wide variety of regression  models within a
common  framework.   Indeed,  we  tackle   a  wide  range   of  possible
dictionaries of  sufficiently smooth  features, observation  schemes and
Gaussian noises  with various structures. The  observations are supposed
to  belong to  a Hilbert  space $H_T$.  Continuous observations  over an
interval of $\R$ as well as discrete observations at given design points
are  therefore  included  in  our  framework.  Furthermore,  the  Hilbert
structure and the mild assumption we make on the noise, encompass a wide
range of  Gaussian noises. In  particular, our framework allows  to take
into account the case of correlated Gaussian noise processes.

\medskip

The main results of this paper gives a high-probability bound for
the prediction error: 
\[
\norm{\hat{\beta}\Phi_{T}(\hat{\vartheta}) -
	\beta^{\star}\Phi_{T} (\vartheta^{\star}) }_T,
\]
where   $(\hat{\beta},  \hat{\vartheta})$   is  the   solution  of   the
optimization  problem  (\ref{eq:generalized_lasso}).   Contrary  to  the
BLasso optimization program over a set of measures whose result can be  a diffuse measure, our formulation of the optimization problem
has  always  a  solution  belonging  to a  finite  set  of  values.  Our
prediction error bound matches (up to logarithmic factors and with high probability) that obtained
in the linear case, that is when $\vartheta^\star$ is known and does not
need to  be estimated. We  also give high-probability bounds on  some loss functions comparing the  estimators   $\hat{\beta} \text{ and } \hat{\vartheta}$  given  by
(\ref{eq:generalized_lasso})  to  the  parameters   $\beta^\star$  and
$\vartheta^\star$, respectively. Our work extends results  that were so far restricted
to the specific case of  a dictionary consisting of complex exponentials
continuously     parameterized      by     their      frequencies     (see
\cite{boyer2017adapting, tang2014near}).  When the  optimization problem
produces a cluster of features  to approximate an element of
the mixture, we also show that  there can be no compensation between the
amplitudes of the features  involved.
\medskip

Following works  in   compressed  sensing   and  super-resolution
(see \cite{candes2014towards,candes2013super} among others), our bounds rely
on the existence of  interpolating functions called ``certificates" (see
Assumptions \ref{assumption1} and  \ref{assumption2}) instead of relying
on compatibility conditions or Restricted Eigenvalue conditions. We give in Section~\ref{sec:conditons_certificates} sufficient conditions for the existence of certificates and 
an  explicit  way   to  construct  such  functions  in   the  spirit  of
\cite{poon2018geometry}. We show  in this paper that  such functions can
be constructed provided the  non-linear parameters belonging to $\Theta$
are well separated with respect  to a Riemannian metric $\dT$ (defined
in     Section~\ref{sec:riemann})    associated     to    the     kernel
$\cK_T(\theta,\theta')  =  \langle  \phi_{T}(\theta),  \phi_{T}(\theta')
\rangle_T$.  This minimal separation distance between the non-linear parameters needs to be rather large, comparable to $s$, in a general context. However, it can be significantly reduced to a constant order in more particular cases such as the sparse spike deconvolution, see Remark~\ref{rem:sep-gauss}.
The  Riemannian metric appears  naturally when it  comes to
tackle a wide variety of dictionaries. In addition, it leads to a lot of
invariances  in many  quantities useful  in the  proofs. Typically,  the
Riemannian  metrics $\dT$  and $\dTh$  associated respectively  to the
kernel     $\ck_T(\cdot,    \cdot)$     and     the    warped     kernel
$\cK_T^{\rm h}  = \ck_T({\rm h}(\cdot),{\rm h}(\cdot))$  for some smooth
enough     diffeomorphism    $\rm h$     are    equal     and    we     have
$\dT(\theta,\theta')       =        \dTh({\rm       h}^{-1}(\theta),{\rm
	h}^{-1}(\theta'))$.

{Our statistical results  rely on  tail bounds for suprema of  Gaussian processes: following \cite{boyer2017adapting}, instead
	of  using controls on  $\norm{w_T}_T$ as in the seminal works \cite{duval2015exact, poon2018geometry}, we
	used bounds,  based on the noise structure from Assumption~\ref{hyp:bruit},  on quantities of the form
	$\sup_{\Theta_T} \left  \langle f(\theta),w_T\right \rangle_T$  for some
	$H_T$-valued    functions     $f$    built    from     the    dictionary
	$(\varphi_T(\theta), \theta\in \Theta) $ and its derivative. This approach is relevant as for some models  the quantity $\norm{w_T}_T$ may be very large, see for example the 
	truncated white noise model from  Section~\ref{sec:continuous_noise}.} We note that the nonlinear parameter $\theta$ is univariate in our setup. Generalization to multivariate non-linear parameters is possible, but highly technical. Indeed, the construction of the certificates holds in the multivariate setting,  but the exponential bounds for suprema of Gaussian fields are less precise concerning their dependence on the dimension.

We give next two applications of our results respectively to the Gaussian sparse spike deconvolution and to the Scaled exponential model also known as Laplace transform inversion. They illustrate how the stringent assumptions in all generality, become less restrictive in more precise setups. The full derivation of these examples can be found in Sections~\ref{sec:example} and~\ref{sec:scaled}, respectively.
 
\subsubsection{Gaussian sparse spike deconvolution, see Section~\ref{sec:example}.}
Consider the discrete-time model \eqref{DTM} as described in section \ref{sec:exple-lambda}, where a process $y$ is observed over a regular grid 
	$t_1 < \cdots < t_T$ on the interval $[a_T, b_T]$ with step size $\Delta_T=(b_T-a_T)/T$, where \(T \in \mathbb{N}^*\), \(b_T = -a_T = \sigma_0\sqrt{\log(T)}\) and  $\sigma_0>0$ is some fixed scale factor. Assume the observations are corrupted by independent centered Gaussian random variables of variance $\sigma^2$. 
	
	   The dictionary  consists of Gaussian spikes that are continuously translated:
	\[
	\left (\varphi(\theta) = \exp\left(-\frac{(\theta - \cdot)^2}{2\sigma_0^2}\right) , \quad  \theta \in \R\right ).
	\] 
 This model can be viewed as a non-linear extension of the Gaussian sequence model, where the mean vector is a linear combination of shifted Gaussian spikes. We are interested in recovering the unknown shift parameters $(\theta_k^\star)_{1 \leq k \leq s}$ belonging to the compact set   \(\Theta_T = [(1-\epsilon)a_T, (1-\epsilon)b_T] \subset [a_T, b_T]\), where $\epsilon$ is a given positive  shrinkage,  as well as the  unknown linear parameters $\beta^{\star}$.
	
 We apply our main result, Theorem~\ref{maintheorem}, which gives that: if the number of observations \(T\) is sufficiently large (depending on \(\sigma_0\), \(\epsilon\) and the sparsity $s$) and if the shift parameters are 
 separated, {\it i.e.} such that for all \(\ell \neq k\), \(|\theta_k^\star - \theta_\ell^\star| \gtrsim \sigma_0\), the estimators \(\hat{\beta}\) and \(\hat{\vartheta}\) defined in the minimization problem 
 \eqref{eq:generalized_lasso} using the regularization weight \(\kappa = \mathcal{C}   \sigma \sqrt{ \Delta_T \log (T) } \) achieve the following prediction error bound:  
	\[
	\norm{\hat{\beta}\Phi_{T}(\hat{\vartheta}) -
	\beta^{\star}\Phi_{T} (\vartheta^{\star}) }_{T}
	\leq \mathcal{C}' \sigma  \sqrt{\sparse \,\frac{\log(T)}{T} },
	\]
	with probability greater than \(1 - \mathcal{C}'' T^{-\gamma}$, for some $\gamma>0$, where  \(\mathcal{C}/\sqrt{\gamma} \), \( \mathcal{C}  '/\sqrt{\gamma}\) and \( (\sqrt{\gamma}\wedge 1)\, \mathcal{C} ''\) are some universal constants and $\norm{f}_{T}  = \frac{1}{\sqrt{T}}\sqrt{\sum_{j=1}^T f(t_j)^2}$. See Remark~\ref{rem:tau-choice} for details, with $\gamma'=\gamma$ therein.

\subsubsection{Scaled exponential model, see Section \ref{sec:scaled}.}    Consider the continuous time model \eqref{CTM} where the real-valued process  $y$ is  observed on $\R_+$ and assume that this process is an element of the Hilbert space $H_T=L^2(\R_+, \Leb)$ where $\Leb$ denotes here the Lebesgue measure over $\R_+$. We  write $H$ instead of $H_T$ for the Hilbert space and we write $\left \langle \cdot , \cdot \right \rangle$ its scalar product and $\norm{\cdot}$ its associated norm.

Let the noise process be a truncated white noise  as  in Section \ref{sec:continuous_noise} such that $w_T=\sum_{k = 1}^T (1/\sqrt{T})\, G_k\,  \psi_k$, where
$(G_k, k\in \N)$ are independent centered Gaussian random variables with variance $\sigma^2$ and $(\psi_k, k\in \N)$ denotes an orthonormal basis of $H$. We stress the fact that by the law of large numbers $\norm{w_T}^2$ tends almost surely to $\sigma^2>0$. Therefore the upper bounds from previous results on super-resolution and BLasso  (see \cite{duval2015exact} or \cite{poon2018geometry}) do not apply here, as they hold for noise processes having $\norm{w_T}$ tending to zero.

Let the dictionary  consist of the exponential functions :
	\[
	\left (\varphi(\theta) = \exp\left(- \theta \cdot \right), \quad \theta \in \R_+^* \right ).
	\]
We aim at recovering the unknown scale parameters $(\theta_k^\star)_{1 \leq k \leq s}$ belonging to a compact set  whose diameter may depend on $T\in \N^*$, say $\Theta_T = [ T^{-\gamma},T^\gamma]$, with $\gamma >0$, as well as the  unknown linear parameters $\beta^{\star}$. 

We apply our main result, Theorem~\ref{maintheorem}, which gives that: if the scale parameters are separated, {\it i.e.} such that for all \(\ell \neq k\), \( \left | \log (\theta_k^\star /\theta_\ell^\star) \right | \gtrsim 1\), the estimators \(\hat{\beta}\) 
and \(\hat{\vartheta}\) defined in the minimization problem \eqref{eq:generalized_lasso}, using the regularization weight \(\kappa =   \mathcal{C}   \,  \sigma \sqrt{\log (T)/T}\) achieve the following prediction bound:
\begin{equation*}
	\norm{\hat{\beta}\Phi_{T}(\hat{\vartheta}) -
		\beta^{\star}\Phi_{T}(\vartheta^{\star}) } \leq  \mathcal{C}'  \,\sigma \sqrt{\sparse\, \frac{\log(T)}{T}} ,  
	\end{equation*}
with                probability               larger                than
	$1 - \mathcal{C}'' T^{-\gamma} ( 1 \vee \sqrt{\gamma \, \log(T)} )$, where  $\mathcal{C}/\sqrt{\gamma}$, $\mathcal{C}'/\sqrt{\gamma}$ and $\mathcal{C}''$ are some universal constants. 
See Remark~\ref{rem:expo} for details, with $\gamma'=\gamma$ therein.

\section{Main Results}
\label{sec:main_result}

Recall that we consider the model (\ref{eq:model}) that we can write in an equivalent way as:
\[
y = \sum_{j\in S^\star} \beta_j^\star
\frac{\varphi_T(\theta_j^\star)}{\|\varphi_T(\theta_j^\star)\|_T}   + w_T\quad \text{in
} H_T,
\]
with $S^\star$ the support of the vector $\beta^\star$. 
The main  theorem of  this paper  gives the  behavior of  the prediction
error with respect to: the decay  rate of the noise variance $\Delta_T$,
the  parameter $T\in  \N$, the  sparsity $\sparse  \in \N^*$,  the upper
bound  on the  number of  components  in the  mixed signal  $K$ and  the
intrinsic noise  level $\sigma$.  We  shall consider assumptions  on the
regularity of the dictionary $\varphi_T$, on the parameter space $\Theta_T$
on which the optimization is performed and on the noise $w_T$. Using the
features  $\varphi_T$  we  build  a  kernel  $\cK_T$  on  the  space  of
parameters  $\Theta$ and  an associated  Riemannian metric $\dT$, see
Section~\ref{sec:riemannian_metric},  which  is  the  intrinsic  metric,
rather than the usual Euclidean  metric.  More assumptions are necessary
on the  closeness of the kernel  $\cK_T$ and its derivatives  defined in
\eqref{eq:def-KT} to a limit kernel $\cK_\infty$ and its derivatives.

The theorem is  stated assuming the existence  of certificate functions,
see  Assumptions  \ref{assumption1}  and  \ref{assumption2}.  Sufficient
conditions   for   their   existence   are  given   later   in   Section
\ref{sec:conditons_certificates},        in        which    Propositions
\ref{prop:certificat_interpolating}   and   \ref{prop:certificat2}
show that  the limit kernel  $\cK_\infty$ must be uniformly  bounded and
have concavity properties.  In this case, the  existence of certificates
stands provided the underlying non-linear  parameters to be estimated are
sufficiently separated according to the Riemannian metric $\dT$, see
Condition~$\ref{hyp:theorem_certificate_separation}$ in
Propositions~\ref{prop:certificat_interpolating}
and~\ref{prop:certificat2}. 

In the following result the parameter set $\Theta_T$ is a one
dimensional compact interval.  We note $|\Theta_T|_{\mathfrak{d}_T}$
its length  with respect to the
Riemannian metric $\dT$ on $\Theta^2$ associated to the
kernel $\cK_T$.

\begin{theorem}
	\label{maintheorem}
	Assume we observe the random element $y$ of $H_T$ under the regression model (\ref{eq:model}) with unknown parameters  $\beta^\star$ and $\vartheta^\star= \left (
	\theta_1^\star,\cdots,\theta_K^\star\right )$ a vector with entries in
	$ \Theta_T$, a compact interval of $\R$,   such that:
	\begin{propenum}
		\item \label{hyp:theorem1_point1}\textbf{Admissible noise:} The noise process $w_T$ satisfies Assumption \ref{hyp:bruit} for a noise level $\sigma>0$ and a decay rate for the noise variance $\Delta_T>0$.

		\item\textbf{Regularity           of            the           dictionary
			$\varphi_T$:}\label{hyp:reg_dic_theorem}  The   dictionary  function
		$\varphi_T$     satisfies     the     smoothness     conditions     of
		Assumption~\ref{hyp:reg-f}.    The    function   $g_T$    defined   in
		\eqref{def:g_T}, satisfies the positivity condition of
		Assumption~\ref{hyp:g>0}. 
		\item\textbf{Regularity of the limit kernel:} \label{it:hyp-reg-K}
		The   kernel   $\ck_{\infty}$ and the functions  $ g_{\infty}$   and
		$h_{\infty}$,   defined  on an interval $\Theta_\infty \subset
		\Theta$, see    \eqref{eq:def-gK}   and
		\eqref{eq:def-h_K}, satisfy  the  smoothness  conditions  of  Assumption
		\ref{hyp:Theta_infini}.
		\item\textbf{Proximity to the limit kernel:} \label{hyp:V_T_theorem} The
		kernel $\cK_T$ defined from  the dictionary, see~\eqref{eq:def-KT}, is
		sufficiently close to the limit  kernel $\cK_\infty$ in the sense that
		Assumption \ref{hyp:close_limit_setting} holds.
		\item\textbf{Existence of  certificates:}\label{hyp:existence_certificate_theorem} The set of  unknown parameters
		$\cq^\star= \{\theta^\star_k,  \, k \in S^\star\}  $, with
		$S^\star=\supp (\beta^\star)$,  satisfies
		Assumptions  \ref{assumption1} and  \ref{assumption2}  with the  same
		$r >0$.
	\end{propenum}
	Then, there exist finite positive constants $\mathcal{C}_0$, 
	$\mathcal{C}_1$, $\mathcal{C}_2$, $\mathcal{C}_3$ 
	depending on the kernel $\cK_\infty$ defined on $\Theta_\infty$ and on $r$ such that for 
	any $\tau > 1$ and a tuning parameter:
	$$
	\kappa \geq \mathcal{C}_1 \sigma \sqrt{\Delta_T  \log \tau},
	$$
	we have the prediction error bound of the estimators $\hat{\beta}$ and $\hat{\vartheta}$ defined in
	\eqref{eq:generalized_lasso} given by:
	\begin{equation}
	\label{eq:main_theorem}
	\begin{aligned}
	\norm{\hat{\beta}\Phi_{T}(\hat{\vartheta}) -
		\beta^{\star}\Phi_{T}(\vartheta^{\star}) }_{T}
	&\leq  \mathcal{C}_0 \,  \sqrt{\sparse} \, \kappa,
	\end{aligned}
	\end{equation}
	with  probability larger than $1  -
	\mathcal{C}_2 \left (  \frac{|\Theta_T|_{\mathfrak{d}_T}}{ \tau \sqrt{\log \tau} }\vee
	\frac{1}{\tau}\right )$.
	Moreover, with the same probability, the difference of the $\ell_1$-norms of $\hat{\beta}$ and $\beta^\star$ is bounded by:
	\begin{equation}
	\label{eq:main_theorem_diff_l1}
	\left |\| \hat{\beta}\|_{\ell_1} - \| \beta^\star \|_{\ell_1} \right | \leq \mathcal{C}_3 \, \kappa \, \sparse.
	\end{equation}
\end{theorem}

{This result holds for  both the continuous and discrete
	settings described in Section~\ref{sec:examples},  covers a wide range
	of smooth  dictionaries, and is  proven under mild assumptions  on the
	noise. We discuss in the next remark that the prediction error is, up to
	a logarithmic factor, almost optimal.}

\begin{rem}[Comparison with the Lasso estimator]
	\label{rem:thm1}
	Let us consider the discrete-time model 
	where the observation space is the Hilbert space $H_T = \R^T$ endowed
	with  the Euclidean norm $\norm{\cdot}_{\ell_2}$. The observation $y \in \R^T$ comes from the model \eqref{eq:model} where the noise is a Gaussian vector with independent entries of variance $\sigma^2$. In this setting, the decay rate of the noise variance is fixed with $\Delta_T=1$. 
	
	We first consider  that the parameters $\vartheta^\star$  are known.  In
	this case,  the model becomes the  classical high-dimensional regression
	model  and the  Lasso estimator $\hat{\beta}_{L}$ can be  used to  estimate $\beta^\star$
	under     coherence    assumptions     on     the    finite dictionary made of the rows of the matrix $\Phi^\star = \Phi_T(\vartheta^\star)$    (see
	\cite{bickel2009simultaneous}). The behavior of  the Lasso estimator has
	been studied in the literature and its prediction risk tends to zero at
	the rate:
	\begin{equation}
	\label{eq:mse}
	\frac 1{T} \|(\hat \beta_{L} - \beta^\star)\Phi^\star\|_{\ell_2}^2 = \mathcal{O}
	\left (\frac{\sigma^2 \, \sparse \, \log  (K)}{T} \right)
	\end{equation}
	with
	high   probability, larger than $1 - 1/K^\gamma$ for some positive constant $\gamma > 0$.   Furthermore,   in    the  case   where
	$\beta^{\star}$     is    an     unknown    $\sparse$-sparse     vector,
	$\vartheta^{\star}$ is known  and $\Phi^\star$
	verifies a  coherence property, then  the lower  bounds of
	order $  {\sigma^2\, \sparse \, \log (K/\sparse)}/{T}$ in expected
	value can be deduced from the  more general bounds for group sparsity in
	\cite{lounici2011oracle}
	(see also \cite{MR2882274}). 
	The  non-asymptotic prediction lower bounds
	for the prediction error given in \cite{MR2882274}  are:
	\[
	\inf_{\hat \beta}\,  \sup_{\beta^\star \, s- \text{sparse} }
	\E\left[\frac 1{T}
	\|(\hat \beta - \beta^\star)\Phi^\star\|_{\ell_2}^2\right]\geq C\cdot
	\frac{  \sigma^2\, \sparse \, \log(K/s)}{T} ,
	\]
	where  the infinimum  is  taken over  all  the estimators  $\hat{\beta}$
	(square integrable measurable  functions of the obervation  $y$) and for
	some  constant  $C>0$  free  of   $s$  and  $T$.   When  the  parameters
	$\vartheta^\star$ are unknown,  Theorem~\ref{maintheorem} gives an upper
	bound for  the prediction risk which  is, {up to a  logarithmic factor},
	almost  the best  rate  we  could achieve  even  knowing the  non-linear
	parameters    $\vartheta^{\star}$.    Consider    the   estimators    in
	\eqref{eq:generalized_lasso} where  the Riemannian  diameter of  the set
	$\Theta_T$ is bounded by a constant  free of $T$ (this is the
	case  of  Example~\ref{example:compact_support}   below).   By  squaring
	\eqref{eq:main_theorem}  and then  dividing it  by $T$,  we obtain  from
	Theorem                      \ref{maintheorem}                      with
	$\kappa  =   \mathcal{C}_1  \sigma  \sqrt{  \Delta_T   \log  \tau}$  and
	$\tau  = T^\gamma$  for  some given $\gamma>0$,  that  with high  probability,
	larger than $1- C/T^\gamma$:
	\begin{equation}
	\label{eq:mse2}
	\frac{1}{T}\norm{\hat{\beta}\Phi_{T}(\hat{\vartheta}) -
		\beta^{\star}\Phi_{T}(\vartheta^{\star}) }_{\ell_2}^2
	=\mathcal{O} \left  (\frac{\sigma^2 \,\sparse\, \log( T)}{T}  \right).
	\end{equation}
	Let us mention that \cite{tang2014near} also obtained a similar  prediction
	error~\eqref{eq:mse2}
	for the  specific 
	dictionary  given by the  complex exponential functions 
	$(\varphi(\theta): t \mapsto \expp{i 2\pi t \theta}, \theta \in
	\Theta=[0, 2 \pi] )$; notice that the proof therein uses the Parseval's
	identity for Fourier series as well as Markov-Bernstein type
	inequalities for trigonometric polynomials.
	Even if the structure of our proof is in the spirit of
	\cite{tang2014near}, our result is more general and does not rely on the
	convex setting  of the
	BLasso approach.
\end{rem}

\begin{rem}[Proximity to the limit kernel]
	\label{rem:cv-=0}
	We comment  on Condition~$\ref{hyp:V_T_theorem}$ on the  proximity of
	the  kernels   $\cK_T$  and  $\cK_\infty   $,  which  also   appears  as
	Conditions~$\ref{hyp:theorem_certificate_metric}$-$\ref{hyp:theorem_certificate_approximation}$
	in   Proposition~\ref{prop:certificat_interpolating}    (and   similarly as
	Condition~$\ref{hyp:theorem_certificate_2_approximation}$             in
	Proposition~\ref{prop:certificat2}).
	
	In the examples of
	Sections~\ref{sec:translation_cont_model}
	and~\ref{sec:translation_scale_model}  on translation  or scaling  model
	with a  continuum of observations, the  parameter $T$ does not  play any
	role in the definition of $\ck_T$, so that one can take $\cK_\infty $ equal to $\cK_T$. In this case,
	the proximity conditions on the kernels are trivially satisfied.
	
The   example  from
	Section~\ref{sec:example} is devoted to  the Gaussian sparse spike deconvolution,
	that  is, to  a mixture  of Gaussian  translation invariant  features
	observed in a  discrete regression model on a regular  grid of size $T$.
	In     this     case,     we     built    a     family     of     models
	$(H_T, \varphi_T,  w_T, \Theta_T)$ with a  dictionary $\varphi_T$ which
	does  not depend  on  $T$ and  such  that the  kernel  $\cK_T$ and  its
	derivatives   converge   to   $\cK_\infty   $   (and   also   $\rho_T$
	from~\eqref{eq:def-rho}  converges  to  1).  In  this setting,  the
	proximity  condition of  Theorem~\ref{maintheorem} holds  for $T$  large
	enough, say  $T$ larger than some  $T_0$ which depends on  $\cK_\infty$,
	see  Assumption  \ref{hyp:close_limit_setting}.   The existence  of  the
	certificates,  see         Propositions~\ref{prop:certificat_interpolating}
	and~\ref{prop:certificat2}, also requires a  proximity criterion which is
	achieved for  $T$ large  enough, say  $T$ larger  than some  $T_1$ which
	depends on $\cK_\infty  $ and is increasing with  the sparsity parameter $s$   (see  for   example
	Condition~$\ref{hyp:theorem_certificate_approximation}$               in
	Proposition~\ref{prop:certificat_interpolating}). 
\end{rem}

\begin{rem}[On the dimension $K$, the upper bound of the sparsity]
	\label{rem:K-dependence}
	We remark  that neither  the bound  on the  prediction error  nor the
	probability on which  the bound holds, depends on the  upper bound $K$
	on the sparsity  $\sparse$. Therefore, the value of $K$  can be taken
	arbitrarily large.  It is not surprising  that $K$ does not  have any
	impact    on   the    bound    since    the   optimisation    problem
	\eqref{eq:generalized_lasso} could be formulated without any bound on
	the sparsity. Indeed, the problem \eqref{eq:generalized_lasso} can be
	embedded  in  an  optimization  problem  over  a  space  of  measures
	following   the    literature   on    the   BLasso    introduced   in
	\cite{de2012exact}. See also Remark~\ref{rem:suiteK}.
\end{rem}

The next theorem gives bounds  on the differences between the parameters
$\hat     \beta$      given     by     the      optimization     problem
\eqref{eq:generalized_lasso}  and the  ``true'' parameters  $\beta^\star$
for    active    features    having    their    parameter
$\hat \theta_\ell$ close, with respect to the Riemannian metric $\dT$,
to  a parameter $\theta_k^{\star}$, with $k$
in $S^\star$.  For $r>0$ given by   Assumptions  \ref{assumption1} and
\ref{assumption2}, we define:
\begin{itemize}
	\item [-] The support of $\hat  \beta$ given by the optimization problem
	\eqref{eq:generalized_lasso}:
	$\hat{S}     =    \supp(\hat     \beta)=    \left\{\ell     \,\colon\,
	\hat{\beta}_{\ell} \neq 0\right\}$.
	
	\item [-] The near region $\tilde{S}(r)$ given by:
	\[
	\tilde{S}(r) = \bigcup_{k \in S^\star}
	\tilde{S}_{k}(r)
	\quad\text{where}\quad
	\tilde{S}_{k}(r) = \left \{\ell\in \hat{S}:
	\mathfrak{d}_T(\hat{\theta}_{\ell},\theta_{k}^{\star}) \leq r \right
	\}  ,
	\]
	which  corresponds to the set of indices $\ell$ in the support of $\hat \beta$
	such that the corresponding  parameter
	$\hat \theta_\ell$ is close to one of the true parameters
	$\theta_k^{\star}$, for some 
	$k\in S^\star$. 
\end{itemize}
The set $\hat{S} \backslash \tilde S(r)$  is also called the far region.
Notice that  the sets  $\tilde{S}_k(r)$ with $k \in S^\star$ are
pairwise  disjoint under
Assumption~\ref{assumption1}, and that they can be empty. In what
follows, we use the 
convention   $\sum_{\emptyset} = 0$. 

\begin{theorem}
	\label{th:bounds}
	We consider  the model in  Theorem \ref{maintheorem} and  suppose that
	Assumptions   \ref{hyp:theorem1_point1}-\ref{hyp:existence_certificate_theorem}  therein
	hold.  Then,  there exist  finite positive  constants $\mathcal{C}_1$,
	$\mathcal{C}_2$, $\mathcal{C}_3 $,  $\mathcal{C}_4 $, $\mathcal{C}_5 $ and $\mathcal{C}_6 $ 
	depending on  $\cK_\infty$ defined on  $\Theta_\infty$ and on  $r$  such that for  any $\tau > 1$  and a
	tuning parameter:
	\[
	\kappa \geq \mathcal{C}_1 \sigma \sqrt{ \Delta_T\log \tau}
	\]
	the  estimators  $\hat{\beta}$ and $\hat{\vartheta}$
 defined  in~\eqref{eq:generalized_lasso}
satisfy  the  following   bounds  with   probability  larger   than
	$1            -            \mathcal{C}_2            \left            (
	\frac{|\Theta_T|_{\mathfrak{d}_T}}{\tau\sqrt{\log     \tau}    }\vee
	\frac{1}{\tau}\right )$:
	\begin{equation}
	\label{eq:bound_th_abc}
	\sum\limits_{k \in S^{\star}}\Big | |\beta_{k}^{\star}| -
	\sum\limits_{\ell \in \tilde{S}_{k}(r) } |\hat{\beta}_{\ell}|  \Big|
	\leq  \mathcal{C}_3 \, \kappa \, \sparse,
	\quad
	\sum\limits_{k \in S^{\star}}\Big | \beta_{k}^{\star} -
	\sum\limits_{\ell \in \tilde{S}_{k}(r) } \hat{\beta}_{\ell} \Big|
	\leq  \mathcal{C}_4 \, \kappa \, \sparse
	\quad\text{and}\quad
	\left \|\hat{\beta}_{\tilde{S}(r)^{c}}\right \|_{\ell_1}  \leq
	\mathcal{C}_5 \, \kappa \, \sparse,
	\end{equation}
 \begin{equation}
	\label{eq:bound_th_I2}
	\sum\limits_{k \in S^{\star}} \sum\limits_{\ell \in \tilde{S}_{k}(r) }
	\left |\hat{\beta}_{\ell}\right |
	\mathfrak{d}_T(\hat{\theta}_{\ell},\theta_k^{\star})^2
	\leq \mathcal{C}_6 \, \kappa \, \sparse,
	\end{equation} 
	where for a subset $S$ of $\mathcal{I} = \{1,\cdots,K\}$, the set $S^c$ denotes the complementary set of $S$ in $\mathcal{I} $, that is $\mathcal{I}  \setminus S$.
\end{theorem}
Notice that  each linear parameter  $\beta_k^\star$ can be  estimated by
the sum of several      linear       coefficients      $\hat{\beta}_\ell$      with
$\ell   \in   \{1,\cdots,   K\}$.   The  first   two   inequalities   in
\eqref{eq:bound_th_abc} show  that there can be  no compensation between
the   estimators   $\hat   \beta_\ell$   that   approximate   the   same
$\beta^\star_k$ with $k \in S^\star$, meaning that there can be no large
values of  $\hat \beta_\ell$  having different  signs that  sum up  to a
possibly  small (in  absolute  value) true  $\beta_k^\star$. The  second
inequality in  \eqref{eq:bound_th_abc} gives the estimation  rate of the
linear parameters $\beta^\star_k$ with $k  \in S^\star$.  The last bound
in~\eqref{eq:bound_th_abc}  basically  means  that  when  an  estimation
$\hat{\theta}_{\ell}$  with $\ell  \in \{1,\cdots,K\}$  is far  from any
parameter $\theta_k^\star$ with $k \in  S^\star$, that is at a distance
greater than  $r$, the  associated  parameters  $\hat{\beta}_\ell$ drop  to
zero if the tuning parameter $\kappa$ is taken equal to its lower bound and the decay rate of the noise variance $\Delta_T$ drops to zero. Therefore, the contribution  of the parameters $\hat{\theta}_\ell$
in the far region,  that are not in $\tilde{S}(r)$, will  drop to zero as
well.


\begin{rem}[Again on the dimension $K$]
	\label{rem:suiteK}
	As in Theorem \ref{maintheorem}, we remark that neither the bounds nor
	the probability  of the event on  which the bounds hold  depend on the
	upper bound $K$ on the sparsity $\sparse$.
	
	If the optimization on  $\vartheta$ in \eqref{eq:generalized_lasso} is
	performed over a subset of $\Theta_T$  in which the coordinates of the
	considered  vectors are  at  a  distance greater than  $2r$ pairwise  with
	respect  to  the Riemannian  metric  $\mathfrak{d}_T$,  then the  sets
	$\tilde{S}_k(r)$ contain at  most one element.  However,  by doing so,
	we  introduce  an  upper  bound   on  the  dimension  $K$  whereas  in
	Theorem~\ref{maintheorem} the  dimension K  can be  arbitrarily large.
	Indeed, $\Theta_T$  is a compact  set and therefore contains  a finite
	number of balls of size $2r$.
\end{rem}

\medskip

{\bf Outline  of the paper.  } In Section  \ref{sec:property_kernel}, we
give  the definition  of the  kernel $\ck_T$  measuring the  correlation
between two  elements in  the continuous dictionary  and we  present the
regularity   assumptions   on    the   function   $\varphi_T$.   Section
\ref{sec:riemannian_metric} introduces the Riemannian geometry framework
useful in  our context.  Section \ref{sec:cv_kernel} defines the
convergence (or closeness condition) of  kernels $\ck_T$
towards a  limit kernel $\ck_{\infty}$.  Then, we require  properties on
the  limit  kernel $\ck_{\infty}$  and  propagate  them to  the  kernels
$\cK_T$ thanks to this  convergence.  In Section \ref{sec:certificates},
we present the assumptions on the existence of the so-called certificate
functions    used    to    state    Theorems    \ref{maintheorem}    and
\ref{th:bounds}.  We give  sufficient  conditions for  the existence  of
certificate functions in Section \ref {sec:conditons_certificates}.  The
examples  of Gaussian sparse spike deconvolution and of Scaled exponential family in our  regression  model  is  fully
detailed      in     Section~\ref{sec:example} and~\ref{sec:scaled}, respectively.       Then,   the Appendix \ref{sec:proofsSection2} is dedicated   to   the   proofs    of   Theorems   \ref{maintheorem}   and
\ref{th:bounds}.  The proofs of  existence and explicit constructions of
the certificates  are detailed in  the Appendix~\ref{sec:proof_interpolating}. Other intermediate results are proven in Appendix~\ref{app:C}
.

\section{Dictionary of features}
\label{sec:property_kernel}

We present in the next section the regularity  assumptions on the
features $(\varphi_T(\theta), \theta\in \Theta)$ we shall consider and then give examples of families of features satisfying such assumptions. 

\subsection{Assumptions on the regularity of the features}
Let  $T  \in     \N$  be  fixed.  We  consider   the  Hilbert  space
$(H_T,   \langle    \cdot,   \cdot   \rangle_T)$   and    the   features
$(\varphi_T(\theta), \theta\in \Theta)$ which  are elements of $H_T$. We
shall consider the following regularity assumptions on the features.

\begin{hyp}[Smoothness of $\varphi_T$]
	\label{hyp:reg-f} 
	We assume that the function $\varphi_T: \Theta \rightarrow H_T$
	is of class    $\cc^3$ and
	$\norm{\varphi_T   (\theta)}_T    >   0$   on    $\Theta$.
\end{hyp}

Recall  $\phi_T= \varphi_T/ \norm{\varphi_T}_T$
from~\eqref{eq:def-phi_T} and notice that $\phi_T$, and thus $\Phi_T$, are continuous functions. 
Under Assumption~\ref{hyp:reg-f},   elementary calculations
using~\eqref{eq:deriv} give:

\begin{equation}
\label{eq:deiv-phi}
\partial_\theta \phi_T(\theta) = \frac{\partial_\theta
	\varphi_T(\theta)}{\norm{\varphi_T(\theta)}_T} -
\frac{\varphi_T(\theta) \left \langle
	\varphi_T(\theta),\partial_\theta \varphi_T(\theta) \right
	\rangle_T}{\norm{\varphi_T(\theta)}^3_T} ,
\end{equation}
and thus, we  deduce that the function $g_T:\Theta\mapsto  \R_+$ defined by:
\begin{equation}
\label{def:g_T}
g_T(\theta)= \norm{\partial_\theta
	\phi_T(\theta)}_T^2 
\end{equation}
is well
defined and continuous.

We  shall  consider  the  following  non-degeneracy  assumption  on  the
features.

\begin{hyp}[Positivity of $g_T$]
	\label{hyp:g>0}
	Assumption \ref{hyp:reg-f} holds and we have	$g_T>0$ on $ \Theta$.
\end{hyp}

Even if Assumption \ref{hyp:g>0} requires Assumption \ref{hyp:reg-f}, in the following we shall stress when Assumption \ref{hyp:reg-f} is in force.

\medskip

The next lemma gives a sufficient condition on $\varphi_T$ for
Assumption~\ref{hyp:g>0} to hold. 
\begin{lem}[On the positivity of $g_T$]
	\label{lem:g>0}
	Suppose  Assumption~\ref{hyp:reg-f}  holds.    If  the elements 
	$\varphi_T(\theta)$  and $\partial_\theta\varphi_T(\theta)$ of $H_T$  are
	linearly   independent   for    all   $\theta\in   \Theta$   and
	$\norm{\partial_\theta   \varphi_T(\theta)}_T  >   0$  for   all
	$\theta \in \Theta$, then $g_T$ is positive on $\Theta$.
\end{lem}
\begin{proof}
	For simplicity, we remove the subscript $T$, and for example write
	simply $\phi=\varphi/\norm{\varphi}$.  Recall     that     by      Assumption~\ref{hyp:reg-f}     we     have
	$\norm{\varphi(\theta)} > 0$.  Assume there exists $\theta \in \Theta$ such
	that $g(\theta) = 0$, that is   $\partial_\theta  \phi(\theta)  =  0$.
	Since  $\norm{\varphi(\theta)}>0$, we  deduce  from~\eqref{eq:deiv-phi}
	that
	$               \partial_\theta\varphi(\theta)\norm{\varphi(\theta)}^2
	-\varphi(\theta)   \left   \langle  \varphi(\theta),   \partial_\theta
	\varphi(\theta)\right  \rangle=0$.   Then  use  that  by  assumption
	$\partial_\theta\varphi(\theta)\neq 0$ and $\norm{\varphi(\theta)}>0$,
	to get  that $\varphi(\theta)$ and $\partial_\theta \varphi(\theta)$
	are   linearly   dependent.    In   conclusion,   we   get   that   if
	$\varphi(\theta)$    and   $\partial_\theta\varphi(\theta)$    are
	linearly independent, then $g(\theta)>0$.
\end{proof}

\subsection{Examples of regular features}
The aim of this section of examples is to stress that  a large variety
of dictionaries of features and type of parameters verify
Assumptions~\ref{hyp:reg-f}  and~\ref{hyp:g>0}.

\subsubsection{Translation discrete-time model} 
\label{sec:translation_dis_model}
Let $t_1<\cdots< t_T$  be a grid on $\R$
of size  $T\in \N$, $\lambda_T$ an  atomic measure whose support  is the
grid,  and  $H_T=L^2(\lambda_T)$.   Consider the  translation  invariant
dictionary:
\begin{equation}
\label{eq:trans-inv}
(\varphi_T(\theta)  = k(\cdot  -\theta),  \, \theta  \in \Theta),
\end{equation}
with $\Theta=  \R$  and  $k$ is  a   real-valued  $\cc^3$  function  defined  on
$\R$. Notice  the dictionary  does not  depend on  $T$. We  now consider
usual choices for the function $k$.

For the Gaussian function $k(t)=  \expp{-t^2/2}$ and the Cauchy function
$k(t) = 1/(1 + t^2)$, we get that Assumption \ref{hyp:reg-f} holds
and,  using Lemma~\ref{lem:g>0}  that Assumption  \ref{hyp:g>0} is  also
satisfied provided respectively $T \geq 2$ and $T\geq 3$.

For           the           Shannon           scaling           function
$k(t  )  = \operatorname{sinc}(t)  =  \sin(\pi  t)/(\pi t)$,  Assumption
\ref{hyp:reg-f}  holds  provided  that $\lambda_T((a+\Z)^c)>0$  for  all
$a  \in \R$,  that is  the grid  is not  a subset  of $a+\Z^*$  for some
$a\in \R$. There is  no easy way to write conditions  on the grid, based
on the use of Lemma~\ref{lem:g>0}, for Assumption \ref{hyp:g>0} to hold
(let       us        mention       that       $T\geq        2$       and
$\min_{1\leq i\leq T-1} (t_{i+1} - t_i)<1/2$ is a sufficient condition for
Assumption \ref{hyp:g>0} to hold).

Eventually notice that the   Laplace function $k(t)=\expp{-|t|}$
is  not smooth  enough  for Assumption  \ref{hyp:reg-f} to hold.

\subsubsection{Translation model with a continuum of 
	observations}
\label{sec:translation_cont_model}
Let $T\in \N$ (which  does not play a role here)  and $H_T = L^2(\Leb)$,
where $\Leb$  is the Lebesgue measure  on $\R$.  In this  framework, the
observation  $y$   defined  in   \eqref{eq:model}  is  a   continuum  of
observations.   Consider  the  translation invariant  dictionaries  from
Section~\ref{sec:translation_dis_model}, where $k$ is either the Gaussian, the
Cauchy or the Shannon  scaling function. Notice that the  Hilbert space and
the dictionary  do not depend  on $T$.  Then, it  is easy to  check that
Assumptions \ref{hyp:reg-f} and \ref{hyp:g>0} hold.

We see that this model, which can be seen as a continuous  approximation
(or limit) of the
discrete models from
Section~\ref{sec:translation_dis_model} when $T$ therein is large, is easier
to handle than the corresponding discrete models.

\subsubsection{Translation model with a varying scaling parameter}
\label{sec:translation_cont_model2}
Let $T\in \N$,  $H_T = L^2(\Leb)$, where $\Leb$ is  the Lebesgue measure
on $\R$, and  consider the translation invariant  dictionary scaled by
$\scale_T>0$ given by:
\[
(\varphi_T(\theta)  = k(\scale_T^{-1}(\cdot  -\theta)),  \, \theta  \in \Theta),
\]
with  $\Theta= \R$  and $k$ is a real-valued  $\cc^3$ function  defined on
$\R$.    Contrary   to   Section~\ref{sec:translation_cont_model},   the
features  depend  on $T$.   Suppose  that  $k$  is the  Shannon  scaling
function (see Section~\ref{sec:translation_dis_model})  and consider the
vector sub-space $V_T$ given by the closure in $H_T$ of the vector space
spanned        by       the        dictionary.       According       to
\cite[Theorem~3.5]{mallat2008wavelet}, the  set $V_T$  is the  subset of
$H_T$ of  all functions whose Fourier  transform support is a  subset of
$[-\pi/\scale_T, \pi /\scale_T]$. Suppose that the sequence $(\scale_T, T\in \N)$ is
decreasing to $0$. Then the sequence $(V_T, T\in \N)$ is increasing
and
$\overline{\bigcup_{T \in \N} V_T} = H_T$. {This model provides an example
	of translation models with possibly varying, but known, scaling parameter $\scale_T$.}

\subsubsection{Scaling exponential model}
\label{sec:translation_scale_model}
Let $T \in \N$ (which  does not play a role here),   $H_T = L^2(\Leb)$, where $\Leb$ is  the Lebesgue measure
on $\R_+$,  and  consider the scale  invariant  dictionary  given by:
\[
(\varphi_T(\theta)  = k(\theta \cdot ),  \, \theta  \in \Theta),
\]
with     $\Theta=    \R_+^*$     and     the    exponential     function
$k:  t  \mapsto \expp{-t}$.  This  dictionary  is  used for  example  in
fluorescence microscopy  (see  \cite{denoyelle2019sliding}).
Clearly     Assumption~\ref{hyp:reg-f}     holds      as     well     as
Assumption~\ref{hyp:g>0} as $g_T(\theta) =1/(4\theta^2)$.

\section{A Riemannian metric on the set of parameters}
\label{sec:riemannian_metric}

\subsection{On the Riemannian metric in dimension one}
\label{sec:riemann}
Recall $\Theta$ is an interval of  $\R$.  We call kernel  a real-valued
function defined on $\Theta^2$. Let $\cK$ be a symmetric kernel of class
$\cc^2$ such  that the  function $g_\cK$  defined on
$\Theta$ by:
\begin{equation}
\label{eq:def-gK}
g_\cK(\theta)= \partial^2_{x,y}
\cK( \theta,\theta)
\end{equation}
is positive and locally  bounded, where $\partial_x$
(resp. $\partial_y$) denotes the usual  derivative with respect to the
first (resp. second) variable.  Following \cite{poon2018geometry}, we
define an intrinsic Riemannian metric, denoted $\dK$, on the parameter set
$\Theta$ using the  function $g_\cK$.  One of the motivations  to use the
Riemannian metric is to  work with intrinsic quantities  related to  the
parameters  which  are  invariant  by  reparametrization,  such  as  the
diameter of (subsets of) $\Theta$. Since $\Theta$ is one-dimensional and
connected,   the  Riemannian  metric $\dK(\theta,\theta')$   between
$\theta, \theta'\in \Theta$ reduces to:
\begin{equation}
\label{eq:def-Riemann-dist-v2}
\dK(\theta,\theta')=  |G_\cK(\theta) -G_\cK(\theta')|,
\end{equation}
where  $G_\cK$ is a primitive of $\sqrt{g_\cK}$.

\begin{rem}
	\label{rem:Rieman-gen}
	We refer to    \cite{lee_book}     and
	\cite{sakai1996riemannian} for a general presentation on Riemannian
	manifolds, and we give an immediate application in dimension one which
	entails in particular ~\eqref{eq:def-Riemann-dist-v2}. 
	Let  $\Theta$  be  a   manifold  (of  dimension  one).   A  path
	$\gamma :  [0,1] \rightarrow \Theta$  is an  admissible path if  it is
	continuous, piecewise  continuously differentiable  with non-vanishing
	derivative.         Its        length        is        given        by
	$\mathcal{L}_\cK(\gamma)           =            \int_{0}^{1}           |
	\dot{\gamma}_s|\,\sqrt{g_\cK(\gamma_s)}\,     \rd     s     $,     where
	$   |  \dot{\gamma}_s|$   is  seen   as   the  norm   of  the   vector
	$ \dot{\gamma}_s$ in the tangent space,  and the scalar product on the
	tangent     space     at     $\theta\in     \Theta$     is     given     by
	$(u,v)    \mapsto    \langle    u,     g_\cK(\theta)    v    \rangle$    with
	$\langle \cdot, \cdot \rangle$ the  usual Euclidean scalar product. (In
	our case,  the tangent vector  space is  $\R$ and the  Euclidean scalar
	product reduces to the usual  product).  The Riemannian metric $\dK$
	between $\theta, \theta'$ in $\Theta$ is then defined by:
	\begin{equation}
	\label{eq:def-Riemann-dist}
	\dK	(\theta,\theta')=  \underset{\gamma}{\inf} \quad \mathcal{L}_\cK(\gamma),
	\end{equation}
	where the  infinimum is  taken over the  admissible paths  $\gamma$ such
	that  $\gamma_0  =  \theta$  and   $\gamma_1  =  \theta'$.
	It is not hard to see that $\gamma$ is a minimizing path, that is, $
	\dK(\theta,\theta')= \mathcal{L}_\cK(\gamma)$,  if and only if $\gamma$ is monotone (and thus $\gamma_s\in
	[\theta\wedge \theta', \theta\vee\theta']$ for all $s\in [0, 1]$). This is equivalent to say
	that the sign of $ \dot{\gamma}_s$ is constant. Assume that $g_\cK$ is
	of class $\cc^1$. The path $\gamma$ is a
	geodesic if it is smooth with zero acceleration, that is, in dimension
	one for all $s\in (0, 1)$:
	\begin{equation}
	\label{eq:geodesic}
	\ddot{\gamma}_s + \frac{1}{2}\frac{g_\cK'(\gamma_s)}{g_\cK(\gamma_s)} \dot{\gamma}_s^2 = 0.
	\end{equation} 
	This is equivalent to $s\mapsto
	\dot{{\gamma}}_s\,\sqrt{g_\cK({\gamma}_s)} $ being 
	constant, which implies that the geodesic is a
	minimizing path.  
	\medskip
	
	We  now derive  the equation  of  the geodesic  path when $g_\cK$ is
	of class $\cc^1$.  Recall  $G_\cK$
	denotes the  primitive of $\sqrt{g_\cK}$. It  is continuous increasing
	and thus  induces a  one-to-one map  from $\Theta$  to its  image. Set
	$a=G_\cK(\theta)$ and $b=G_\cK(\theta') - G_\cK(\theta)$, so that the path
	$\gamma:     [0,     1]     \rightarrow     \Theta$     defined     by
	$\gamma_s=G_\cK^{-1}(a+bs)$  is a  geodesic and  minimizing path  from
	$\theta$                to                $\theta'$ with 
	$ \mathcal{L}_\cK(\gamma)=\dK(\theta,\theta') $.

\end{rem}

\medskip

Following   \cite{poon2018geometry},   we    introduce   the   covariant
derivatives,  see  \cite[Sections~3.6 and  ~5.6]{absil2009optimization},
which have elementary expressions as  the set of parameters $\Theta$ is
one-dimensional.  For  a smooth function $f$ defined on $\Theta$ and
taking values in an Hilbert space, say $H$, the covariant  derivative $D_{i;\cK}[f]$ of
order  $i\in \N$  is defined  recursively by  $D_{0;\cK}[f] =  f$ and  for
$i\in \N$, assuming that $g_\cK$ is of class $\cc^{i}$, and $\theta\in \Theta$:
\begin{equation}
\label{eq:def-cov-deriv}
D_{i+1;\cK}[f](\theta) = g_\cK(\theta)^{\frac{i}{2}}\partial_{\theta} \left(
\frac{D_{i;\cK}[f](\theta) }{g_\cK(\theta)^{\frac{i}{2}}}\right).
\end{equation}
In particular,  we have for  $f\in \cc^2(\Theta, H)$ (and  assuming that
$g_\cK$ is of class $\cc^1$ for the last equality) that:
\begin{equation}
\label{eq:deriv1-3}
\begin{aligned}
&D_{0;\cK}[f] = f, \quad D_{1;\cK}[f] = \partial_{\theta}f, \quad D_{2;\cK}[f] = \partial_{\theta}^2 f - \frac{1}{2}\frac{g_\cK'}{g_\cK}\partial_{\theta}f.
\end{aligned}
\end{equation}

We shall also consider the following modification of the covariant
derivative, for $i\in \N$:
\begin{equation}
\label{eq:def-tD}
\tD_{i;\cK}[f](\theta)=g_\cK(\theta)^{-i/2}\, D_{i; \cK}[f](\theta).
\end{equation}
We have $\tD_{0; \cK}[f]=f$,   and we deduce from~\eqref{eq:def-cov-deriv}
that for $i\in \N^*$, assuming that $g_\cK$ is of class $\cc^{i}$:
\begin{equation}
\label{eq:tDi}
\tD_{i; \cK}=\tD_{1; \cK}\circ \tD_{i-1; \cK}=\left(\tD_{1; \cK}\right)^i,
\end{equation}
so that $\tD_{1; \cK}$ can be seen as a derivative operator.

We now give an elementary  variant of the Taylor-Lagrange
expansion using the previously defined Riemannian metric and covariant
derivatives. Its proof can be found in the Appendix, Section~\ref{sec:proof-sec4}.
\begin{lem}
	Assume $g_\cK$ is  positive and of class $\cc^1$. Let  $f$ be a function defined on $\Theta$
	taking values in an Hilbert space of class $\cc^2$. Setting $f^{[i]}=\tD_{i; \cK}
	[f]$ for $i\in \{1, 2\}$, 
	we have that for all $\theta, \theta_0\in \Theta$:
	\begin{equation}
	\label{eq:expansion}
	f(\theta) = f(\theta_0) +
	\operatorname{sign}(\theta-\theta_0)\,\dK(\theta,\theta_0)\,
	f^{[1]}(\theta_0)
	+  \dK(\theta,\theta_0)^2\,
	\int_0^1 (1-t) f^{[2]}(\gamma_t)\, \rd t,
	\end{equation}
	where  $\gamma$ is  a  geodesic  path  such that  $\gamma_0 =  \theta_0$, $\gamma_1=  \theta$ (and
	thus $\dK(\theta,\theta_0) = \mathcal{L}_\cK(\gamma)$).
	\label{lemma:expansion}
\end{lem}

{For a real-valued function $F$ defined on $\Theta^2$, we say that $F$ is
	of class $\cc^{0,  0}$ on $\Theta^2$ if it is  continuous on $\Theta^2$,
	and of class $\cc^{i, j}$ on $\Theta^2$,  with $i, j\in \N$, as soon as:
	$F$  is of  class  $\cc^{0, 0}$,  and  if $i\geq  1$  then the  function
	$\theta\mapsto F(\theta, \theta')$  is of class $\cc^i$  on $\Theta$ and
	its derivative $\partial_x F$ is  of class $\cc^{i-1, j}$ on $\Theta^2$,
	and if $j\geq 1$ the function $\theta' \mapsto F(\theta, \theta')$ is of
	class $\cc^j$ on $\Theta$ and its  derivative $\partial_y F$ is of class
	$\cc^{i, j-1}$ on $\Theta^2$.
	For a  real-valued  symmetric function $F$ defined
	on $\Theta^2$ of class $\cc^{i,j}$, we define the
	covariant  derivatives $D_{i,j;\cK}[F]$ of
	order  $(i,j)\in \N^2$    recursively by  $D_{0,0;\cK}[F] =  F$ and  for
	$i,j\in \N$, assuming that $g_\cK$ is of class $\cc^{\max(i,j)}$, and $\theta, \theta'\in \Theta$:}
\begin{equation}
\label{eq:def-cov-deriv-2}
D_{i+1,j;\cK}[F](\theta,\theta')
= g_\cK(\theta)^{\frac{i}{2}}\partial_{\theta} \left(
\frac{D_{i,j;\cK}[F](\theta, \theta')
}{g_\cK(\theta)^{\frac{i}{2}}}\right)
\quad\text{and}\quad
D_{i,j;\cK}[F](\theta, \theta')=D_{j,i;\cK}[F](\theta', \theta). 
\end{equation}
In particular, we have $D_{0, 0; \cK}[F]=F$, $D_{1, 0; \cK}=\partial_x
F$, $D_{0, 1; \cK}=\partial_y
F$ and $D_{1,1; \cK}=\partial^2_{xy}  F$. 
We shall also consider the following modification of the covariant
derivative, for $i,j\in \N$:
\begin{equation}
\label{eq:def-tD2}
\tD_{i,j;\cK}[F](\theta, \theta')=\frac{D_{i,j; \cK}[F](\theta,
	\theta')}{g_\cK(\theta)^{i/2}\, g_\cK(\theta')^{j/2} }\cdot
\end{equation}
We have $\tD_{1,0; \cK} \circ \tD_{0,1; \cK} =\tD_{0,1; \cK} \circ
\tD_{1,0; \cK} $ and  for
$i,j\in \N$, assuming that $g_\cK$ is of class $\cc^{\max(i,j)}$:
\begin{equation*}
\tD_{i,j; \cK}=\left(\tD_{1,0; \cK}\right)^i\circ \left(\tD_{0,1; \cK}\right)^j.
\end{equation*}

For $i, j\in \N$, if $\cK$  is of class $\cc^{i\vee 1,j\vee 1}$, we consider
the real-valued function defined on $\Theta^2$ by:
\begin{equation}
\label{eq:def-deriv-K}
\cK^{[i, j]}=\tD_{i,j;\cK}[\cK].
\end{equation}
In particular, since $\cK$ is of class $\cc^2$, we have:
\begin{equation}
\label{eq:K-g}
\cK^{[0, 0]}=\cK 
\quad\text{and}\quad
\cK^{[1, 1]}(\theta, \theta)=1.
\end{equation}

\subsection{The kernel and the Riemannian metric associated to the
	dictionary of features}
\label{sec:Rieman}
Let $T\in  \N$ be fixed and assume that Assumption \ref{hyp:g>0}
holds. We define the kernel $\cK_T$ on $\Theta^2$ by: 
\begin{equation}
\label{eq:def-KT}
\cK_T(\theta, \theta')=\langle \phi_{T}(\theta), \phi_{T}(\theta')
\rangle_T =\frac{\langle \varphi_T(\theta), \varphi_T(\theta')
	\rangle_T}{\norm{\varphi_T (\theta)}_T\norm{\varphi_T (\theta')}_T} ,
\end{equation}
where we recall  that $\phi_{T}={\varphi_T}/{\norm{\varphi_T}_T}$.  When
considering the  kernel $\cK_T$, we  shall write $g_T$  for $g_{\cK_T}$,
and similarly we shall use  the notations   $\tD_{i;T}$ and $\tD_{i,j;  T}$ instead of $\tD_{i;\ck_T}$ and $\tD_{i,j;  \ck_T}$.
Recall the derivatives of the kernel $\cK_T$ defined
by~\eqref{eq:def-deriv-K}. 
The next
lemma  insures in  particular that  the two  definitions of  $g_T$ given
by~\eqref{def:g_T} and~\eqref{eq:def-gK} are consistent, that is:
\begin{equation}
\label{eq:formula-gT-2}
g_T(\theta)=\partial^2_{xy} \cK_T(\theta, \theta)= \norm{\partial_\theta
	\phi_T(\theta)}_T^2.
\end{equation}
The proof of the next lemma can be found in the Appendix, Section~\ref{sec:proof-sec4}.
\begin{lem}
	\label{lem:gT_consistent}
	Let  $T\in \N$  be fixed  and assume  that Assumptions~\ref{hyp:reg-f}
	and~\ref{hyp:g>0}  hold.  Then,  the  symmetric kernel  $\cK_T$ is  of
	class $\cc^{3,3}$ on $\Theta^2$ and for $i, j\in \{0, \ldots, 3\}$ and
	$\theta, \theta'\in \Theta$, we have:
	\begin{equation}
	\label{def:derivatives_kernel}
	\cK_T^{[i,j]}(\theta,\theta') = \langle
	\tD_{i;T}[\phi_{T}](\theta), \tD_{j;T}[\phi_{T}](\theta')
	\rangle_T. 
	\end{equation}
	We also have:
	\begin{equation}
	\label{eq:formula_K_00}
	\sup_{\Theta^2} |\cK_T^{[0,0]}|\leq 1,
	\quad  \cK_T^{[0,0]}(\theta,\theta)=1,
	\quad   \cK_T^{[1,0]}(\theta,\theta) = 0, \quad 
	\cK_T^{[2,0]}(\theta,\theta) = -1 \quad \text{and} \quad
	\cK_T^{[2,1]}(\theta,\theta) = 0. 
	\end{equation}
\end{lem}

\section{Approximating the kernel associated to the dictionary}
\label{sec:cv_kernel}

In the section we detail  the assumptions guaranteeing the approximation
of  the kernel  $\cK_T$ (which  is usually  difficult to  compute) by  a
kernel  $\cK_\infty $  (which is  easier to  handle).  Both  kernels are
defined on  $\Theta^2$, however, we  shall qualify the  approximation of
$\cK_T$ by $\cK_\infty  $ and properties of $\cK_\infty $  on subsets of
$\Theta$, respectively $\Theta_T$ (which will be a compact interval) and
$\Theta_\infty $ (which will be an interval possibly unbounded).  We use
notations  from   Section~\ref{sec:riemannian_metric}  and   recall  the
definition of $g_\cK$, resp.  $\cK^{[i,j]}$, given in~\eqref{eq:def-gK},
resp. in~\eqref{eq:def-deriv-K}.  
Assuming the kernel $\cK$  is of class $\cc^{3,3}$ and using the notation~\eqref{eq:def-deriv-K}, we
also set for $\theta\in \Theta$:
\begin{equation}
\label{eq:def-h_K}
h_\cK(\theta)=\cK^{[3,3]}(\theta, \theta).
\end{equation}
For simplicity, for an  expression $A$ we write $A_*$  for $A_{\cK_*}$ where
$*$ is equal to $T$ or $\infty $. We first give a regularity
assumption on the kernel $\cK_\infty $.

\begin{hyp}[Properties of the asymptotic kernel $\cK_\infty$ and function $h_{\infty}$]
	\label{hyp:Theta_infini}
	The symmetric kernel $\cK_\infty $ defined on $\Theta^2$ is of 
	class $\cc^{3,3}$,   the  function
	$g_\infty $  defined by~\eqref{eq:def-gK} on $\Theta$  is positive and
	locally bounded (as well as of class $\cc^2$), and we have $\cK_\infty (\theta,\theta)=- \ck_\infty
	^{[2, 0]}(\theta, \theta)=1$ for $\theta\in \Theta$. The set  $\Theta_\infty \subseteq \Theta$ is  an
	interval and we have:
	\begin{equation}
	\label{def:M_h_M_g}
 L_3:=
	\sup_{\Theta_\infty } \, h_{\infty} < + \infty,
	\quad\text{and}\quad	L_{i,j}: = \sup_{\Theta_{\infty}^2} |\cK_\infty ^{[i, j]}|
	<+\infty
	\quad\text{for all  $ i,j \in \{0, 1, 2\}$}.
	\end{equation}
\end{hyp}

Since $\Theta_T$ is compact, under
Assumptions~\ref{hyp:g>0} and~\ref{hyp:Theta_infini}, we deduce that the constant $\RT$
below is positive and finite, where:
\begin{equation}
\label{eq:def-rho}
\RT=\max \left(\sup_{\Theta_T} \sqrt{\frac{g_T}{g_\infty
}},\sup_{\Theta_T} \sqrt{\frac{g_\infty }{g_T }} \right) . 
\end{equation}
From the definition of the Riemannian metric given
in~\eqref{eq:def-Riemann-dist-v2}
(see also~\eqref{eq:def-Riemann-dist}), we readily deduce that  the metrics $\dT$ and $\dI$ are then strongly equivalent on
$\Theta_T$;  more precisely we have that on $\Theta_T^2$:
\begin{equation}
\label{eq:equi-dT-dI}
\inv{\RT}\,  \dI \leq  \dT \leq  \RT\,  \dI. 
\end{equation}

\medskip

We then give an assumption on the quality of approximation of $\cK_T$ by
$\cK_\infty $.
We set:
\begin{equation}
\label{def:V_1}
\DT=\max( \DT^{(1)}, \DT^{(2)})
\quad\text{with}\quad
\DT^{(1)}=\max_{i,j\in \{0, 1, 2\} }\, \sup_{\Theta_T^2} |
\cK_T^{[i,j]} - \cK_\infty ^{[i,j]}|
\quad\text{and}\quad
\DT^{(2)}=\sup_{\Theta_T} |h_T - h_\infty |. 
\end{equation}

Let us recall that Assumption~\ref{hyp:g>0} implies regularity conditions
on $\cK_T$, see Lemma~\ref{lem:gT_consistent}. 
\begin{hyp}[Quality of the approximation]
	\label{hyp:close_limit_setting}
	Let $T \in \N$ be  fixed. Assumptions~\ref{hyp:g>0} 
	and~\ref{hyp:Theta_infini} hold, the interval  $\Theta_T\subset
	\Theta_\infty $ is a compact interval,  and  we have:
	\[
	\DT\leq   L_{2,2} \wedge L_3.
	\]
\end{hyp}

Notice     that     if      Assumption~\ref{hyp:g>0}     holds,     then
Assumptions~\ref{hyp:Theta_infini}     and~\ref{hyp:close_limit_setting}
hold  trivially when one  takes  $\cK_\infty  =\cK_T$ and $\Theta_\infty=\Theta_T$; notice also  that $\rho_T=1$ in this case.  In the  next example, the
sequence of kernels $(\cK_T, T\in \N)$ converges to the kernel $\cK_\infty $ as
$T$ goes to infinity, so that   Assumption~\ref{hyp:close_limit_setting} holds for $T$ large enough.

\begin{example}
	\label{example:compact_support}
	We consider the discrete-time example from Section~\ref{sec:exple-lambda}.  We assume that the process $y$ is
	a function  defined on $[0,  1]$ which,  for $T\in \N^*$  is observed
	through  the  regular grid  $\{t_{j,  T}=j/T\,  \colon\, 1\leq  j\leq
	T\}$. The  process $y$  is seen  as an element  of the  Hilbert space
	$H_T=L^2(\lambda_T)$,      with      the     probability      measure
	$\lambda_T=  \Delta_T \sum_{j=1}^T  \delta_{t_{j,T}}$  on  $[0, 1]$  with
	$\Delta_T=1/T$.  Let $\Theta$  be a compact interval of  $\R$ and set
	$\Theta_T=\Theta_\infty    =\Theta$.      Consider    a    dictionary
	$(\varphi(\theta), \theta\in  \Theta)$ independent  of $T$,  that is,
	$\varphi_T=\varphi$ for all $T\in \N^*$, and assume that the function
	$(\theta,    t)\mapsto    \varphi(\theta)(t)$     is    defined    on
	$\Theta\times  [0, 1]$  and of  class $\cc^{3,  0}$. Assume  that the
	dictionary     satisfies     the    regularity     assumptions     of
	Assumption~\ref{hyp:g>0}.
	
	Let $\Leb$ be the Lebesgue measure on $[0, 1]$, so that $(\lambda_T,
	T\in \N^*)$ converges weakly to $\Leb$. Then,  define the kernel
	$\cK_\infty $ by~\eqref{eq:def-KT} with $\varphi_T$ replaced by
	$\varphi$ (as the dictionary does not depend on $T$) and the scalar
	product $\langle \cdot, \cdot \rangle_T$ by the usual scalar product on
	$L^2(\Leb)$. Thanks to Lemma~\ref{lem:gT_consistent}, we deduce that 
	Assumption~\ref{hyp:Theta_infini} on the properties of $\cK_\infty $
	is satisfied. Using the weak convergence of $(\lambda_T, T\in \N^*)$
	to $\Leb$, we deduce that $\lim_{T\rightarrow \infty } \partial_x^i
	\partial_y^j \cK_T=\partial_x^i
	\partial_y^j
	\cK_\infty $ uniformly on $[0, 1]^2$ for all $i, j\in \{0,
	\dots, 3\}$. 
	This implies that:
	\[
	\lim_{T\rightarrow  \infty }  \DT=0
	\quad\text{and}\quad
	\lim_{T\rightarrow       \infty      }       \RT=1.  
	\]
	Thus Assumption~\ref{hyp:close_limit_setting}
	holds for $T$ large enough. 
\end{example}


\section{Certificates}
\label{sec:certificates}

In this section, we make assumptions  on the existence of functions from
$\Theta$ to $\R$ called certificates. These functions have interpolation
properties   that  are   corner   stones  in   the   proof  of   Theorem
\ref{maintheorem}.  The  term  ``certificate''  is  inherited  from  the
compressed sensing field were such functions were used to get rid of the
Restricted Isometry Property condition  (RIP) for exact recontruction of
signals   (see  \cite{candes2005decoding}   for  details   on  the   RIP
condition). In  \cite{candes2011probabilistic}, the authors  showed that
is possible to reconstruct exactly a sparse signal from the observations
of  a  finite  number  of  Fourier coefficients  by  exhibiting  a  dual
certificate. Many papers have followed  this line of research since then
(see \emph{e.g.} \cite{candes2013super,candes2014towards, duval2015exact}).

In sparse linear models the bounds for prediction error  are proved
using RIP, Restricted Eigenvalue or compatibility conditions
(see \cite{bickel2009simultaneous, van2016estimation}). Among these
assumptions, the compatibility conditions are the less
restrictive. Indeed, the authors of \cite{van2009conditions} have shown
that it is implied by both the RIP and the
Restricted Eigenvalue. However, in many contexts even the weaker
condition fails to hold. Typically the compatibility condition fails to
hold in the context of super-resolution which aims at  extracting the
frequencies and amplitudes of a linear combination of complex
exponentials from a small number of noisy time samples
(see~\cite{boyer2017adapting}). 

In  the  papers  \cite{boyer2017adapting} and  \cite{tang2014near},  the
authors achieve  nearly optimal  rates for the  prediction error  in the
super-resolution framework using certificate functions. Their method and
proof  are however  quite  specific to  complex  exponentials and  their
certificates  are trigonometric  polynomials.  The  insightful paper  of
\cite{duval2015exact} builds certificates in a quite general setting for
a one dimensional parameter  set $\Theta$. In \cite{de2019supermix}, the
authors  exhibit  certificate  functions   to  deal  with  more  general
probability density  models where $\Theta$ is  multidimensional. However
they       are      restricted       to      translation       invariant
dictionaries~\eqref{eq:trans-inv}.   The most  general framework  has been
introduced in  \cite{poon2018geometry} where the Riemannian  geometry is
key to  build in a natural  way the so-called certificate  functions. In
fact a separation distance between the parameters to estimate is needed
to build certificates and the Euclidean metric yields overly pessimistic
minimum  separation  condition. In  what  follows  we
introduce new  certificates, called derivative certificates,  in order
to control the prediction error.

\medskip

We   consider    the   following    assumption   in   the    spirit   of
\cite{poon2018geometry}.  We  consider  the  setting where  $T$  may  be
finite.    Let   $T\in   \N$,   $H_T$    be   an   Hilbert   space   and
$(\varphi_T(\theta),   \theta\in   \Theta)$  a   dictionary   satisfying
Assumptions~\ref{hyp:reg-f}  and~\ref{hyp:g>0},   so  that   the  kernel
$\cK_T$ is of  class $\cc^{3, 3} $ on $\Theta^2$.  Recall the Riemannian
metric $\dKT$ associated to $\cK_T$, which we simply denote by $\dT$.
We define  the closed ball centered at $\theta\in
\Theta_T$ with radius $r$ by:
\[
\mathcal{B}_T(\theta,r) = \left\{ \theta' \in \Theta_T,\,
\dT(\theta, \theta') \leq r\right\}  \subseteq \Theta_T.
\]
Let $\cq^{\star}$ be a  subset of $\Theta_T$  of cardinal
$\sparse$. For  $r>0$, the near  region of  $\cq^\star$ is the  union of
balls
$\bigcup_{\theta^\star\in  \cq^{\star}} \mathcal{B}_T(\theta^{\star},r)$
and   its    far   region   is   the   complementary of the near region  in   $\Theta_T$:
$\Theta_T      \setminus      \bigcup_{\theta^\star\in      \cq^{\star}}
\mathcal{B}_T(\theta^{\star},r)$.   Sufficient conditions  for the  next
assumption       to      hold       are      given       in      Section
\ref{sec:conditons_certificates}.

\begin{hyp}[Interpolating certificate]
	\label{assumption1}
	Let $T \in \N$, $s\in \N^*$, $r>0$ and $\cq^{\star}$ be a subset of $\Theta_T$ of cardinal
	$\sparse$. Suppose Assumptions~\ref{hyp:reg-f}  and~\ref{hyp:g>0} on
	the dictionary $(\varphi_T(\theta), \theta\in \Theta)$, 
	and Assumption~\ref{hyp:Theta_infini} on the kernel $\ck_\infty $,
	defined on $\Theta^2$,   hold.  Suppose that $
	\mathfrak{d}_T(\theta,\theta') > 2r$ for all $\theta, \theta' \in
	\cq^\star \subset \Theta_{T}$, and that there exist finite positive constants $C_{N}  , C_{N}', C_{F} $,
	$C_{B}$,  with $C_F < 1$, depending on $r$
	and $\cK_\infty$ such that for  any  application $v: \cq^\star \rightarrow \{ -1,1\} $ there exists  an element $p \in H_T$ satisfying:
	\begin{propenum} 
		\item\label{it:as1-<1}
		For all $\theta^\star \in \cq^{\star}$ and $\theta\in
		\mathcal{B}_T(\theta^{\star},r)$, we have
		$ | \langle \phi_{T}(\theta),p  \rangle_T| \leq 1 -
		C_{N}\, \mathfrak{d}_T(\theta^{\star},\theta)^{2}$.
		\item\label{it:as1-ordre=2}
		For all $\theta^\star \in \cq^{\star}$ and $\theta\in
		\mathcal{B}_T(\theta^{\star},r)$, we have
		$ | \langle\phi_{T}(\theta),p \rangle_T - v(\theta^\star)| \leq
		C_{N}' \, \mathfrak{d}_T(\theta^{\star},\theta)^{2}$.
		\item\label{it:as1-<1-c}  For all $\theta$ in $\Theta_T$ and $\theta  \notin \bigcup\limits_{\theta^\star
			\in \cq^{\star}} \mathcal{B}_T(\theta^{\star},r)$ (far region), we
		have $|\langle\phi_{T}(\theta),p \rangle_T| \leq 1 - C_{F}$.
		\item \label{it:norm<c} We have $\norm{p}_{T} \leq C_{B} \,\sqrt{\sparse} $.
	\end{propenum}
\end{hyp}

The function $\eta : \theta \mapsto \langle\phi_{T}(\theta),p \rangle_T$
is the so-called  ``interpolating certificate'' of the  function $v$, as
thanks to $\ref{it:as1-ordre=2}$  with $\theta=\theta^\star$, the function
$\eta$ coincides with the function $v$ on $\cq^{\star}$. In addition, the
interpolating certificate is required 
to have curvature  properties in the near region and to be  bounded by a
constant strictly inferior  to $1$ in the far region.   When $r$ is
sufficiently    small (that    is,   $r    \leq   \sqrt{2/(C_N+C_N')}$)
Conditions $\ref{it:as1-<1}$ and $\ref{it:as1-ordre=2}$ are equivalent to the fact that
the  function $\eta$  is in-between two quadratic  functions in the near
region of $\cq^{\star}$: for all $\theta^\star \in \cq^{\star}$ such that
$v(\theta^\star) =1$  (resp. $v(\theta^\star) =-1$)  and $\theta\in
\mathcal{B}_T(\theta^{\star},r)$, we have 
$ 1- C_{N}' \, \mathfrak{d}_T(\theta^{\star},\theta)^{2} \leq
\eta(\theta) \leq 
1 - C_{N} \, \mathfrak{d}_T(\theta^{\star},\theta)^{2}$
(resp. 	$ -1 + C_{N} \, \mathfrak{d}_T(\theta^{\star},\theta)^{2} \leq
\eta(\theta) \leq -1 + C_{N}' \,
\mathfrak{d}_T(\theta^{\star},\theta)^{2}$). 

\medskip

Sufficient conditions  for the  next
assumption       to      hold       are  also    given       in      Section
\ref{sec:conditons_certificates}.

\begin{hyp}[Interpolating derivative certificate]
	\label{assumption2}
	Let $T \in \N$, $s\in \N^*$, $r>0$ and $\cq^{\star}$ be a subset of $\Theta_T$ of cardinal
	$\sparse$.  Suppose Assumptions~\ref{hyp:reg-f}  and~\ref{hyp:g>0} on
	the   dictionary   $(\varphi_T(\theta),    \theta\in   \Theta)$,   and
	Assumption~\ref{hyp:Theta_infini} on the kernel $\ck_\infty $, defined
	on  $\Theta^2$,  hold.   Assume  that
	$     \mathfrak{d}_T(\theta,\theta')      >     2r$      for     all
	$\theta,  \theta' \in  \cq^\star  \subset  \Theta_{T}$ and that there  exist finite positive constants 
	$c_{N},  c_{F} $,  $c_{B}$   depending  on $r$  and
	$\cK_\infty$,   such  that  for   any  application
	$v:  \cq^\star  \rightarrow  \{  -1,1\}  $  there  exists  an  element
	$q \in H_T$ satisfying:
	\begin{propenum}
		\item\label{it:as2-ordre=2} For  all $\theta^\star \in  \cq^{\star}$ and
		$\theta\in \mathcal{B}_T(\theta^{\star},r)$, we have:
		\begin{equation*}
		\left | \langle\phi_{T}(\theta),q \rangle_T - v(\theta^\star)\, \operatorname{sign}(\theta-\theta^{\star})\mathfrak{d}_T(\theta,\theta^{\star})\right | \leq
		c_{N}\, \mathfrak{d}_T(\theta^{\star},\theta)^{2}.
		\end{equation*}
		
		\item\label{it:as2-<1-c}   For   all    $\theta$   in   $\Theta_T$   and
		$\theta    \notin   \bigcup\limits_{\theta^\star    \in   \cq^{\star}}
		\mathcal{B}_T(\theta^{\star},r)$     (far     region),     we     have
		$|\langle\phi_{T}(\theta),q\rangle_T| \leq c_{F}$.
		\item \label{it:as2-<c} $||q||_{T} \leq c_{B} \,\sqrt{\sparse}$.
	\end{propenum}
\end{hyp}

The function $\theta \mapsto \langle\phi_{T}(\theta),q \rangle_T$ will
be called an ``interpolating derivative certificate'' as it  vanishes
on $\cq^{\star}$. In addition, this function is required to decrease
similarly to  the function $\mathfrak{d}_T(\cdot ,\theta^{\star})$ near
$\theta^{\star}$ and to be bounded in the far region of  $\cq^{\star}$.


\section{Sufficient conditions for the existence of  certificates}
\label{sec:conditons_certificates}

In this section, we prove the  existence of the certificate functions of
Assumptions~\ref{assumption1}  and~\ref{assumption2}  provided that  the
parameters to  be estimated are  sufficiently separated in terms  of the
Riemannian   metric.    According  to   \cite{tang2015resolution},   the
separation condition cannot be avoided  to build certificate functions in
general.  It is however possible  to remove this separation condition in
some  particular  cases, see  \cite{schiebinger2018superresolution}  for
models with positive amplitudes.

In  order  to  find  sufficient  conditions for  the  existence  of  the
interpolating certificate functions  of Assumption~\ref{assumption1}, we
extend the construction from \cite{poon2018geometry} to a non asymptotic
setting. For  the existence of the  interpolating derivative certificate
functions of  Assumption~\ref{assumption2}, we  generalize the  proof of
\cite[Lemma~2.7]{candes2013super}  dedicated  to the dictionary  of  complex
exponential  functions.  The  proofs for  the existence  of certificates
given in Section~\ref{sec:proof_interpolating}  require boundedness and
local concavity properties of the kernel $\cK_T$. 
For practical application, they are deduced from the boundedness and local concavity
properties of the kernel $\cK_\infty $ and the quality of approximation
of $\cK_T$ by $\cK_\infty $ discussed in Section \ref{sec:cv_kernel}.

\subsection{Boundedness and local concavity of the kernel $\cK_T$}
In this work, we shall consider bounded kernels locally concave on the
diagonal. More precisely, for $T\in \bar \N=\N\cup\{\infty\}$ and $r>0$, we define:

\begin{align}
\label{eq:def-e0}
\varepsilon_{T}(r)
& = 1 - \sup \left\{ |\cK_{T}(\theta,\theta')|;
\quad \theta,\theta'\in \Theta_T  \text{ such that }
\mathfrak{d}_{T}(\theta',\theta) \geq r
\right\},\\
\label{eq:def-e2}
\nu_{T}(r)
&= -\sup \left\{
\cK_{T}^{[0,2]}(\theta,\theta'); \quad 
\theta,\theta'\in \Theta_T  \text{ such that }
\mathfrak{d}_{T}(\theta',\theta) \leq r \right\}.
\end{align}

The  fact that  $  \varepsilon_{T}(r)$ and  $  \nu_{T}(r)$ are  positive
depends  on  the function  $\varphi_T$,  the  space  $H_T$ and  the  set
$\Theta_T$.    Let us mention
that in many examples the positiveness of $ \varepsilon_{\infty}(r)$
and $ \nu_{\infty}(r)$ is easy  to check whereas the positiveness of
$\varepsilon_{T}(r)$ and $\nu_{T}(r)$ might be more difficult to prove.

Notice that \eqref{eq:formula_K_00} for $T\in  \N$ and
Assumption~\ref{hyp:Theta_infini} for $T=\infty $, 
and the continuity of $\cK_T$ and
$\cK_T^{[0, 2]}$  give that:
\begin{equation}
\label{eq:lim_varespilon}
\lim_{r \rightarrow 0^{+}}\varepsilon_{T}(r) = 0 \quad \text{
	and } \quad \lim_{r \rightarrow 0^{+}} \nu_{T}(r) = 1. 
\end{equation}

Recall     $\RT$     and     $\DT$     defined     in~\eqref{eq:def-rho}
and~\eqref{def:V_1}.     The    next     lemmas     state    that     if
$  \varepsilon_{\infty}(r/\RT)$  (resp.    $  \nu_{\infty}(r  \RT)$)  is
positive and if  the approximation of $\cK_T$ by $\cK_\infty  $ is good,
\emph{i.e.}  $\mathcal{V}_T$  is  small, then  $ \varepsilon_{T}(r)$  (resp.
$ \nu_{T}(r)$) is also positive.

\begin{lem}
	\label{lem:comp_epsilon}
	Let  $T\in  \N$. Suppose  Assumptions~\ref{hyp:reg-f},  \ref{hyp:g>0} and
	\ref{hyp:Theta_infini} hold. Then we
	have for $r>0$:
	\[
	\varepsilon_{T} (r)\geq  \varepsilon_{\infty}(r/\RT)- \mathcal{V}_T \quad \text{and} \quad 
	\nu_{T} (r)\geq  \nu_{\infty}(r \RT)- \mathcal{V}_T. 
	\]
\end{lem}

\begin{proof}
	As Assumptions \ref{hyp:g>0} and
	\ref{hyp:Theta_infini} hold,  recall that   $
	\mathfrak{d}_{\infty}/\rho_T \leq  \mathfrak{d}_T \leq \rho_T
	\mathfrak{d}_{\infty} $ on $\Theta_T^2$, see   \eqref{eq:equi-dT-dI}.

	Let        ${\theta,\theta'\in        \Theta_T}$        such        that
	$\mathfrak{d}_T(\theta',\theta)       \geq       r$. We       have
	$\mathfrak{d}_{\infty}(\theta',\theta) \geq  r/\rho_T $. We  get from
	the definition of $\mathcal{V}_T$ that:
	\begin{equation*}
	\begin{aligned}
	|\cK_{T}(\theta,\theta')| & \leq |\cK_{\infty}(\theta,\theta')| + \mathcal{V}_T \leq 1 - \varepsilon_{\infty} (r/\rho_T) + \mathcal{V}_T.
	\end{aligned}
	\end{equation*}
	Then, use \eqref{eq:def-e0} to get $\varepsilon_{T} (r)\geq
	\varepsilon_{\infty}(r/\rho_T)- \mathcal{V}_T$. 
	We also  have $\mathfrak{d}_{\infty}(\theta',\theta)  \leq r  \rho_T$. We deduce that:
	\begin{equation*}
	-\cK_{T}^{[0,2]}(\theta,\theta') \geq -\ck_{\infty}^{[0,2]}(\theta,\theta') - \mathcal{V}_T \geq \nu_{\infty}(r  \rho_T) - \mathcal{V}_T.
	\end{equation*}
	Finally, using \eqref{eq:def-e2}, we obtain $\nu_{T} (r)\geq  \nu_{\infty}(r  \rho_T)- \mathcal{V}_T$.
\end{proof}
When   we   require   in   addition  of   the   assumptions   of   Lemma
\ref{lem:comp_epsilon}                                              that
$\varepsilon_{\infty}(r/\rho_T)    \wedge     \nu_{\infty}(r    \rho_T)>
\mathcal{V}_T\geq  0$,  then  we   have  $\varepsilon_{T}(r)  >  0$  and
$\nu_{T}(r) > 0$.

\subsection{Separation conditions for the non-linear parameters}
In what follows, we measure the interferences (or the overlap) between
the features  in the mixture through a quantity $\delta_T$ introduced in \cite{poon2018geometry} and defined below.
Let $T\in \bar \N$, $\delta > 0$ and $s \in \mathbb{N}^{*}$.
We define the set $\Theta_{T , \delta}^s\subset \Theta_T ^s$
of  vector of parameters of dimension $s\in
\N^*$ and  separation $\delta>0$ as:
\[
\Theta^s_{T , \delta}= \Big  \{  (\theta_1,\cdots,\theta_s) \in \Theta_T ^s\,
\colon\,  \mathfrak{d}_{T}(\theta_{\ell},\theta_{k}) >  \delta \text{  for
	all distinct } k, \ell\in \{1, \ldots, s\}   \Big \}.
\]
Using  the  convention  $\inf  \emptyset=+\infty $,
we set for $u > 0$: 
\begin{multline}
\label{eq:def-delta-rs}
\delta_T(u,s) = \inf \Big\{ \delta>0\, \colon\,   \max_{1\leq  \ell \leq
	s} \sum\limits_{k=1, k\neq
	\ell}^{s}
|\cK_{T}^{[i,j]}(\theta_{\ell},\theta_{k})|
\leq u \\  \text{ for all } (i,j) \in \{0,1\} \times \{0,1,2 \} \text{ and } 
(\theta_1,\cdots,\theta_s)  \in \Theta^s_{T , \delta} \Big\}.
\end{multline}
\medskip

The quantity  $\delta_T(u,s)$ is the  minimum distance (with  respect to
the Riemannian metric $\mathfrak{d}_{T}$) between $s$ parameters so that
the coherence of the associated dictionary is bounded by $u$. The notion
of   coherence    between   the   features   in    the   definition   of
$\delta_T(u,s)$ is quite similar to  the one used in compressed
sensing  (see  \cite[Section~5]{FoucartBook}).   A standard  problem  in
compressed sensing  is to retrieve  the vector $\beta^{\star}$  when the
multivariate  function  $\Phi_T(\vartheta^{\star})$   is  known  in  the
discrete        setting        of   
Section~\ref{sec:exple-lambda}.    In   this  framework,   the   matrix
$\Phi_{T}(\vartheta^{\star})$,   whose  rows   correspond  to   the  $K$
discretized  functions in  the dictionary, is known.   The coherence  is
defined                                                               as
$\underset{1      \leq     k\neq      \ell     \leq      K}{\max}     \,
|\cK_T(\theta_k^{\star},\theta_{\ell}^{\star})|$.  Usually,  the smaller
the   coherence,  the   easier   it  is   to   retrieve  the   parameter
$\beta^{\star}$.      The     Babel      function,     introduced     in
\cite{tropp2004greed}, is  even closer  to our  measure of  overlap.  We
refer to \cite{poon2018geometry} for a discussion on this function.

\begin{rem}[Rewriting the separation condition with operator norm]
	\label{rem:norme}
	We shall stress that the definition of $\delta_T$
	in~\eqref{eq:def-delta-rs} is related to the operator norm
	$\norm{\cdot}_\op$ associated to the $\ell_\infty $ norm on
	$\R^\sparse$. We restate~\eqref{eq:def-delta-rs} using this operator norm
	$\norm{\cdot}_\op$, and leave the interested reader to  check that another choice of  operator norm does not improve the bounds on the certificates.
	Let us define for $i,j=0, 1, 2$ (assuming the kernel
	$\cK_T$ is smooth enough) and  $ \vartheta = (\theta_1, \ldots,
	\theta_\sparse)\in \Theta_{T}^\sparse$ the $s\times s$ matrix:
	\begin{equation}
	\label{def:gamma-ii}
	\cK_T^{[i,j]}(\vartheta) = \left
	(\cK_T^{[i,j]}(\theta_{k},\theta_{\ell}) \right
	)_{1\leq k,\ell \leq s}. 
	\end{equation}
	Let $I$ be the  identity matrix of size $s\times s$.  For $i=0$ or $i=1$,
	since the diagonal coefficients  of $\cK_T^{[i,i]}(\vartheta)$ are equal
	to 1, see~\eqref{eq:K-g}, we get:
	\[
	\norm{I - \cK_T^{[i,i]}(\vartheta) }_{\op } = \underset{1\leq k \leq  s}{\max}
	\sum\limits_{\ell\neq k} |
	\cK_T^{[i,i]} (\theta_{k},\theta_{\ell})|.
	\]    
	Since   the   diagonal   coefficients   of   $\cK_T^{[1,0]}(\vartheta)$,
	$\cK_T^{[0,1]}(\vartheta)$    and    $\ck_T^{[1,2]}(\vartheta)    $    are    zero,
	see~\eqref{eq:formula_K_00}, we also get:
	\[
	\norm{\cK_T^{[1,0]}(\vartheta) }_{\op } = \underset{1\leq k \leq  s}{\max}
	\sum\limits_{\ell\neq k} |
	\cK_T^{[1,0]} (\theta_{k},\theta_{\ell})| 
	\quad\text{and}\quad
	\norm{\cK_T^{[1,2]}(\vartheta) }_{\op } = \underset{1\leq k \leq  s}{\max}
	\sum\limits_{\ell\neq k} |
	\cK_T^{[1,2]} (\theta_{k},\theta_{\ell})|  
	\]    
	and by symmetry, with $\norm{\cdot}_\op^*$ for the operator norm
	associated to the $\ell_1$ norm:
	\[
	\norm{\cK_T^{[0,1]}(\vartheta) }_{\op}
	= \norm{\cK_T^{[1,0]\top}(\vartheta) }_{\op}
	=\norm{\cK_T^{[1,0]}(\vartheta) }_{\op}^*
	=\underset{1\leq \ell \leq  s}{\max}
	\sum\limits_{k\neq \ell} |
	\cK_T^{[1,0]} (\theta_{k},\theta_{\ell})|
	= \underset{1\leq k \leq  s}{\max}
	\sum\limits_{\ell\neq k} |
	\cK_T^{[0,1 ]} (\theta_{k},\theta_{\ell})|.
	\]    
	Since      the    diagonal   coefficients   of
	$\cK_T^{[2,0]}(\vartheta)$ are equal to -1, see~\eqref{eq:formula_K_00}, we
	also get:
	\[
	\norm{I+\cK_T^{[2,0]}(\vartheta) }_{\op} = \underset{1\leq k \leq  s}{\max}
	\sum\limits_{\ell\neq k} |
	\cK_T^{[2,0]} (\theta_{k},\theta_{\ell})|. 
	\]    
	Thus, we have:
	\begin{equation}
	\label{eq:def-delta-rs2}
	\delta_T(u,s) =    \inf \Big\{ \delta>0\, \colon\, A_{T, \ell_\infty }(\vartheta)  \leq u,
	\vartheta  \in \Theta^s_{T , \delta} \Big\}, 
	\end{equation}
	where:
	\begin{multline}
	\label{eq:def-MT}
	A_{T, \ell_\infty }(\vartheta)=\\
	\max\left(
	\norm{I - \cK_T^{[0,0]}(\vartheta) }_{\op },
	\norm{I - \cK_T^{[1,1]}(\vartheta) }_{\op },
	\norm{I+\cK_T^{[2,0]}(\vartheta) }_{\op },
	\norm{\cK_T^{[1,0]}(\vartheta) }_{\op } , \norm{\cK_T^{[0,1]}(\vartheta) }_{\op },
	\norm{\cK_T^{[1,2]}(\vartheta) }_{\op }  
	\right).
	\end{multline}
\end{rem}

Lemma \ref{lem:approx-delta} below enables us to compare the separation distance at $T$ fixed and at the limit case where $T = +\infty$.
Recall that the constant $\rho_T$ is defined in \eqref{eq:def-rho}.
\begin{lem}
	\label{lem:approx-delta}
	Let    $T    \in   \bar    \N$    and    $s   \in    \N^*$.    Suppose
	Assumptions~\ref{hyp:reg-f},   \ref{hyp:g>0} and   \ref{hyp:Theta_infini}
	hold.  Then, for $u > 0$ and with:
	\begin{equation*}
	u_T(s) = u + (s-1) \mathcal{V}_T,
	\end{equation*}  we have:
	\[
	\delta_T(u_T(s), s) \leq  \rho_T \, \delta_\infty (u, s)
	\quad \text{and} \quad
	\Theta^s_{T , \rho_T \, \delta_\infty (u, s)} \subseteq \Theta^s_{T
		, \delta_T (u_T(s), s)}. 
	\]
\end{lem}

\begin{proof}
	Since Assumptions~ \ref{hyp:g>0} and   \ref{hyp:Theta_infini} hold, we have from  \eqref{eq:equi-dT-dI} that $ \mathfrak{d}_{T} \leq
	\rho_T\,     \mathfrak{d}_{\infty}
	$ on $\Theta_T^2$. Hence for any $\delta > 0$, we have the inclusion
	$\Theta^s_{T , \rho_T \, \delta} \subseteq \Theta^s_{\infty , \delta}$. In
	particular, we have for $u>0$ that  $\Theta^s_{T , \rho_T \, \delta_\infty
		(u, s)} \subseteq \Theta^s_{\infty , \delta_\infty (u, s)}$.  Using
	the triangle inequality and the definition of $\mathcal{V}_T$ in
	\eqref{def:V_1}, we have that  for $ (i,j) \in \{0,1\}
	\times \{0,1,2 \}$ and $
	(\theta_1,\cdots,\theta_s)  \in \Theta^s_{T}$:
	\begin{equation*}
	\begin{aligned}
	\sum\limits_{k=1, k\neq
		\ell}^{s}
	|\cK_{T}^{[i,j]}(\theta_{\ell},\theta_{k})|
	&\leq  \sum\limits_{k=1, k\neq
		\ell}^{s}
	\left (
	|\cK_{\infty}^{[i,j]}(\theta_{\ell},\theta_{k})|  + \mathcal{V}_T \right ).
	\end{aligned} 
	\end{equation*}
	Then, the inclusion $\Theta^s_{T , \rho_T \, \delta_\infty (u, s)} \subseteq \Theta^s_{\infty , \delta_\infty (u, s)}$  gives that for all $ (i,j) \in \{0,1\} \times \{0,1,2 \}$ and $ 
	(\theta_1,\cdots,\theta_s)  \in \Theta^s_{T , \rho_T \, \delta_\infty (u, s)} $:
	\begin{equation*}
	\begin{aligned}
	\sum\limits_{k=1, k\neq
		\ell}^{s}
	|\cK_{T}^{[i,j]}(\theta_{\ell},\theta_{k})|
	& \leq u + (s-1)\mathcal{V}_T.
	\end{aligned} 
	\end{equation*}
	With $u_T(s) = u + (s-1) \mathcal{V}_T$, we deduce that $\delta_T(u_T(s), s) \leq  \rho_T \, \delta_\infty (u, s)$, which proves the inclusion $\Theta^s_{T , \rho_T \, \delta_\infty (u, s)} \subseteq \Theta^s_{T , \delta_T (u_T(s), s)}$.
\end{proof}

\subsection{The interpolating certificates}

We    define     quantities    which    depend    on  $\cK_\infty$, $\Theta_\infty$ and on real parameters $r>0$
and $\rho \geq 1$:
\begin{equation}
\label{eq:def_H}
\begin{aligned}
H_{\infty}^{(1)}(r,\rho)
&= \inv{2} \wedge L_{2,0} \wedge L_{2,1}  \wedge
\frac{\nu_{\infty}(\rho r)}{10} \wedge
\frac{\varepsilon_{\infty}(r/\rho) }{10},\\ 
H_{\infty}^{(2)} (r,\rho)
&= \frac{1}{6} \wedge \frac{8\,\varepsilon_{\infty} (r/\rho) }{10(5 + 2
	L_{1,0})}  \wedge 
\frac{8\,\nu_{\infty}(\rho r)}{9 ( 2L_{2,0} + 2 L_{2,1} + 4)} ,
\end{aligned}
\end{equation} 
where  the constants involved are defined  in
\eqref{def:M_h_M_g}.     By      recalling     the      behaviors     of
$\varepsilon_{\infty}(r)$ and  $\nu_{\infty}(r)$ when  $r$ goes  down to
zero from \eqref{eq:lim_varespilon}, we have for $\rho \geq 1$:
\begin{equation*}
\lim_{r \rightarrow 0^{+}}H_{\infty}^{(1)}(r,\rho) = 0 \quad \text{  and }
\quad \lim_{r \rightarrow 0^{+}} H_{\infty}^{(2)}(r,\rho)  = 0. 
\end{equation*} 

We state the first  main result of this section whose  proof is given in
Section \ref{sec:proof_interpolating}.
\begin{prop}[Interpolating certificate]
	\label{prop:certificat_interpolating}
	Let  $T \in  \N$,  $s  \in  \N^{*}$, $\rho \geq 1$ and $r>0$. We  assume that: 
	\begin{propenum}
		\item \textbf{Regularity of the dictionary
			$\varphi_T$:} \label{hyp:theorem_certificate_regularity}
		Assumptions~\ref{hyp:reg-f} and~\ref {hyp:g>0}  hold.

		\item \label{hyp:theorem_certificate_concavity}
		\textbf{Regularity of the limit kernel $\ck_\infty$:}
		Assumption~\ref{hyp:Theta_infini} holds, we have $r \in \left (0,1/\sqrt{2L_{2,0}} \right )$, and also $\varepsilon_{\infty}(r/\rho)  > 0$ and $\nu_{\infty}(\rho r)   > 0$.
		\item \label{hyp:theorem_certificate_separation}
		\textbf{Separation of the non-linear parameters:}
		There exists $u_{\infty} \in \left
		(0,H_{\infty}^{(2)}(r,\rho) \right ) $ such that:
		\[
		\delta_{\infty}(u_{\infty},s) < + \infty.
		\]
		\item \label{hyp:theorem_certificate_metric}
		\textbf{Closeness of the
			metrics $\mathfrak{d}_T$ and  $\mathfrak{d}_\infty$:}
		We have $\rho_T \leq \rho$.
		\item \label{hyp:theorem_certificate_approximation}
		\textbf{Proximity of
			the kernels $\cK_T$ and $\cK_\infty $:} 
		\[
		\mathcal{V}_T \leq H_{\infty}^{(1)}(r,\rho)
		\quad\text{and}\quad  (s-1) \mathcal{V}_T \leq H_{\infty}^{(2)}(r,\rho)-u_{\infty}.
		\]
	\end{propenum}
	Then,  with the positive constants:
	\begin{equation}
	\label{eq:cstes-certif1}
	C_N=C_N(r) = \frac{\nu_{\infty}(\rho r)}{180},
	\quad
	C_N'=\frac{5}{8}L_{2,0} + \frac{1}{8}L_{2,1} + \frac{1}{2},
	\quad
	C_B = 2
	\quad\text{and}\quad
	C_F = C_F(r) = \frac{\varepsilon_{\infty}(r/\rho)}{10} \leq 1, 
	\end{equation}
	Assumption~\ref{assumption1} holds (with the same $r$) for any subset
	$\cq^{\star}=\{\theta^\star _i,  \, 1\leq i\leq s\}$ such that for all
	$\theta\neq \theta'\in \cq^{\star}$: 
	\[
	\dT(\theta, \theta')> 2 \, \max(r, \rho_T   \,
	\delta_{\infty}(u_{\infty},s))  .
	\]
\end{prop}
Note   that   $\ref{hyp:theorem_certificate_regularity}$  concerns   the
dictionary  $\varphi_T$,  $\ref{hyp:theorem_certificate_concavity}$  and
$\ref{hyp:theorem_certificate_separation}$     the      limit     kernel
$\cK_{\infty}$      and     the      set     of      parameters,     and
$\ref{hyp:theorem_certificate_metric}$                               and
$\ref{hyp:theorem_certificate_approximation}$   the   regime   for   the
parameters $s$  and $T$. 
\begin{rem}[On the assumptions of Proposition \ref{prop:certificat_interpolating} when $\cK_\infty $ = $\cK_T$]
In the setting where the limit kernel and the approximating kernel are equal, some assumptions in the proposition become less restrictive, without any changes to the proofs. If $\cK_\infty $ is chosen
equal     to     $\cK_T$,     then    $\mathcal{V}_T=0$ and $\rho_T=1$,      and     also
$\ref{hyp:theorem_certificate_metric}$                               and
$\ref{hyp:theorem_certificate_approximation}$ hold and $\rho$ can be
chosen equal to 1 and $u_\infty$ can be chosen equal to $H_{\infty}^{(2)}(r,1) $.   
\end{rem}

We now give the second main result of this section whose proof is given
in Section \ref{sec:proof_derivative}.

\begin{prop}[Interpolating derivative certificate]
	\label{prop:certificat2}
	Let $T \in  \N$ and  $s \in \N^{*}$.  We assume  that:
	\begin{propenum}
		\item   \label{hyp:theorem_certificate_2_regularity} \textbf{Regularity
			of the dictionary $\varphi_T$:}
		Assumptions~\ref{hyp:reg-f} and~\ref {hyp:g>0}  hold.

		\item \label{hyp:theorem_certificate_2_regularity_kernel}
		\textbf{Regularity of the limit kernel $\ck_\infty$:}
		Assumption~\ref{hyp:Theta_infini}  holds.
		
		\item  \label{hyp:theorem_certificate_2_separation} \textbf{Separation of the non-linear parameters:} There exists $u'_{\infty} \in (0, 1/6)$,
		such that:
		\[
		\delta_{\infty}(u'_{\infty},s) < + \infty. 
		\]
		\item  \label{hyp:theorem_certificate_2_approximation}
		\textbf{Proximity of
			the kernels $\cK_T$ and $\cK_\infty $:} 
		We have:
		\[
		\mathcal{V}_T \leq 1
		\quad\text{and}\quad 
		(s-1) \mathcal{V}_T + u' _\infty \leq 1/6.
		\]
	\end{propenum}
	Then,  with the positive constants:
	\begin{equation}
	\label{eq:cstes-certif2}
	c_N = \frac{1}{8}L_{2,0} + \frac{5}{8}L_{2,1}  + \frac{7}{8},
	\quad
	c_B=2 
	\quad\text{and}\quad
	c_F  = \frac{5}{4}L_{1,0} +\frac{7}{4} ,
	\end{equation}
	Assumption~\ref{assumption2} holds for  any $r>0$ and   any subset
	$\cq^{\star}=\{\theta^\star _i,  \, 1\leq i\leq s\}$ such that for all
	$\theta\neq \theta'\in \cq^{\star}$:
	\[
	\dT(\theta, \theta')> 2 \, \max(r, \rho_T   \,
	\delta_{\infty}(u'_{\infty},s))  .
	\]
\end{prop}

Let  us  briefly  indicate  how  the  certificates  are  constructed  in
Section~\ref{sec:proof_interpolating}   using   the  features   of   the
dictionary.   Let   $\alpha=(\alpha_1,   \ldots,   \alpha_s)$   and
$  \coeff=(\coeff_1, \ldots,  \coeff_s)$ be  elements of  $\R^s$.  Let
$p_{\alpha,\coeff}\in H_T$ be defined by:
\begin{equation*}
p_{\alpha,\coeff} = \sum\limits_{k=1}^{s} \alpha_{k}
\phi_{T}(\theta^{\star}_{k})
+ \sum\limits_{k=1}^{s} \coeff_{k} \, 
\phi^{[1]}_{T}(\theta^{\star}_{k}) ,
\end{equation*}
where $\phi_T^{[1]}$ denotes the derivative $\tD_{1;T}[\phi_T]$. Using      \eqref{def:derivatives_kernel}       in
Lemma~\ref{lem:gT_consistent}, set   the  interpolating  real-valued  function
$\eta_{\alpha,\coeff}$ defined on $\Theta$ by:
\[
\eta_{\alpha,\coeff}(\theta): =
\langle\phi_{T}(\theta),p_{\alpha,\coeff} \rangle_T =
\sum_{k=1}^{s} 
\alpha_{k}\,  \cK_T(\theta,\theta^{\star}_{k}) + \sum_{k=1}^{s}
\coeff_{k}\,  \cK_T^{[0,1]}(\theta,\theta^{\star}_{k}).
\]
By Assumption~\ref{hyp:g>0} on the regularity of $\varphi_T$ and the positivity of $g_T$ and
Lemma~\ref{lem:gT_consistent}, we get that the function $ \eta_{\alpha,\coeff} $ is of class
$\cc^3$  on  $\Theta$, and using~\eqref{eq:tDi}, we get that:
\[
\eta_{\alpha, \coeff}^{[1]} := \tD_{1; T} [ \eta_{\alpha, \coeff}](\theta)=
\sum_{k=1}^{s} 
\alpha_{k}\,  \cK_T^{[1, 0]}(\theta,\theta^{\star}_{k}) + \sum_{k=1}^{s}
\coeff_{k}\,  \cK_T^{[1,1]}(\theta,\theta^{\star}_{k}).
\]
We show  in Section~\ref{sec:proof_interpolating} that for  any function
$v:\cq^\star  \rightarrow \{  -1,1\}$  there  exists a  unique
choice of  $\alpha$ and $\coeff$ such  that $\eta_{\alpha,\coeff}$ becomes
an interpolating certificate, that is,
$  \eta_{\alpha,\coeff} =v$ and 
$   \eta_{\alpha,\coeff}^{[1]} = 0$
on $\cq^{\star}$,
and  $p_{\alpha,\coeff}$ satisfies Points $\ref{it:as1-<1}$-$\ref{it:norm<c}$ of
Assumption~\ref{assumption1}. 

Moreover,  for     any     function
$v:\cq^\star \rightarrow \{ -1,1\}$ there exists  another unique choice
of $\alpha$ and $\coeff$ such that $\eta_{\alpha,\coeff}$ is  an
interpolating derivative certificate, that is, 
$\eta_{\alpha,\coeff}=0$ and $\eta_{\alpha,\coeff}^{[1]}=v$ on $\cq^\star$,
and  $p_{\alpha,\coeff}$ satisfies Points $\ref{it:as2-ordre=2}$-$\ref{it:as2-<c}$ of
Assumption~\ref{assumption2}.

\section{Gaussian sparse spike deconvolution} \label{sec:example}

We develop here in full details  the particular example of a mixture of
Gaussian features observed in  a discrete regression model  with regular design.
In  particular,  we   check  the  numerous  but   not  very  restrictive
assumptions, and  we illustrate  that our  general and  more restrictive
sufficient conditions for the existence of certificates can turn simpler
and far less restrictive on concrete examples. The model is presented in
Section~\ref{sec:model-hyp}, where we also check the first assumptions.
The technical Section~\ref{sec:model-hyp2} on the existence of the
certificates allows to point out the separation distance
in~\eqref{eq:sep-param} and with the simpler expression
in~\eqref{eq:sep-param2}. This 
separation distance is usually very pessimistic, but one can rely on
numerical estimations to be more realistic, see
Remark~\ref{rem:sep-gauss} in this direction. Eventually, we apply to this context our main
Theorem~\ref{maintheorem} in Section~\ref{sec:predic-gauss} as
Corollary~\ref{cor:gaussian_example} and illustrate a particular choice
of the tuning parameter in Remark~\ref{rem:tau-choice} in the spirit of
\cite{tang2014near, boyer2017adapting} established for the specific
dictionary of  complex exponentials. 

\subsection{Model and first assumptions of Theorem~\ref{maintheorem}}
\label{sec:model-hyp}

Consider  a  real-valued process  $y$  observed  over a  regular  grid
$t_1<\cdots< t_T$ of a symmetric interval $[a_T, b_T]$, with $T\geq
2$, 
$b_t=-a_T>0$,  $t_j= a_T + j \Delta_T $ for $j=1, \ldots, T$ and grid step:
\[
\Delta_T=\frac{b_T -a_T}{T}\cdot
\]
Assuming  that all  the observations  have  the same  weight amounts  to
considering $y$ as an element  of the Hilbert space $H_T=L^2(\lambda_T)$ of
real valued functions defined in $\R$ and square integrable with respect
to the atomic measure $\lambda_T$ on $\{t_1, \ldots, t_T\}$:
\[
\lambda_T(\rd t)=\Delta_T\,  \sum_{j=1}^T \delta_{t_j}(\rd t).
\] 

We consider  a noise process $w_T(t)=\sum_{j=1}^T  G_j \ind_{\{t_j=t\}}$
for $t\in \R$, where $(G_1, \ldots,  G_T)$ is a centered Gaussian vector
such that, for some noise level $\sigma_1>0$:
\begin{equation*}
\label{eq:gaussian_noise}
\E[G_j^2] =\sigma_1^2
\quad\text{and}\quad
| \E[G_j G_i]| \leq \sigma_1^2/T
\quad\text{for $j\neq i$ in \{1,\ldots,T\}}. 
\end{equation*}
Thus,  the norm  of  the  noise $\norm{w_T}_T$  is  finite  almost
surely,    and    for   any    $f    \in    L^2(\lambda_T)$   we
have:
\[
\Var(\left \langle f,w_T \right \rangle_T) = \Var\Big(\Delta_T \sum_{j=1}^T
f(t_j) G_j\Big)  \leq   2\sigma_1^2 \Delta_T \norm{f}_T^2.
\]
Hence, Assumption
\ref{hyp:bruit} on the noise is satisfied with $\sigma^2= 2\sigma_1^2$. 
(Notice that if the random variables $G_1, \ldots,  G_T$ are independent, then
$ \Var(\left \langle f,w_T \right \rangle_T) =\sigma^2 \Delta_T
\norm{f}_T^2$ with $\sigma^2= \sigma_1^2$.) This gives that Point
$\ref{hyp:theorem1_point1}$ of Theorem~\ref{maintheorem} holds.
\medskip

We consider the dictionary given by the translation model of Section
\ref{sec:translation_dis_model} with Gaussian features and fixed  scaling
parameter $\sigma_0>0$, that is the dictionary does not depend on $T$
and is  given by:
\[
\Big( \varphi(\theta)= k\Big (\frac{\cdot-\theta}{\sigma_0} \Big ), \,
\theta \in \Theta\Big) 
\quad\text{with}\quad
k(t)=\expp{-t^2/2}
\quad\text{and}\quad
\Theta = \R.
\]
Thus,  the   signal  $\beta^\star   \Phi(\vartheta^\star)  $   in  model
\eqref{eq:model} can indeed be written as the convolution product of the
function $k$  and an  atomic measure.   It is  elementary to  check that
Assumption  \ref{hyp:reg-f} on  the  regularity of  the features  holds.
Furthermore,       the       functions       $\varphi(\theta)$       and
$\partial_\theta\varphi(\theta)$      are      linearly      independent
$\lambda_T-a.e$ for all $\theta\in \Theta$ as $T \geq 2$. Hence the
function $g_T$ is
positive on $\Theta$  by Lemma \ref{lem:g>0} and thus
Assumption~\ref{hyp:g>0} holds.
This gives that Point
$\ref{hyp:reg_dic_theorem}$ of Theorem~\ref{maintheorem} holds.
\medskip

We now define the limit kernel $\cK_\infty$. To do so, we shall
assume that $(b_T, T\geq 2)$ is a sequence of
positive numbers, such that:
\begin{equation}
\label{eq:lim-D-b_T}
\lim_{T\rightarrow \infty } b_T  = + \infty
\quad\text{and}\quad
\lim_{T\rightarrow \infty } \Delta_T = 0. 
\end{equation}
This   in   particular   implies   that   the   sequence   of   measures
$(\lambda_T, T\geq  2) $  converges with respect  to the  vague topology
towards    the    Lebesgue    measure,    say    $\lambda_\infty$,    on
$\Theta_\infty   =\R$.     We   also   consider   the    Hilbert   space
$H_\infty =  L^2(\lambda_\infty)$ endowed with its  usual scalar product
denoted $\langle  \cdot, \cdot  \rangle_\infty $ and  corresponding norm
denoted $\norm{\cdot}_\infty  $ (not  to be  confused with  the supremum
norm!).  Note that the kernel $\cK_T$ and the associated quantities such
as  $\varepsilon_T$   and  $\nu_T$  defined  in   \eqref{eq:def-e0}  and
\eqref{eq:def-e2},    respectively,   or    the   uniform    bounds   on
$\cK_T^{[i,j]}$, are difficult to  calculate. However the uniform bounds
on   $\Theta_\infty   =\R$   for  the   kernel   $\cK_\infty$,   defined
by~\eqref{eq:def-KT}  with  $T$  replaced   by  $\infty  $,  are  easily
computed.  Elementary calculations give for $\theta,\theta' \in \Theta$:
\[
\norm{\varphi(\theta)}_\infty ^2= \sqrt{\pi}\, \sigma_0 ,
\quad   \phi_\infty(\theta) =  \frac{1}{\pi^{\frac{1}{4}}
	\sqrt{\sigma_0}} \varphi(\theta), 
\quad
\cK_\infty(\theta,\theta') =  k\Big( \frac{\theta-\theta'}{\sqrt{2} \,
	\sigma_0}\Big)
\quad\text{and}\quad
g_{\infty}(\theta ) = \frac{1}{2\sigma_0^2}\cdot
\]
In particular, we have $ g_{\infty}'(\theta) = 0$. The Riemannian metric
is  equal to  the  Euclidean  distance up  to  a multiplicative  factor,
for all  $\theta,\theta'\in\Theta_\infty  =\R$:
\begin{equation}
\label{eq:R-met-infini}
\mathfrak{d}_{\infty}(\theta,\theta') = \frac{|\theta  - \theta'|}{ \sqrt{2}\,
	\sigma_0}\cdot
\end{equation}
We see  that
$\cK_\infty$ is of class $\mathcal{C}^{\infty ,\infty }$ and that:
\begin{equation}
\label{eq:def-K-gauss}
\cK_\infty^{[i,j]}(\theta, \theta')= (-1)^j \,
k^{(i+j)} \Big( \frac{\theta-\theta'}{\sqrt{2} \,
	\sigma_0}\Big)
\quad\text{and}\quad
k^{(i)}(t)=P_i(t)\, k(t),
\end{equation}
where we give for convenience
the formulae for some of the polynomials $P_i$:
\begin{align*}
&P_1(t)=-t ,
\quad
P_2(t)= -1+t^2,
\quad
P_3(t)= 3t -t^3, 
\quad
P_4(t)= 3 -6t^2 + t^4, 
\quad P_6(t)= -15+ 45t^2 - 15t^4 + t^6.
\end{align*}
Then,  we explicitly compute  the constants $L_{i,j}$ for $i,j \in \{0,\cdots,2\}$ and $L_3$ defined in \eqref{def:M_h_M_g}:
\begin{equation*}
\begin{aligned}
&m_g = (2\sigma_0^2)^{-1},
\quad L_{0,0} = 1,
\quad  L_{1,0} = L_{0,1}= \expp{-1/2},
\quad L_{1,1}  = L_{2,0}= L_{0,2} = 1,
\\
&L_{2,1}= L_{1,2}  = \sqrt{18-6\sqrt{6}}\expp{\sqrt{3/2}-3/2}\leq \sqrt{2},
\quad L_{2,2} = 3
\quad\text{and}\quad
L_3 = 15.
\end{aligned}
\end{equation*}
Notice the  constants $L_{i,j}$ and $L_3$  do not depend on  the scaling
factor $\sigma_0$.   Thus Assumption~\ref{hyp:Theta_infini}  holds. This
gives that Point $(iii)$ of Theorem~\ref{maintheorem} holds.  \medskip

We  now check  the proximity  of the  kernel $\cK_T$  to the  limit kernel
$\cK_\infty$.   The support  of $\lambda_T$  is spread  over the  window
$[a_T, b_T]$  where the signal  is observed.  Hence it is  legitimate to
look for the location parameters on a smaller  subset of this window, and thus
restrict the optimization \eqref{eq:generalized_lasso} to the compact set:
\begin{equation*}
\label{eq:opimization_set_gaussian}
\Theta_T = [(1-\epsilon) a_T, (1-\epsilon)  b_T] \subset [a_T,b_T]
\quad \text{with a given shrinkage $\epsilon \in (0, 1)$.}
\end{equation*} 
The proof of the next lemma is given in
Section~\ref{sec:proof-lem-gauss}. 
Recall  $\RT$ and $\DT$ defined in~\eqref{eq:def-rho}
and~\eqref{def:V_1}. 
Set:
\[
\gamma_T=2 \Delta_T \, \sigma_0^{-1}+ \sqrt{\pi}\,\expp{-\epsilon^2
	b_T^2/2\sigma_0^2}.
\]
\begin{lem}
	\label{lem:rates0}
	There exist finite positive universal constants $c_0$,
	$c_1$ and $c_2$, such that $\gamma_T<c_0$ implies: 
	\begin{equation}
	\label{eq:lem:rates0}
	\mathcal{V}_T \leq c_1 \gamma_T
	\quad \text{and} \quad
	|1 - \rho_T|\leq  c_2 \gamma_T. 
	\end{equation}
\end{lem}
This implies that Assumption~\ref{hyp:close_limit_setting} holds for $T$
such that  $\gamma_T\leq c_0 $ and  $c_1\gamma_{T}\leq 3$, which
holds for  $T$ large  enough thanks to~\eqref{eq:lim-D-b_T}.  Thus Point
$(iv)$  of   Theorem~\ref{maintheorem}  holds  for  $T$   large  enough.

\subsection{Existence of certificates}
\label{sec:model-hyp2}
We  keep the  model and  the notations from  Section~\ref{sec:model-hyp}. In
order  to get  the prediction  error from  Theorem~\ref{maintheorem}, we
only  need to  check that  Point~$\ref{hyp:V_T_theorem}$ therein  on the
existence  of the  certificates holds.   
To  check the  existence of  the
certificates,                 we                 can                 use
Propositions~\ref{prop:certificat_interpolating}
and~\ref{prop:certificat2}, and  check that all the  hypotheses required
in those two propositions hold. \medskip

We first concentrate on the hypotheses of
Proposition~\ref{prop:certificat_interpolating}.
Assumption~$\ref{hyp:theorem_certificate_regularity}$ on the regularity of
the dictionary holds, see Section~\ref{sec:model-hyp}.
\medskip

We       recall       that       $L_{0,       2}=1$       and       thus
$1/\sqrt{2 L_{0,  2}}=1/\sqrt{2}> 1/2$. 
Recall  $\varepsilon_{\infty}(r)$ and $\nu_{\infty}(r)$ defined
in~\eqref{eq:def-e0} and~\eqref{eq:def-e2}, and thanks to the explicit
form of the Riemannan metric, we get for $r\in (0, 1)$:
\[
\varepsilon_{\infty} (r) = 1 - \expp{-r^2/2} > 0
\quad\text{and} \quad
\nu_{\infty}(r) = \left ( 1 - {r}^{2}\right) 
\expp{-r^2/2}.
\]
This  and the regularity  of the kernel  $\cK_\infty   $  from
Section~\ref{sec:model-hyp}   imply  that 
Assumption~$\ref{hyp:theorem_certificate_concavity}$ holds for all $r\in
(0, 1/( \rho\vee \sqrt{2}))$.  \medskip

We obtain from~\eqref{eq:def-K-gauss} that $\lim_{q\rightarrow \infty }
\sup_{|\theta-\theta'|\geq q} |\cK_\infty ^{[i, j]}(\theta,
\theta')|=0$ for all $i,j\in \{0, 1, 2\}$. Thus, we deduce from the
definition~\eqref{eq:def-delta-rs} 
of $\delta_\infty $ that $\delta_\infty (u,s)$ is finite for all $s\in
\N^*$ and $u>0$.  This implies that
Assumption~$\ref{hyp:theorem_certificate_separation}$ on the separation
of the parameters holds.
\medskip

To  simplify,  we  set  $  \rho=2$  (but we  could  take  any  value  of
$\rho>1$).  We deduce  from Lemma~\ref{lem:rates0},  that for  $T$ large
enough          $\rho_T\leq          \rho=2$,          and          thus
Assumption~$\ref{hyp:theorem_certificate_metric}$  on  the closeness  of
the metrics $\mathfrak{d}_T$ and $\mathfrak{d}_\infty$ holds.  \medskip

Recall the definition of $H^{(1)}_\infty $ and $H^{(2)}_\infty $
from~\eqref{eq:def_H}. To get the smallest separation  distance, we also set:
\begin{equation}
\label{eq:val-r} 
r = \argmax_{0 <r' < 1/2}\,    H_{\infty}^{(2)}(r',\rho) \approx
0.49.
\end{equation}
Notice that the function is not  \emph{a priori} monotone in $\rho$.  We
have   $\varepsilon_{\infty}(r/2)    \approx   2.9    \times   10^{-2}$,
$\nu_{\infty}(2        r)\approx       3.7       \times       10^{-2}$,
$H_{\infty}^{(1)}(r,2)     \approx    2.9    \times    10^{-3}$     and
$H_{\infty}^{(2)}(r,2) \approx  3.7 \times  10^{-3}$. Again in  order to
get  a ``small''  separation distance,  we  choose $u_\infty  $ close  to
$H_{\infty}^{(2)}(r,2)$,    say
$u_\infty =  \eta_0 H_{\infty}^{(2)}(r,2)$ for some  $\eta_0<1$ close to
$1$.      For    simplicity     set     $\eta_0=9/10$.     Thanks     to
hypothesis~\eqref{eq:lim-D-b_T},                  we                 get
$\lim_{T\rightarrow  \infty  }  \gamma_T=0$  and  Lemma~\ref{lem:rates0}
implies that for  $T$ large enough, depending  on $\sigma_0$, $\epsilon$
and the sparsity parameter $s$, we have:
\begin{equation}
\label{eq:controle-gauss}
\rho_T \leq 2, \quad \mathcal{V}_T \leq H_{\infty}^{(1)}(r,2) 
\quad\text{and} \quad
(s-1) \mathcal{V}_T \leq (1-\eta_0) H_{\infty}^{(2)}(r,2), 
\end{equation}
and thus Assumption~$\ref{hyp:theorem_certificate_approximation}$ on the
proximity of the kernels $\cK_T$ and $\cK_\infty $ holds. 
\medskip

Thus,  the assumptions of
Proposition~\ref{prop:certificat_interpolating} are satisfied, and we
deduce that Assumption~\ref{assumption1} holds with, thanks
to~\eqref{eq:cstes-certif1}:
\[
C_N \approx 2 \times 10^{-4}, \quad
C_N'\approx 1.3, \quad
C_B = 2   \quad\text{and}\quad  C_F \approx 2.9 \times 10^{-3}.
\]
\medskip

We       now      concentrate       on      the       hypotheses      of
Proposition~\ref{prop:certificat2}.
Assumptions~$\ref{hyp:theorem_certificate_2_regularity}$-$\ref{hyp:theorem_certificate_2_separation}$
clearly       hold      for       the       same      reasons       as
Assumptions~$\ref{hyp:theorem_certificate_regularity}$-$\ref{hyp:theorem_certificate_separation}$
of  Proposition~\ref{prop:certificat_interpolating}.

Again in  order to
get  a ``small''  separation distance, there is no need to choose
$u'_\infty $ larger that $u_\infty $, and for this reason we take
$u'_\infty =u_\infty $. We deduce from~\eqref{eq:controle-gauss} that  for  $T$ large enough, depending  on $\sigma_0$, $\epsilon$
and the sparsity parameter $s$:
\[
\mathcal{V}_T \leq 1
\quad\text{and} \quad
(s-1) \mathcal{V}_T + u'_\infty \leq 1/6,
\]
and thus Assumption~$\ref{hyp:theorem_certificate_2_approximation}$ on the
proximity of the kernels $\cK_T$ and $\cK_\infty $ holds. 
\medskip

Thus,  the assumptions of
Proposition~\ref{prop:certificat2} are satisfied, and we
deduce, thanks
to~\eqref{eq:cstes-certif2}, that Assumption~\ref{assumption2} holds
with 
the same value of $r$ given by~\eqref{eq:val-r}:

\[
c_N \approx 1.9, \quad
c_B = 2 ,
\quad\text{and}\quad  c_F  \approx 2.6.
\]
\medskip

In conclusion, we get that Assumptions~\ref{assumption1}
and~\ref{assumption2} hold for $T$ large enough, and thus
Point~$\ref{hyp:existence_certificate_theorem}$ of
Theorem~\ref{maintheorem} holds for $T$ large enough and  $\cq^\star$
such that  for all
$\theta\neq \theta'\in \cq^{\star}$ the distance $\dT(\theta, \theta')$ is larger
than the separation distance:
\begin{equation}
\label{eq:sep-param}
2 \max(r,\, \rho_T   \,
\delta_{\infty}(u_{\infty},s), \,  \rho_T   \,
\delta_{\infty}(u'_{\infty},s)). 
\end{equation}
Notice       that       since        $u_\infty       =u'_\infty       $,
$\rho_T\dT(\theta,     \theta')\geq     \dI(\theta,    \theta')$     and
$\rho_T\leq 2$, we deduce  from~\eqref{eq:R-met-infini}, that a slightly
stronger condition is to assume that $|\theta-\theta'|$ is larger than:
\begin{equation}
\label{eq:sep-param2}
\sqrt{2}\, \sigma_0 \max (1, 4 \delta_{\infty}(u_{\infty},s)).
\end{equation}

\begin{rem}[On the separation distance~\eqref{eq:sep-param}]
	\label{rem:sep-gauss}
	The  separation  distance~\eqref{eq:sep-param}  is  a  non-decreasing
	function  of $s$.  We now  provide an   upper  bound. Let
	$(i,j)  \in \{0,1\}  \times \{0,1,2  \}$.  By  considering the  kernel
	$\cK_T$  and its  derivative given  by~\eqref{eq:def-K-gauss} and  the
	bound $M= \max_{0\leq i\leq 3}  \sup |P_i|\, \sqrt{k}$, we deduce that
	$  |\cK_{\infty}^{[i,j]}(\theta,\theta')| \leq  M \expp{-  \dI(\theta,
		\theta')^2/2}$  for  all  $\theta,\theta' \in  \Theta$.   We  easily
	obtain                             that                            for
	$\vartheta = (\theta_1,\cdots,\theta_s) \in \Theta^s_{\infty , \delta}
	$ with $\delta>0$:
	\[
	\max_{1\leq  \ell \leq s} \sum\limits_{k=1, k\neq
		\ell}^{s}      |\cK_{\infty}^{[i,j]}(\theta_{\ell},\theta_{k})|
	\leq  \psi_{s}(\delta)
	\quad\text{with}\quad
	\psi_{s}(\delta)=2M \int_0^{s/2 +1} \expp{-t^2\delta^2/4} \rd t.
	\]
	The  function $\psi_s$  is  decreasing and  one to  one  from $\R_+$  to
	$(0,M(s+2)]$. Setting $\psi_s^{-1}(u)  = 0$ for $u >  M(s+2)$, we deduce
	from~\eqref{eq:def-delta-rs} that for $u>0$:
	\[
	\delta_\infty(u,s) \leq   \psi_s^{-1}(u). 
	\]
	Since the map $\delta \mapsto \psi_s(\delta)$ is decreasing and the map $s\mapsto \psi_s(\delta)$ is increasing with limit
	$\psi_\infty (\delta)= 2\sqrt{\pi}\, M/\delta$, we deduce that for $s\in
	\N^*$:
	\[
	\delta_\infty(u,s) \leq \frac{ 2\sqrt{\pi}\, M}{u} ,
	\]
	so   that   the  separation   distance~\eqref{eq:sep-param}
	(or~\eqref{eq:sep-param2}) can be bounded uniformly in $s$ for given $r$
	and $u_\infty =u'_\infty $.

	\medskip
	
	In  fact,   we  shall  illustrate   for  $s=2$  that   the  separation
	distance~\eqref{eq:sep-param}  is largely  overestimated. We can compute $ \delta_{\infty}(u,s)$ thanks to its expression \eqref{eq:def-delta-rs2}. For  $s=2$
	and with the values chosen in  this section for
	$u_{\infty}=u'_{\infty}$,   we obtain $\delta_{\infty}(u_{\infty},2) \approx 4.5$. We  deduce    that    the    separation
	distance~\eqref{eq:sep-param} expressed with respect to the metric $\mathfrak{d}_T$ is approximately $9 \,  \rho_T$ (which gives $13 \,\sigma_0 \rho_T^2$ in terms of the Euclidean metric), which is
	unconveniently large.   However, a  detailed numerical approach  (using the
	very      certificates     provided      in      the     proof      of
	Propositions~\ref{prop:certificat_interpolating}
	and~\ref{prop:certificat2}) with $T$ large  so that the kernel $\cK_T$
	is   indeed   well   approximated   by   $\cK_\infty   $   (and   thus
	$\rho_T\approx 1$), gives  that one can take for  $s=2$ the separation
	distance with respect to the Euclidean metric equal  to $3.1 \times \, \sigma_0$ (that is approximately equal to $2.2$ with respect to the metric $\mathfrak{d}_\infty$),  which  is  much more  realistic.
	Therefore, the theoretical  separation distance~\eqref{eq:sep-param} is
	in general largely overestimated.
\end{rem}

\subsection{Prediction error}
\label{sec:predic-gauss}
We keep the model and the notations from Section~\ref{sec:model-hyp} and the
values   chosen   in   Section~\ref{sec:model-hyp2}.  We   deduce   from
Theorem~\ref{maintheorem} the following result.

\begin{corollary}
	\label{cor:gaussian_example}
	For  $T$  large enough,  depending  on  $\sigma_0$, $\epsilon$  and  the
	sparsity parameter  $s$, such  that \eqref{eq:controle-gauss}  holds and
	for                                                                  all
	$\theta\neq \theta'\in  \cq^\star= \{\theta^\star_k, \, k  \in S^\star\}
	$, with  $S^\star=\supp (\beta^\star)$  such that  $|\theta-\theta'|$ is
	larger  than the  separation  parameter $  \sqrt{2}\, \sigma_0 \max (1, 4 \delta_{\infty}(u_{\infty},s))$ given  by~\eqref{eq:sep-param2},
	then, with some universal finite constants $\mathcal{C}_0, ..., \mathcal{C}_3 >0$, 
	for  any $\tau > 1$ and a tuning parameter:
	\begin{equation}
	\label{eq:def-kappa-gauss}
	\kappa \geq \mathcal{C}_1 \sigma \sqrt{\Delta_T \log (\tau)},
	\end{equation}
	we have the prediction error bound of the estimators $\hat{\beta}$ and $\hat{\vartheta}$ defined in~\eqref{eq:generalized_lasso} given by:
	\begin{equation}
	\label{eq:gaussian_theorem}
	\sqrt{\Delta_T} \,  \norm{\hat{\beta}\Phi_{T}(\hat{\vartheta}) -
		\beta^{\star}\Phi_{T}(\vartheta^{\star}) }_{\ell_2} \leq
	\mathcal{C}_0 \,  \sqrt{\sparse} \, \kappa, 
	\end{equation}
	with                probability               larger                than
	$1 - \mathcal{C}_2 \left  ( \frac{\sqrt{2}b_T}{ \sigma_0 \tau \sqrt{\log
			(\tau)} }\vee  \frac{1}{\tau}\right )$.   Moreover, with  the same
	probability,                 we                have                 that
	$\left|\|\hat \beta \|_{\ell_1}  - \|\beta^\star\|_{\ell_1} \right| \leq
	\mathcal{C}_3    \kappa    s$    as    well    as    the    inequalities
	\eqref{eq:bound_th_abc} of Theorem \ref{th:bounds}.
\end{corollary}
The values of the universal constants $\mathcal{C}_i$, $i=0,\ldots, 3$, can be given explicitly and they are large, but they could be improved numerically.

\begin{rem}[A particular choice of the tuning parameter]
	\label{rem:tau-choice}
	Let $\gamma>0$ and $\gamma'\geq  \gamma$ such that
	$1>\gamma'-\gamma$. Set
	$\tau=T^{\gamma'}$, $b_T=\sigma_0 T^{\gamma'-\gamma} \sqrt{\log(T)}$ and $
	\kappa = \mathcal{C}_1 \sigma \sqrt{\Delta_T \log (\tau)}$ (which
	corresponds to the equality in~\eqref{eq:def-kappa-gauss}). Then, we get
	under the assumptions of Corollary~\ref{cor:gaussian_example} (and thus
	$T$ large enough) that:
	\[
	\frac 1{\sqrt{T}} \norm{\hat{\beta}\Phi_{T}(\hat{\vartheta}) -
		\beta^{\star}\Phi_{T}(\vartheta^{\star}) }_{\ell_2}
	\leq  \mathcal{C}''_0 \, \sigma \sqrt{\sparse \frac{\log (T)}{T}} ,
	\]
	with probability larger than $1 -
	\mathcal{C}''_2/T^{\gamma}$ where
	$\mathcal{C}''_0 =\sqrt{\gamma'}\,  \mathcal{C}_0 \,\mathcal{C}_1$  and 
	$\mathcal{C}''_2=  \sqrt{2/\gamma'} \, \mathcal{C}_2$. 
	Hence,
	we obtain a similar prediction error bound as the one given in
	Remark \ref{rem:thm1}, see~\eqref{eq:mse2}. Notice however that in 
	the model and references given in Remark~\ref{rem:thm1}, the  Riemannian diameter
	of the  parameter  set $\Theta_T$
	is bounded by a constant free of $T$, whereas  in this section it grows (sublinearly)
	with $T$  without degrading the prediction error bound. 
\end{rem}

\section{Scaled exponential model} \label{sec:scaled}

We develop in this section an example involving a dictionary that is not translation invariant and for which the associated metric differs from the Euclidean metric.
We consider a continuous dictionary composed of exponential functions continuously scaled which is used  in miscroscopy where it is often necessary to invert a Laplace transform (see  for instance \cite{poon2018geometry}, \cite{denoyelle2019sliding}).

\subsection{The model}

Consider  a  real-valued process  $y$  observed  continuously over $\R_+$ and assume that this process is an element of the Hilbert space $H_T=L^2(\R_+, \Leb)$ where $\Leb$ denotes here the Lebesgue measure over $\R_+$. We  write $H$ instead of $H_T$ for the Hilbert space and we write $\left \langle \cdot , \cdot \right \rangle$ its scalar product and $\norm{\cdot}$ its associated norm.

We consider  a truncated white noise  as  in Section \ref{sec:continuous_noise} such that $w_T=\sum_{k = 1}^T (1/\sqrt{T})\, G_k\,  \psi_k$, where
$(G_k, k\in \N)$ are independent centered Gaussian random variables with variance $\sigma^2$ and $(\psi_k, k\in \N)$ denotes an orthonormal basis of $H$. Hence Assumption~\ref{hyp:bruit} holds as 
$\norm{w_T}^2=\sum_{k=1}^T  G_k^2 / T$ is  a.s. finite
and
$\Var( \langle f , w_T\rangle) 
\leq \sigma^2 \,\Delta_T\,   \norm{f}^2$ with $\Delta_T=  1/T$. This gives that Point
$\ref{hyp:theorem1_point1}$ of Theorem~\ref{maintheorem} holds.

\begin{rem}
	We stress that by the law of large numbers $\norm{w_T}$ tends almost surely to $\sigma>0$. Therefore the upper bounds in previous results on super-resolution and BLasso  (see \cite{duval2015exact} or \cite{poon2018geometry}) which hold when  $\norm{w_T}$ tends to zero do not apply here.
\end{rem}

We consider the dictionary given by the scaling exponential model of Section
\ref{sec:translation_scale_model} given by:
\[
\Big (\varphi(\theta)  = k(\theta \,  \cdot  ) ,\,
\theta \in \Theta\Big) 
\quad\text{with}\quad
k(t)=\expp{-t}
\quad\text{and}\quad
\Theta = \R_+^*.
\]
We insist on the fact that in this example the dictionary and the observation space $H$ do not depend on $T$.
For simplicity we omit the index $T$ for the quantities which shall not depend on $T$. 
As the kernels do not depend on $T$, we choose the limit kernel to be the same, \emph{i.e}, $\ck := \ck_T = \ck_\infty$. In particular, Point
$\ref{hyp:V_T_theorem}$  of   Theorem~\ref{maintheorem} holds automatically. 
One easily checks that
Assumption  \ref{hyp:reg-f} on  the  regularity of  the features  holds, and elementary calculations give for $\theta,\theta' \in \Theta$:
\[
\norm{\varphi(\theta)} ^2= 1 / (2\theta),
\quad   \phi(\theta) =  \sqrt{2\theta}\expp{-\theta t}, 
\quad
\cK(\theta,\theta') =  \frac{2\sqrt{\theta\theta'}}{\theta + \theta'} 
\quad \text{and}\quad
g(\theta)= \frac{1}{4 \theta^2}.
\]
Since    the
function $g$ is
positive on $\Theta$, we get that 
Assumption~\ref{hyp:g>0} holds.
This gives that Point
$\ref{hyp:reg_dic_theorem}$ of Theorem~\ref{maintheorem} holds.
The Riemannian metric obtained from $g$
is given by, for   $\theta,\theta'\in\Theta$:
\begin{equation}
\mathfrak{d}(\theta,\theta') = \frac{1}{2} \left | \log \left (\frac{\theta}{\theta'}\right ) \right |\cdot
\end{equation}
Notice it is not equivalent to the  Euclidean  distance on $\Theta$. We see  that
$\cK$ is of class $\mathcal{C}^{3 ,3 }$ and that:
\begin{equation*}
\begin{aligned}
&\cK^{[i,j]}(\theta, \theta')= (-1)^j f^{(i+j)}\left (\frac{1}{2}\log \left (\frac{\theta}{\theta'} \right) \right ) 
\quad\text{with} \quad f(x)  = \frac{1}{\cosh(x)} \cdot
\end{aligned}
\end{equation*}
We shall retrieve scaling parameters over a compact set whose diameter may depend on $T$, for example we can take:
\[
\Theta_T = [ M_T^{-1},M_T]\quad\text{with}\quad M_T >1.
\]
Assumption~\ref{hyp:Theta_infini}  holds on $\Theta_\infty = \R_{+}^*$. This
gives that Point~\ref{it:hyp-reg-K} of Theorem~\ref{maintheorem} holds.  \medskip

\subsection{Existence of certificates}
In order  to get  the prediction  error from  Theorem~\ref{maintheorem}, it remains to show 
that  Point~$\ref{hyp:existence_certificate_theorem}$ therein  on the
existence  of the  certificates holds.   
To  check the  existence of  the
certificates,                 we                 can                 use
Propositions~\ref{prop:certificat_interpolating}
and~\ref{prop:certificat2}, and  check that all the  hypotheses required
in those two propositions hold. \medskip

We show first that the hypotheses of
Proposition~\ref{prop:certificat_interpolating} hold.
Assumption~$\ref{hyp:theorem_certificate_regularity}$ on the regularity of
the dictionary holds, see Section above.
\medskip

Elementary calculations    give       that       $L_{0,2}= 1$.
Recall  $\varepsilon_{\infty}(r)$ and $\nu_{\infty}(r)$ defined
in~\eqref{eq:def-e0} and~\eqref{eq:def-e2} (noted simply $\varepsilon$ and $\nu$ in this section).
Let    $\theta     <    \theta'$    in    $\Theta$ and    let   us    set
$r = \mathfrak{d}(\theta,\theta')$. We have,
$\cK(\theta,\theta') = f(r)$.
Since $f$ is positive and decreasing on $\mathbb{R}_{+}$, we have for $r > 0$, 
$\varepsilon(r) = 1 - 	f(r) > 0$.
Similarly we have:
\begin{equation*}
\cK^{[0,2]}(\theta,\theta') = f^{(2)}(r) = \frac{1}{\cosh(r)^3} \left(\cosh(r)^2- 2 \right). 
\end{equation*}
The function $f^{(2)}$ is increasing  and negative on $(0,\log(1 + \sqrt{2}))$. Hence, provided $r < \log(1 + \sqrt{2})$, we have 
$\nu(r) = - f^{(2)}(r) > 0$. 
This  and the regularity  of the kernel  $\cK $     imply  that 
Assumption~$\ref{hyp:theorem_certificate_concavity}$ of
Proposition~$\ref{prop:certificat_interpolating}$ holds for $\rho=1$ and all $r\in
(0, 1 / \sqrt{2})$. \medskip

Notice that $f^{(i)}$ can be written as the ratio of a polynomial of degree $i-1$ in $\cosh$ and $\sinh$ and of $\cosh^i$. In particular, there exists a finite constant $M$  such that for all $i\in \{0,\ldots, 3\}$ and $x\in \R$:
\begin{equation}
\label{eq:bound_f_i<MF}
|f^{(i)}(x)|\leq M f(x).
\end{equation}

So, we get  that $\lim_{q\rightarrow \infty }
\sup_{\mathfrak{d}(\theta,\theta')\geq q} |\cK^{[i, j]}(\theta,
\theta')|= \lim_{r \rightarrow \infty } \left |f^{(i+j)}(r) \right |= 0$ for all $i,j\in \{0, 1, 2\}$. Thus, we deduce from the
definition~\eqref{eq:def-delta-rs} 
of $\delta_\infty $ that $\delta_\infty (u,s)$ is finite for all $s\in
\N^*$ and $u>0$.  This implies that
Assumption~$\ref{hyp:theorem_certificate_separation}$ on the separation
of the parameters holds.
\medskip

As all kernels are equal in this setup, \emph{i.e} $\ck := \ck_T = \ck_\infty$, we have $\DT=0$ and $\RT=0$. Thus Assumption~$\ref{hyp:theorem_certificate_approximation}$ on the closeness to the limit kernel and Assumption~$\ref{hyp:theorem_certificate_metric}$  on  the closeness  of
the metrics $\mathfrak{d}_T$ and $\mathfrak{d}_\infty$ come for free with $\rho = 1$. 

\medskip

Recall the definition of $H^{(2)}_\infty $
from~\eqref{eq:def_H}. 
We  choose $u_\infty  = H_{\infty}^{(2)}(r_0,1)$ (as $\ck_\infty$ is chosen equal to $\ck_T$) for some $r_0 \in (0,1/\sqrt{2})$. We remark that in order to take $u_\infty$ as large as possible and then have a separation distance as small as possible (since it is a decreasing function of $u_\infty$), one could  take $r_0$ maximizing $H^{(2)}_\infty$.  

\medskip
Thus,  the assumptions of
Proposition~\ref{prop:certificat_interpolating} are satisfied, and we
deduce that Assumption~\ref{assumption1} holds.
\medskip

We       now      concentrate       on      the       hypotheses      of
Proposition~\ref{prop:certificat2}.
Assumptions~$\ref{hyp:theorem_certificate_2_regularity}$-$\ref{hyp:theorem_certificate_2_separation}$
clearly       hold      for       the       same      reasons       as
Assumptions~$\ref{hyp:theorem_certificate_regularity}$-$\ref{hyp:theorem_certificate_separation}$
of  Proposition~\ref{prop:certificat_interpolating}.
We take
$u'_\infty =u_\infty $. Assumption $\ref{hyp:theorem_certificate_2_approximation} $ comes for free in this setting.
\medskip

Thus,  the assumptions of
Proposition~\ref{prop:certificat2} are satisfied, and we
deduce, thanks
to~\eqref{eq:cstes-certif2}, that Assumption~\ref{assumption2} holds.
\medskip

In conclusion, we get that Assumptions~\ref{assumption1}
and~\ref{assumption2} hold and thus
Point~$\ref{hyp:existence_certificate_theorem}$ of
Theorem~\ref{maintheorem} holds for any set of parameters $\cq^\star$
such that  for all
$\theta\neq \theta'\in \cq^{\star}$ the distance $\mathfrak{d}(\theta, \theta')$ is larger
than the separation distance:
\begin{equation}
\label{eq:sep-param-scaling}
\max(r_0,
\delta_{\infty}(u_{\infty},s)). 
\end{equation}

\begin{rem}[On the separation distance~\eqref{eq:sep-param-scaling}]
	The  separation  distance~\eqref{eq:sep-param-scaling}  is  a  non-decreasing
	function  of $s$.  Similarly as in remark \ref{rem:sep-gauss} where an upper bound on the minimal distance in the Gaussian spike deconvolution case is given, we can  provide an   upper  bound for this distance. Let
	$(i,j)  \in \{0,1\}  \times \{0,1,2  \}$.  By  considering the  definition of the kernel
	$\cK$  and  the
	bound $\eqref{eq:bound_f_i<MF}$, we deduce that
	$  |\cK^{[i,j]}(\theta,\theta')| \leq  M f \left (  \dI(\theta,
		\theta') \right )$  for  all  $\theta,\theta' \in  \Theta$.   We  then
	obtain                             that                            for
	$\vartheta = (\theta_1,\cdots,\theta_s) \in \Theta^s_{\infty , \delta}
	$ with $\delta>0$:
	\[
	\max_{1\leq  \ell \leq s} \sum\limits_{k=1, k\neq
		\ell}^{s}      |\cK^{[i,j]}(\theta_{\ell},\theta_{k})|
	\leq  \psi_{s}(\delta)
	\quad\text{with}\quad
	\psi_{s}(\delta)=2M \int_0^{s/2 +1} f(\delta \, t) \,  \rd t.
	\]
	The  function $\psi_s$  is  decreasing and  one to  one  from $\R_+$  to
	$(0,M(s+2)]$. Setting $\psi_s^{-1}(u)  = 0$ for $u >  M(s+2)$, we deduce
	from~\eqref{eq:def-delta-rs} that for $u>0$:
	\[
	\delta_\infty(u,s) \leq   \psi_s^{-1}(u). 
	\]
	We can bound the quantity above independently of $s$.
	Since the map $\delta \mapsto \psi_s(\delta)$ is decreasing and the map $s\mapsto \psi_s(\delta)$ is increasing with limit
	$\psi_\infty (\delta)= 2 M \int_{0}^{+\infty} f(\delta t) \, \rd t = M \pi / \delta $, we deduce that for $s\in
	\N^*$:
	\[
	\delta_\infty(u,s) \leq \psi_{\infty}^{-1}(u) = \frac{M \pi}{u}\cdot
	\]

\end{rem}

\subsection{Prediction error}
\label{sec:predic-scaling}

From Theorem~\ref{maintheorem}, we deduce the subsequent following corollary. This demonstrates that by appropriately adjusting the penalization, the prediction error decreases to zero at the expected rate as the noise level tends to 0.

\begin{corollary}
	\label{cor:scaling_example}
	For                                                                  all
	$\theta\neq \theta'$ belonging to $  \cq^\star= \{\theta^\star_k, \, k  \in S^\star\}
	$, with  $S^\star=\supp (\beta^\star)$  such that  $\mathfrak{d}(\theta,\theta')$ is
	larger  than the  separation  given  by~\eqref{eq:sep-param-scaling},
	then, with some universal finite constants $\mathcal{C}_0, ..., \mathcal{C}_3 >0$, 
	for  any $\tau > 1$ and a tuning parameter:
	\begin{equation}
	\kappa \geq \mathcal{C}_1 \sigma \sqrt{ \log (\tau) / T},
	\quad\text{where}\quad \Delta_T=\frac{1}{T}, 
	\end{equation}
	we have the prediction error bound of the estimators $\hat{\beta}$ and $\hat{\vartheta}$ defined in~\eqref{eq:generalized_lasso} given by:
	\begin{equation}
	\norm{\hat{\beta}\Phi_{T}(\hat{\vartheta}) -
		\beta^{\star}\Phi_{T}(\vartheta^{\star}) } \leq
	\mathcal{C}_0 \,  \sqrt{\sparse} \, \kappa, 
	\end{equation}
	with                probability               larger                than
	$1 - \mathcal{C}_2 \left  ( \frac{\log(M_T)}{ 
		\tau \sqrt{\log
			(\tau)} }\vee  \frac{1}{\tau}\right )$.   Moreover, with  the same
	probability,                 we                have                 that
	$\left|\|\hat \beta \|_{\ell_1}  - \|\beta^\star\|_{\ell_1} \right| \leq
	\mathcal{C}_3    \kappa    s$    as    well    as    the    inequalities
	\eqref{eq:bound_th_abc} of Theorem \ref{th:bounds}.
\end{corollary}

\begin{remark} 
\label{rem:expo}
	We consider the particular case     $M_T=T^\gamma$ and $\tau=T^{\gamma'}$, with  $\gamma$ and $\gamma'$  positive. 
	We also take $\kappa=\mathcal{C}_1 \sigma \sqrt{\gamma' \log(T)/T}$.  The prediction error is then given by:
	\[
	\norm{\hat{\beta}\Phi_{T}(\hat{\vartheta}) -
		\beta^{\star}\Phi_{T}(\vartheta^{\star}) } \leq
	\mathcal{C}_0 \, \mathcal{C}_1 \,  \sqrt{\sparse} \, \sigma \sqrt{\gamma' \frac{\log(T)}T}, 
	\]
	with probability larger than $1 - \mathcal{C}_2 \left  ( \frac{\gamma}{\sqrt{\gamma'}}\frac{\sqrt{ \log
			(T)}}{ 
		T^{\gamma'} }\vee  \frac{1}{T^{\gamma'}}\right )$.
	
\end{remark}


\bibliographystyle{imsart-number} 
\bibliography{ref.bib}       


\begin{appendix}


\section{Proofs of Theorems \ref{maintheorem} and  \ref{th:bounds}}
\label{sec:proofsSection2}

\subsection{Proof of Theorem \ref{maintheorem}}
\label{sec:proof_main_theorem}

Let us bound the prediction error $\hat{R}_T := \norm{\hat{\beta}\Phi_{T}(\hat{\vartheta}) - \beta^{\star}\Phi_{T}(\vartheta^{\star}) }_{T}$. 
By definition \eqref{eq:generalized_lasso} of $\hat{\beta}$ and
$\hat{\vartheta}$ for the tuning parameter $\kappa$,
we have:
\begin{equation*}
\frac{1}{2}\norm{y - \hat{\beta}\Phi_T(\hat{\vartheta})}_T^2 + \kappa \|\hat{\beta}\|_{\ell_1} \leq \frac{1}{2}\norm{y - \beta^\star\Phi_T(\vartheta^\star)}_T^2 + \kappa \|\beta^\star\|_{\ell_1}.
\label{eq:optim}
\end{equation*}
We define the application $\hat{\Upsilon} $ from $H_T$ to $\R$ by:
\[
\hat{\Upsilon}( f)= \left \langle \hat{\beta}
\Phi_{T}(\hat{\vartheta})-
\beta^{\star}\Phi_{T}(\vartheta^{\star}),f \right \rangle_T.
\]
This gives,  by rearranging terms  and using  the equation of  the model
$y = \beta^\star\Phi_T(\vartheta^\star) + w_T$, that:
\begin{equation}
\label{eq:optim_bis}
\frac{1}{2}\hat{R}_T^2
\leq \hat{\Upsilon} (w_T) + \kappa \left ( \|\beta^{\star}\|_{\ell_1} -
\|\hat{\beta}\|_{\ell_1}\right ). 
\end{equation}
Next,  we  shall  expand  the  two  terms on  the  right  hand  side  of
\eqref{eq:optim_bis} according to $\hat \beta_\ell$ close to some
$\beta^\star_k$ or not. In the
rest of the proof, we fix  $r>0$ so that Assumptions \ref{assumption1} and
\ref{assumption2}, are verified by $\cq^{\star}$. In particular, for all
$k\neq         k'    $ in $ S^\star= \{k'' \in \{1,\cdots,K\}, \, \beta^{\star}_{k''} \neq 0 \}$          we         have
$\mathfrak{d}_T(\theta_k^\star,\theta_{k'}^\star) > 2r$.

Recall  the definitions  given in  Section \ref{sec:main_result}  of the
sets  of indices  $\hat{S}$, $\tilde{S}_{k}(r)$  and $\tilde{S}(r)$  for
$k      \in       S^\star$.       Since      the       closed      balls
$\mathcal{B}_T(\theta^{\star}_{k  },r)$  with   $k  \in  S^{\star}$  are
pairwise disjoint, the  sets $\tilde{S}_k(r)$, for $k  \in S^\star$, are
also pairwise disjoint and one can write the following decomposition:
\begin{equation*}
\begin{aligned}
\hat{\beta}\Phi_{T}(\hat{\vartheta}) - \!
\beta^{\star}\Phi_{T}(\vartheta^{\star})
& =  \sum\limits_{k=1}^{K} \hat{\beta}_{k}\phi_{T}(\hat{\theta}_{k}) - \! \sum\limits_{k\in S^{\star}}\beta_{k}^{\star}\phi_{T}(\theta^{\star}_{k}) =  \sum\limits_{k \in S^\star}\sum\limits_{\ell \in \tilde{S}_k(r) }\hat{\beta}_{\ell}\phi_{T}(\hat{\theta}_{\ell}) + \! \! \sum\limits_{k \in \tilde{S}(r)^{c}}\hat{\beta}_{k}\phi_{T}(\hat{\theta}_{k}) -  \sum\limits_{k\in S^{\star}}\beta_{k}^{\star}\phi_{T}(\theta^{\star}_{k}).
\end{aligned}
\end{equation*} 
This  decomposition  groups  the   elements  of  the  predicted  mixture
according to the proximity of the estimated parameter $\hat \theta_\ell$
to a true underlying parameter $\theta_k^\star$ to be estimated. We 
use   a   Taylor-type   expansion  with   the   Riemannian  metric
$\mathfrak{d}_T$ for  the function $\phi_T(\theta)$ around  the elements
of $\cq^\star$. By Assumption  \ref{hyp:reg-f}, the function $\phi_T$ is
twice continuously differentiable with  respect to the variable $\theta$
and  the  function  $g_T$  defined in  \eqref{def:g_T}  is  positive  on
$\Theta_T$  and of class $\mathcal{C}^1$ by   Assumption  \ref{hyp:g>0}.  We  set   in  this  section
$\tilde{D}_{i;T}[\phi_{T}] = \phi_{T}^{[i]}$ for  $i=0, 1, 2$.  According
to Lemma  \ref{lemma:expansion}, we have for  any $\theta_k^{\star}$ and
$\hat{\theta}_{\ell}$ in $\Theta_T$:
\[
\phi_{T}(\hat{\theta}_\ell) = \phi_{T}(\theta_k^{\star}) +
\operatorname{sign}(\hat{\theta}_\ell-\theta_k^{\star})\,
\mathfrak{d}_T(\hat{\theta}_\ell,\theta_k^{\star})\,
\phi_{T}^{[1]}(\theta_k^{\star}) +
\mathfrak{d}_T(\hat{\theta}_\ell,\theta_k^{\star})^2 \,  \int_0^1 (1-s)
\phi_{T}^{[2]}(\gamma_s^{(k\ell)}) \, \rd s, 
\]
where $\gamma^{(k\ell)}$ is a distance realizing geodesic path belonging to $\Theta_T$ such that $\gamma_0^{(k\ell)} = \theta_{k}^{\star}$, $\gamma_1^{(k\ell)}= \hat \theta_\ell$ and $\mathfrak{d}_T(\hat \theta_\ell,\theta_k^\star) = \mathcal{L}_T(\gamma^{(k\ell)})$.
Hence we obtain:
\begin{multline}
\label{eq:decomp_taylor}
\hat{\beta}\Phi_{T}(\hat{\vartheta}) -
\beta^{\star}\Phi_{T}(\vartheta^{\star})
=    \sum\limits_{k \in S^{\star}} I_{0,k}(r)\, 
\phi_{T}(\theta_{k}^{\star}) + \sum\limits_{k \in S^{\star}} I_{1, k}
(r)\, \phi_{T}^{[1]}(\theta_k^{\star})
+ \sum\limits_{k \in \tilde{S}(r)^{c}}\hat{\beta}_{k} \, \phi_{T}(\hat{\theta}_{k})\\ 
\quad+ \sum\limits_{k \in S^{\star}} \left( \sum\limits_{\ell \in \tilde{S}_{k}(r) } \hat{\beta}_{\ell}  \, \mathfrak{d}_T(\hat{\theta}_{\ell},\theta_k^{\star})^2 \, \int_0^1 (1-s) \phi_{T}^{[2]}(\gamma_s^{(k\ell)}) \, \rd s  \right ),
\end{multline}
with
\begin{equation*}
I_{0, k}(r)=\big (\sum\limits_{\ell \in \tilde{S}_{k}(r) }
\hat{\beta}_{\ell} \big)- \beta_{k}^{\star}
\quad\text{and}\quad
I_{1, k} (r)= \sum\limits_{\ell \in \tilde{S}_{k}(r) }
\hat{\beta}_{\ell} \,
\operatorname{sign}(\hat{\theta}_{\ell}-\theta_k^{\star})\,
\mathfrak{d}_T(\hat{\theta}_{\ell},\theta_k^{\star}). 
\end{equation*}
Let us introduce some  notations in order to bound the different terms of the expansion above:
\begin{align}
\label{def:I0-I1}
I_0(r)
&= \sum\limits_{k \in S^{\star}} |I_{0,k}(r)|
\quad\text{and}\quad
I_1(r)  = \sum\limits_{k \in S^{\star}} |I_{1,k}(r)|,\\
\label{eq:I2}  
I_{2,k}(r)
&= \sum\limits_{\ell \in \tilde{S}_{k}(r) } \left
|\hat{\beta}_{\ell}\right |
\mathfrak{d}_T(\hat{\theta}_{\ell},\theta_k^{\star})^2
\quad\text{and}\quad
I_2(r)  = \sum\limits_{k \in S^{\star}} I_{2,k}(r),\\
\label{eq:I3}  
I_3(r) & = \sum\limits_{\ell \in \tilde{S}(r)^{c}} \left |\hat{\beta}_{\ell}
\right | = \left \|\hat{\beta}_{\tilde{S}(r)^{c}}\right
\|_{\ell_1},
\end{align}
and we omit the dependence in $r$ when there is no ambiguity. 
\medskip

We bound the difference $\|\beta^{\star}\|_{\ell_1}-\|\hat{\beta}\|_{\ell_1}$ by noticing that:
\begin{equation}
\label{eq:diff_l1_norm}
\|\beta^{\star}\|_{\ell_1}-\|\hat{\beta}\|_{\ell_1} 
= \sum\limits_{k \in S^{\star}} \Big (|\beta_k^\star| -
\sum\limits_{\ell \in \tilde{S}_{k}(r) } |\hat{\beta}_{\ell}| \Big )
- \sum\limits_{k \in \tilde{S}(r)^{c}} \left |\hat{\beta}_{k} \right |
\leq \sum\limits_{k \in S^{\star}} \Big |\beta_k^\star-
\sum\limits_{\ell \in \tilde{S}_{k}(r) } \hat{\beta}_{\ell} \Big| =
I_0. 
\end{equation}

In  the  next lemma,  we  give  an upper  bound  of  $I_0$.  Recall  the
constants $C_N'$ and $C_F$ from Assumption \ref{assumption1}.

\begin{lem}
	\label{boundI0}
	Under the assumptions of Theorem \ref{maintheorem} and with the element
	$p_1  \in  H_T$  from  Assumption \ref{assumption1}  associated  to  the
	function $v: \cq^\star \rightarrow \{ -1,1\}$ defined by:
	\begin{equation*}
	v(\theta^\star_k) = \operatorname{sign}(I_{0, k}) \quad\text{for all $ k \in S^{\star}$},
	\end{equation*}
	we get that:
	\begin{equation}
	\label{eq:I_0_I_3_bound_I_2_I_star}
	I_0  \leq  C_{N}' I_2 + (1 - C_{F})I_3 + | \hat{\Upsilon}(p_1) |.  
	\end{equation}
\end{lem}

\begin{proof}
	Let  $v\in \{-1,1\}^s$ with entries $v_k = v(\theta_k^\star)$ so that: 
	\[
	I_0 = \sum_{k\in S^{\star}} |I_{0, k}|=\sum_{k\in S^{\star}} v_k I_{0, k}=
	\sum\limits_{k \in S^{\star}}  v_{k}\Big ( \Big (\sum\limits_{\ell
		\in \tilde{S}_{k}(r) } \hat{\beta}_{\ell} \Big )-
	\beta_{k}^{\star}\Big ). 
	\]
	Let $p_1$ be an element of $H_T$ from Assumption~\ref{assumption1} associated to
	the application $v$ such that  properties
	$\ref{it:as1-<1}$-$\ref{it:norm<c}$ therein hold. By adding and substracting
	$\sum\limits_{k \in S^{\star}}  \sum\limits_{\ell \in \tilde{S}_{k}(r) }
	\hat{\beta}_{\ell} \left \langle \phi_{T}(\hat{\theta}_{\ell}),p_{1}
	\right \rangle_T $ to $I_0$ and using the property $\ref{it:as1-ordre=2}$
	satisfied by the element $p_1$, that is,  $\left \langle
	\phi_{T}(\theta_{k}^{\star}),p_{1}\right \rangle_T=v_k$ for all $k\in
	S^\star$, we obtain: 
	\[
	I_0 
	=  \sum\limits_{k \in S^{\star}}  \sum\limits_{\ell \in \tilde{S}_{k}(r)
	} \hat{\beta}_{\ell} \left (v_{k}-\left \langle
	\phi_{T}(\hat{\theta}_{\ell}),p_{1}\right \rangle _T\right ) +
	\left \langle \hat{\beta}\Phi_{T}(\hat{\vartheta})-
	\beta^{\star}\Phi_{T}(\vartheta^{\star}),p_{1} \right \rangle_T -
	\sum\limits_{\ell \in \tilde{S}(r)^{c}} \hat{\beta}_{\ell}\left \langle
	\phi_{T}(\hat{\theta}_{\ell}),p_{1}\right \rangle_T .
	\]
	We deduce that:
	\[
	I_0
	\leq \sum\limits_{k \in S^{\star}}  \sum\limits_{\ell \in
		\tilde{S}_{k}(r) }
	|\hat{\beta}_{\ell}| \left |v_{k}-\left \langle
	\phi_{T}(\hat{\theta}_{\ell}),p_{1}\right \rangle_T \right |
	+  | \hat{\Upsilon}(p_1) |
	+ \sum\limits_{\ell \in \tilde{S}(r)^{c}} |\hat{\beta}_{\ell}| \left
	|\left \langle \phi_{T}(\hat{\theta}_{\ell}),p_{1} \right \rangle_T
	\right |. 
	\]
	Notice that for $\ell \in \tilde{S}(r)^{c}$,  $\hat \theta_\ell
	\notin \bigcup\limits_{k \in S^\star}
	\mathcal{B}_T(\theta^{\star}_k,r)$. Then, by using the properties
	$\ref{it:as1-ordre=2}$ 
	and $\ref{it:as1-<1-c}$ from Assumption \ref{assumption1}, we get that
	\eqref{eq:I_0_I_3_bound_I_2_I_star} holds 
	with  the constants $C_N'$ and $C_F$ from  Assumption \ref{assumption1}.
\end{proof}

In  the  next lemma,  we  give  an upper  bound  of  $I_1$.  Recall  the
constants $c_N$ and $c_F$ from Assumption \ref{assumption2}.

\begin{lem}
	\label{boundI1}
	Under  the  assumptions  of  Theorem \ref{maintheorem}  and  with  the
	element $q_{0} \in H_T$  from Assumption \ref{assumption2} associated
	to the function $v :\cq^\star \rightarrow \{ -1,1\}$ defined by:
	\begin{equation*}
	v(\theta_k^\star) =
	\operatorname{sign}(I_{1, k})
	\quad\text{for all $ k \in S^{\star}$},  
	\end{equation*}
	we get that:
	\begin{equation}
	\label{eq:I_1_bound_I_2_I_star}
	I_1 \leq c_{N}I_2 + c_{F}I_3 +  | \hat{\Upsilon}(q_0)  |.
	\end{equation}
\end{lem}

\begin{proof} Let  $v\in \{-1,1\}^s$ with entries $v_k = v(\theta_k^\star)$ so that:  
	\[
	I_1
	= \sum\limits_{k \in S^{\star}}|I_{1, k}|
	= \sum\limits_{k \in S^{\star}}v_{k} I_{1, k}
	= \sum\limits_{k \in S^{\star}}\sum\limits_{\ell\in \tilde{S}_{k}(r)
	}\hat{\beta}_{\ell} \,  v_{k} \,
	\operatorname{sign}(\hat{\theta}_{\ell}-\theta_k^{\star})\,
	\mathfrak{d}_T(\hat{\theta}_{\ell},\theta_k^{\star}) .
	\]
	Let $q_0\in H_T$ from Assumption~\ref{assumption2} associated to
	the       application        $v$       such        that       properties
	$\ref{it:as2-ordre=2}$-$\ref{it:as2-<c}$  therein   hold.   By   adding  and
	substracting
	$\sum_{\ell \in  \tilde{S}(r) } \hat{\beta}_{\ell}  \left \langle
	\phi_{T}(\hat{\theta}_{\ell}),  q_{0}\right  \rangle_T= \left\langle
	\hat{\beta}\Phi_{T}(\hat{\vartheta}), q_{0}\right \rangle_T - 
	\sum_{\ell \in \tilde{S}(r)^{c}} \hat{\beta}_{\ell} \left \langle
	\phi_{T}(\hat{\theta}_{\ell}), q_{0}\right \rangle_T $ 
	to  $I_1$  and
	using the triangle inequality, we obtain:
	\begin{multline*}
	I_1  
	\leq \sum\limits_{k \in S^{\star}}\sum\limits_{\ell\in \tilde{S}_{k}(r)
	}|\hat{\beta}_{\ell}| \left | v_{k} \,
	\operatorname{sign}(\hat{\theta}_{\ell}-\theta_k^{\star})\,
	\mathfrak{d}_T(\hat{\theta}_{\ell},\theta_k^{\star}) - \left \langle
	\phi_{T}(\hat{\theta}_{\ell}), q_{0}\right \rangle_T \right |\\
	+ \sum\limits_{\ell \in \tilde{S}(r)^{c}} |\hat{\beta}_{\ell}| \left
	|\left \langle  \phi_{T}(\hat{\theta}_{\ell}), q_{0}\right \rangle_T
	\right | + \left | \left \langle  \hat{\beta}\Phi_{T}(\hat{\vartheta}),
	q_{0}\right \rangle_T \right |  . 
	\end{multline*}
	The property $\ref{it:as2-ordre=2}$ of Assumption  \ref{assumption2} gives
	that  $\left \langle \phi_{T}(\theta_{k}^{\star}) ,q_{0} \right
	\rangle_T = 0$ for all $k \in S^{\star}$. This implies that  $\left
	\langle  \beta^\star\Phi_{T}(\vartheta^\star), q_{0}\right \rangle_T =
	0$. Then, by using the definition of $I_2$ and $I_3$ from \eqref{eq:I2}-\eqref{eq:I3}
	and the properties $\ref{it:as2-ordre=2}$ and $\ref{it:as2-<1-c}$ of
	Assumption  \ref{assumption2}, we obtain:   
	\[
	I_1
	\leq c_{N}I_2 + c_{F}I_3 + \left | \left \langle \hat{\beta}
	\Phi_{T}(\hat{\vartheta}),q_{0} \right \rangle_T
	\right |
	= c_{N}I_2 + c_{F}I_3 +  | \hat{\Upsilon}(q_0)|, 
	\]
	with the constants $c_N$ and $c_F$ from  Assumption \ref{assumption2}.
\end{proof}

We consider the following  suprema of Gaussian processes for $i=0, 1, 2$:
\begin{equation*}
M_{i} = \underset{\theta \in \Theta_T}{\sup}\left |\left
\langle w_T, \phi_{T}^{[i]}(\theta) \right \rangle_T\right |.
\end{equation*}
By using the expansion \eqref{eq:decomp_taylor} and the bounds \eqref{eq:I_1_bound_I_2_I_star}  and  \eqref{eq:I_0_I_3_bound_I_2_I_star} for the second inequality, we obtain:
\begin{align}
\label{eq:w.diff-0}
| \hat{\Upsilon}(w_T) |
&\leq (I_0 + I_3) M_{0} + I_1 M_{1} + I_2 \, 2^{-1} \, M_{2}\\
&\leq (C_{N}' I_{2 } + (2 - C_{F})I_3 + |\hat{\Upsilon}(p_1)|) M_{0}
+ (c_{N}I_2 + c_{F}I_3 + |\hat{\Upsilon}(q_0)|)M_{1} + I_2 \,
2^{-1} \, M_{2}.
\label{eq:w.diff}
\end{align}

\medskip

At this point, one needs to bound $I_2$ and $I_3$. 
In order to do so, we will bound from above and from below the                    
Bregman divergence $D_{B}$
defined by:
\begin{equation}
\label{def:bregman_divergence}
D_{B} = \|\hat{\beta}\|_{\ell_1} - \|\beta^{\star}\|_{\ell_1} - \hat{\Upsilon}(p_0),
\end{equation}
where $p_{0}$  is the element of $H_T$ given by the Assumption~\ref{assumption1}  associated to the application $v:\cq^\star \rightarrow \{ -1,1\}$ given by:
\begin{equation}
\label{def:p_0}
v(\theta_{k}^\star) = 	 \text{sign} (\beta^{\star}_{k}) \quad\text{for
	all $ k \in S^{\star}$}.
\end{equation}

The next lemma gives a lower bound of the Bregman divergence.

\begin{lem}
	\label{lower_bound_bregman}
	Under  the  assumptions  of   Theorem  \ref{maintheorem}  and  with  the
	constants $C_N$ and  $C_F$  of  Assumption \ref{assumption1}, we
	get that:
	\begin{equation}
	\label{eq:lower_bound_bregman}
	D_{B} \geq C_{N}I_2 + C_{F}I_3 . 
	\end{equation}
\end{lem}

\begin{proof} By definition \eqref{def:bregman_divergence} of $D_B$ we have:
	\[
	D_{B} 
	=\sum\limits_{k \in \hat{S}} |\hat{\beta}_{k}| -\hat{\beta}_{k} \left
	\langle \phi_{T}(\hat{\theta}_{k}) ,p_{0} \right \rangle_T  - \left(
	\sum\limits_{k \in S^{\star}} |\beta^{\star}_{k}| -\beta^{\star}_{k}
	\left \langle \phi_{T}(\theta^{\star}_{k}) ,p_{0} \right \rangle_T
	\right) .
	\]
	By  using the  interpolating  properties of  the  element $p_{0}$ of $H_T$  from
	Assumption~\ref{assumption1}  associated  to  the function  $v$  defined
	in~\eqref{def:p_0},                        we                       have
	$\sum_{k  \in  S^{\star}} |\beta^{\star}_{k}|  -\beta^{\star}_{k}  \left
	\langle   \phi_{T}(\theta^{\star}_{k})  ,p_{0}   \right  \rangle_T   =
	0$. Hence, we deduce that:
	\begin{equation*}
	\begin{aligned}
	D_{B}  &=\sum\limits_{k \in \hat{S}} |\hat{\beta}_{k}| -\hat{\beta}_{k}
	\left \langle \phi_{T}(\hat{\theta}_{k}) ,p_{0} \right \rangle_T \\
	&\geq 		\sum\limits_{k \in \hat{S}} |\hat{\beta}_{k}|
	-|\hat{\beta}_{k}| \left | \left \langle \phi_{T}(\hat{\theta}_{k})
	,p_{0} \right \rangle_T \right |\\
	&=   \sum\limits_{\ell \in \tilde{S}(r)}|\hat{\beta}_{\ell}| \left ( 1 -
	\left | \left \langle \phi_{T}(\hat{\theta}_{\ell}) ,p_{0} \right
	\rangle_T  \right |\right )  + \sum\limits_{k \in \tilde{S}(r)^{c}}
	|\hat{\beta}_{k}| \left ( 1 -  \left | \left \langle
	\phi_{T}(\hat{\theta}_{k}) ,p_{0} \right \rangle_T  \right |\right
	).
	\end{aligned}
	\end{equation*}
	Thanks to properties $\ref{it:as1-<1}$ and $\ref{it:as1-<1-c}$ of Assumption
	\ref{assumption1} and the definitions  \eqref{eq:I2} and \eqref{eq:I3} of $I_2$ and $I_3$,
	we obtain:
	\[
	D_{B}  \geq  \sum\limits_{k \in  S^{\star}} \sum\limits_{\ell\in
		\tilde{S}_{k}(r)} C_{N} | \hat{\beta}_{\ell}| \mathfrak{d}_T
	(\hat{\theta}_{\ell}, \theta_{k}^{\star})^{2} + \sum\limits_{k \in
		\tilde{S}(r)^{c}} C_{F}|\hat{\beta}_{k}| = C_{N}I_2 + C_{F}I_3,
	\]
	where the constants $C_N$ and $C_F$ are that of Assumption \ref{assumption1}.
\end{proof}

We now give an upper bound of  the Bregman divergence.
\begin{lem}
	\label{lem:upper_bound_bregman}
	Under  the  assumptions  of   Theorem  \ref{maintheorem}, we have:
	\begin{multline}
	\label{eq:upper_bound_bregman}
	\kappa D_{B}  \leq I_2\left (C_{N}' M_{0}+ c_{N} M_{1}+
	2 ^{-1}  M_{2}\right) + I_3\left ( (2- C_{F})  M_{0} +  c_{F} M_{1}
	\right )
	+ |\hat{\Upsilon}(p_1)| M_{0}+|\hat{\Upsilon}(q_0)| M_{1} +  \kappa
	|\hat{\Upsilon}(p_0)| .
	\end{multline}
\end{lem}
\begin{proof} Recall that $\cq^{\star}\subset \Theta_T$.
	We deduce from \eqref{eq:optim_bis} that:
	\begin{equation}
	\kappa(||\hat{\beta}||_{\ell_1} - ||\beta^{\star}||_{\ell_1})
	\leq  \hat{\Upsilon}(w_T) -\frac{1}{2}
	\norm{\beta^{\star}\Phi_{T}(\vartheta^{\star})-
		\hat{\beta}\Phi_{T}(\hat{\vartheta})}_{T}^{2} 
	\leq \hat{\Upsilon}(w_T).  
	\end{equation} 
	Using~\eqref{def:bregman_divergence},  we obtain:
	\begin{equation*}
	\kappa	D_{B} \leq  |\hat{\Upsilon}(w_T) |+ \kappa  |\hat{\Upsilon}(p_0)|.
	\end{equation*}
	Then, use \eqref{eq:w.diff} to get~\eqref{eq:upper_bound_bregman}.
\end{proof}

By combining the upper and lower bounds \eqref{eq:lower_bound_bregman}
and \eqref{eq:upper_bound_bregman}, we deduce
that:
\begin{multline}
\label{boundI2}
I_{2 } \left ( C_{N} - \frac{1}{\kappa} \left (C_{N}' M_{0}+ c_{N}
M_{1}+ 2 ^{-1} M_{2}\right  ) \right ) + I_3\left (C_{F} -
\frac{1}{\kappa} \left ( (2- C_{F})  M_{0} +  c_{F} M_{1} \right
)\right )\\ 
\leq\frac{1}{\kappa} |\hat{\Upsilon}(p_1)| M_{0}+\frac{1}{\kappa}
|\hat{\Upsilon}(q_0)| M_{1} +  |\hat{\Upsilon}(p_0)|. 
\end{multline}
\medskip

We define the events:
\begin{equation}
\label{eq:events_A}
\mathcal{A}_i = \left \{  M_{i} \leq  \mathcal{C} \, \kappa \right \},
\, \quad\text{for $i \in \{0,1,2\}$}
\quad \text{and}\quad
\mathcal{A} = \mathcal{A}_0 \cap \mathcal{A}_1 \cap \mathcal{A}_2,
\end{equation}
where:
\[
\mathcal{C}=  \frac {C_F }{2 ( 2-C_F + c_F)} \wedge
\frac{C_N}{ 2  ( C'_{N} + c_{N} + 2^{-1})}\cdot 
\]
(We shall prove in~\eqref{eq:bounds_event_A}
that  the event $\mathcal{A}$ occurs with high probability.)
We get from Inequality $\eqref{boundI2}$, that on the event $\mathcal{A}$:
\begin{equation}
\label{eq:bound_I_2_I_star}
C_{N}  I_{2 } + C_{F} I_3
\leq   2 \mathcal{C}'  \left (|\hat{\Upsilon}(p_1)| +
|\hat{\Upsilon}(q_0)| +|\hat{\Upsilon}(p_0)|\right )
\quad\text{with}\quad
\mathcal{C}' = \mathcal{C}\vee 1. 
\end{equation}
\medskip

By reinjecting \eqref{eq:diff_l1_norm},
\eqref{eq:w.diff}, \eqref{eq:I_0_I_3_bound_I_2_I_star} and~\eqref{eq:I_1_bound_I_2_I_star}
in 
\eqref{eq:optim_bis} one gets:
\begin{multline*}
\frac{1}{2}\hat{R}_T^2
\leq I_2(C_{N}'M_{0} + c_{N}M_{1} + 2^{-1} M_{2} + \kappa C_{N}') + I_3(
(2 - C_{F})M_{0}+ c_{F} M_{1} + \kappa (1 - C_{F}))\\ 
+|\hat{\Upsilon}(p_1)| (M_{0}+ \kappa) + |\hat{\Upsilon}(q_0)|M_{1}. 
\end{multline*}
Using~\eqref{eq:bound_I_2_I_star}, we obtain an upper bound
for the prediction error on the event $\mathcal{A}$:
\begin{equation}
\label{eq:bound_A_square}
\hat{R}_T^2 \leq C \, \kappa \, (|\hat{\Upsilon}(p_0)| + |\hat{\Upsilon}(p_1)| + |\hat{\Upsilon}(q_0)|),
\end{equation}
with
\[
C = 4 {\mathcal{C}'} \left(1+
\frac{\mathcal{C}'}{C_N}( 2 C_N'+c_N+1)
+ \frac{\mathcal{C}'}{C_F}(3-2C_F + c_F)
\right).
\]
\medskip

Using the Cauchy-Schwarz inequality and the definition of
$\hat{\Upsilon}$, we get that  for $f \in H_T$:
\begin{equation}
\label{eq:functional_bound}
|\hat{\Upsilon}(f) |\leq \hat{R}_T \norm{f}_T.
\end{equation}
Using Assumption~\ref{assumption1}~$\ref{it:norm<c}$ for $p_0$ and $p_1$,
and Assumption~\ref{assumption2}~$\ref{it:as2-<c}$ for $q_0$,  we get:
\begin{equation}
\label{eq:bounds_pq}
\norm{p_0}_T \leq C_B\sqrt{\sparse}, \quad \norm{p_1}_T \leq
C_B\sqrt{\sparse}
\quad\text{and}\quad
\norm{q_0}_T \leq c_B\sqrt{\sparse}.
\end{equation} 
Plugging this in \eqref{eq:bound_A_square}, we get that on the event $\mathcal{A}$:
\begin{equation}
\label{eq:bound_prediction_squared}
\hat{R}_T^2 \leq \mathcal{C}_0 \, \kappa \hat{R}_T \,  \sqrt{\sparse}
\quad\text{with}\quad \mathcal{C}_0 = (c_B+2C_B) C.
\end{equation}
This gives \eqref{eq:main_theorem}.

\medskip

The   proof   of   \eqref{eq:main_theorem_diff_l1}   is   postponed   to
Section~\ref{sec:proof_theorem_bounds} and  will be easily  deduced from
the first and third inequalities in 
\eqref{eq:bound_th_abc}.

\medskip

To complete the  proof of Theorem~\ref{maintheorem} we  shall give a
lower bound  for the probability  of the event $\mathcal{A}$  defined in
\eqref{eq:events_A}.
For $i=0, 1, 2$ and $\theta\in \Theta$, set $X_i(\theta)= \left  \langle
w_T,  \phi_{T}^{[i]}(\theta) \right  \rangle_T$ a   real   centered
Gaussian   process  with   continuously
differentiable sample paths, so that its supremum is 
$M_{i}=\sup_{\Theta_T} |X_i|$.

We first consider $i=0$.  We have,
thanks  to~\eqref{def:derivatives_kernel}   and~\eqref{eq:K-g}  for  the
second part:
\[
\norm{\phi_T(\theta)}_T^2 = 1
\quad \text{and} \quad
\norm{\phi_{T}^{[1]}(\theta)}_T^2= \cK_T^{[1,1]}(\theta,\theta) = 1.
\]
Recall  Assumption  \ref{hyp:bruit} on the noise $w_T$ holds. We deduce from Lemma \ref{lem:bound_azais_final} with $C_1=C_2=1$ that:
\begin{equation}
\label{eq:bound_A_0}
\P\left (\mathcal{A}_0^c \right )=
\P \left ( \sup_{\Theta_T} |X_0| > \mathcal{C}  \, \kappa \right ) \leq c_0 \left ( \sigma \frac{|\Theta_T|_{\mathfrak{d}_T}\sqrt{ \Delta_T}}{\mathcal{C}  \, \kappa }\vee 1 \right )\,  \expp{-(\mathcal{C}  \, \kappa)^2/(4
	\sigma^2 \Delta_T)},
\end{equation}
where  $|\Theta_T|_{\mathfrak{d}_T}$ denotes  the  diameter  of the  set
$\Theta_T$ with respect to the metric $\mathfrak{d}_T$ and $c_0 = 3$.

We consider $i=1$. Thanks to~\eqref{def:derivatives_kernel}, we get:
\[
\norm{\phi_{T}^{[1]}(\theta)}_T^2 = 1
\quad \text{and} \quad
\norm{\tilde{D}_{1;T}[\phi_{T}^{[1]}](\theta) }_T^2 = \norm{\phi_{T}^{[2]}(\theta)  }_T^2 = \cK_T^{[2,2]}(\theta,\theta).
\]
Recall  $L_{2,2}$ and $\DT$ are defined in  \eqref{def:M_h_M_g} and \eqref{def:V_1}.
Since Assumptions~\ref{hyp:Theta_infini}
and~\ref{hyp:close_limit_setting} hold, we get
that for $\theta \in \Theta_T$:
\[
\cK_T^{[2,2]}(\theta,\theta) \leq L_{2,2} + \DT \leq 2 L_{2,2}.
\]
We deduce from  Lemma \ref{lem:bound_azais_final}
with $C_1=1$ and $C_2= \sqrt{2 L_{2,2}}$
and taking  $c_1 =  2\sqrt{2L_{2,2}} + 1$, that:
\begin{equation}
\label{eq:bound_A_1}
\P\left (\mathcal{A}_1^c \right )
= \P \left ( \sup_{ \Theta_T} |X_1| > \mathcal{C}  \, \kappa \right )
\leq c_1 \left ( \sigma \frac{|\Theta_T|_{\mathfrak{d}_T}\sqrt{
		\Delta_T}}{\mathcal{C}  \, \kappa }\vee 1 \right )\,
\expp{-(\mathcal{C}  \, \kappa)^2/(4  	\sigma^2 \Delta_T)}. 
\end{equation}

We consider $i=2$.
Thanks to~\eqref{def:derivatives_kernel}, we get:
\[
\norm{\phi_T^{[2]}(\theta)}_T^2 =  \ck_T^{[2,2]}(\theta,\theta)
\quad \text{and} \quad
\norm{\tilde{D}_{1;T}[\phi_T^{[2]}](\theta)}_T^2 =
\norm{\phi_{T}^{[3]}(\theta)}_T^2 = \cK_T^{[3,3]}(\theta,\theta) . 
\]
Recall the definition of the function $h_\infty $ given in~\eqref{eq:def-h_K}
and the constants  $L_{2,2}$, $L_3$, $\DT$ defined in \eqref{def:M_h_M_g} and \eqref{def:V_1}. 
Using also Assumption~\ref{hyp:close_limit_setting} so that $\DT \leq
L_{2,2}\wedge L_3$, we get that for all $\theta \in \Theta_T$:
\[
\cK_T^{[2,2]}(\theta,\theta) \leq L_{2,2} + \DT \leq 2 L_{2,2}
\quad\text{and}\quad
\cK_T^{[3,3]}(\theta,\theta) \leq L_3 + \DT\leq  2 \, L_3.
\]
We deduce from  Lemma \ref{lem:bound_azais_final} with $C_1= \sqrt{2 L_{2,2}}$ and $C_2=\sqrt{2 L_3}$
and taking  $c_2 =  2\sqrt{2L_3} + 1$,
that:
\begin{equation}
\label{eq:bound_A_2}
\P\left (\mathcal{A}_2^c \right )= \P \left ( \sup_{ \Theta_T} |X_2| >
\mathcal{C}  \, \kappa \right )
\leq c_2 \left ( \sigma
\frac{|\Theta_T|_{\mathfrak{d}_T}\sqrt{ \Delta_T}}{\mathcal{C}  \,
	\kappa }\vee 1 \right )\,  \expp{-(\mathcal{C}  \, \kappa)^2/(8 
	\sigma^2 \Delta_T L_{2,2})}.
\end{equation}
\medskip

Since $\ca=\ca_0\cap \ca_1\cap \ca_2$, we deduce from
\eqref{eq:bound_A_0}, \eqref{eq:bound_A_1} and \eqref{eq:bound_A_2} that:
\[
\P\left ( \mathcal{A}^c \right ) =
\P\left
(\mathcal{A}_0^c \cup \mathcal{A}_1^c \cup \mathcal{A}_2^c \right )
\leq  \mathcal{C}_2' \, \left ( \sigma
\frac{|\Theta_T|_{\mathfrak{d}_T}\sqrt{ \Delta_T}}{\mathcal{C}\kappa }\vee 1
\right ) \, 
\expp{- \kappa^2/( \mathcal{C}_1^2 \, \sigma^2 \Delta_T)},
\]
with the  finite positive constants:
\[
\mathcal{C}_1 = \frac{2}{\mathcal{C}} \, \left(1 \vee \sqrt{2
	L_{2,2}}\right)
\quad\text{and}\quad
\mathcal{C}_2' = {c_0 +c_1 + c_2}  .
\]
By taking $\kappa \geq \mathcal{C}_1  \sigma \sqrt{\Delta_T \log \tau}$, for any positive constant $\tau > 1$, we get:
\begin{equation}
\label{eq:bounds_event_A}
\P\left (\mathcal{A}_0^c \cup \mathcal{A}_1^c \cup \mathcal{A}_2^c
\right )
\leq \mathcal{C}_2\left (  \frac{|\Theta_T|_{\mathfrak{d}_T}}{\tau
	\sqrt{\log \tau } }\vee \frac{1}{\tau}\right )
\quad\text{with}\quad
\mathcal{C}_2 = \mathcal{C}_2' \, \left ( \frac{1}{\mathcal{C} \mathcal{C}_1} \vee 1 \right).
\end{equation}
This  completes the proof of the theorem. 

\subsection{Proof of Theorem \ref{th:bounds} and of Equation \eqref{eq:main_theorem_diff_l1}} 
\label{sec:proof_theorem_bounds}
We keep notations from  Section \ref{sec:proof_main_theorem}.
Recall that Assumptions $\ref{hyp:theorem1_point1}$-$\ref{hyp:existence_certificate_theorem}$
of Theorem \ref{maintheorem} are in force.
We shall first provide an upper bound of $I_i$ for $i=0,1,2, 3$.
We deduce from \eqref{eq:functional_bound}, \eqref{eq:bounds_pq} and
\eqref{eq:bound_prediction_squared},  that, on the event $\mathcal{A}$:
\[
|\hat{\Upsilon}(p_0)|  \leq   \mathcal{C}_0 C_B \, \kappa \, \sparse, \quad 
|\hat{\Upsilon}(p_1)| \leq  \mathcal{C}_0 C_B \, \kappa \, \sparse
\quad\text{and}\quad
|\hat{\Upsilon}(q_0)|  \leq  \mathcal{C}_0 c_B \, \kappa \, \sparse. 
\]
Then, we obtain from \eqref{eq:bound_I_2_I_star} that, on the event $\mathcal{A}$:
\begin{equation}
\label{eq:bound_1_th2-1}
I_3  \leq  \mathcal{C}_5 \, \kappa \, \sparse
\quad\text{and}\quad
I_2  \leq   \mathcal{C}_6 \, \kappa \, \sparse
\quad\text{with}\quad
\mathcal{C}_5 = 2 \frac {\mathcal{C}'}{C_F} \mathcal{C}_0 (c_B + 2 C_B)
\quad\text{and}\quad 
\mathcal{C}_6 =\frac{ C_F}{C_N} \mathcal{C}_5. 
\end{equation}
This gives the third inequality in \eqref{eq:bound_th_abc}, as well as
Inequality~\eqref{eq:bound_th_I2}.  
We also deduce from \eqref{eq:I_0_I_3_bound_I_2_I_star} that, on the
event $\mathcal{A}$:
\begin{equation}
\label{eq:bound_1_th2_0}
I_0\leq    \mathcal{C}_4 \, \kappa \, \sparse
\quad\text{with}\quad
\mathcal{C}_4=C'_N \mathcal{C}_6+ (1-C_F) \mathcal{C}_5+  \mathcal{C}_0 C_B.
\end{equation}
This gives the second inequality in \eqref{eq:bound_th_abc}. 
\medskip

We now establish the first
inequality in \eqref{eq:bound_th_abc}. We deduce
from \eqref{eq:optim_bis} that:
\begin{equation}
\label{eq:chocolat}
\kappa(\|\hat{\beta}\|_{\ell_1} - \|\beta^{\star}\|_{\ell_1})
\leq \hat{\Upsilon}(w_T). 
\end{equation}
Then, using the bounds~\eqref {eq:bound_1_th2_0} and~\eqref{eq:bound_1_th2-1}
on $I_0$, $I_2$ and $I_3$,
we deduce from  \eqref{eq:w.diff-0} and~\eqref{eq:I_1_bound_I_2_I_star} that, on the event $\mathcal{A}$:
\begin{equation}
\label{eq:banane}
|\hat{\Upsilon}(w_T)| \leq  \mathcal{C}_7 \, s \, \kappa^2
\quad\text{with}\quad
\mathcal{C}_7 = \mathcal{C} \left(\mathcal{C}_4+ \mathcal{C}_5(1+ c_F) +
\mathcal{C}_6(1+ c_N) + \mathcal{C}_0 c_B
\right). 
\end{equation}
Thus,  \eqref{eq:chocolat} and \eqref{eq:banane} imply that, on the event $\mathcal{A}$:
\begin{equation}
\label{eq:kiwi}
\|\hat{\beta}\|_{\ell_1} - \|\beta^{\star}\|_{\ell_1} \leq \mathcal{C}_7 \, s \, \kappa.
\end{equation}
Then, use \eqref{eq:diff_l1_norm} and~\eqref{eq:bound_1_th2_0} to deduce that, on the event
$\mathcal{A}$:
\[
\big|	\|\hat{\beta}\|_{\ell_1} - \|\beta^{\star}\|_{\ell_1} \big|
\leq
(\mathcal{C}_4 \vee \mathcal{C}_7) \, s \, \kappa.
\]
This    proves    \eqref{eq:main_theorem_diff_l1}   (we    shall    take
$\mathcal{C}_3  =  \mathcal{C}_7  +   2  \mathcal{C}_4  $, see  below).   Let
$\mathcal{I}^+$   (resp.  $\mathcal{I}^-$)   be  the   set  of   indices
$k       \in      S^\star$       such       that      the       quantity
$\Big(\sum_{\ell     \in     \tilde{S}_{k}(r)     }     |\hat{\beta}_{\ell}|\Big)-
|\beta_{k}^{\star}|$  is non  negative  (resp. negative).   We have  the
following decomposition:
\begin{equation}
\begin{aligned}
\sum\limits_{k \in S^{\star}}
\Big| \sum\limits_{\ell \in \tilde{S}_{k}(r) }
|\hat{\beta}_{\ell}| - |\beta_{k}^{\star}| \Big |
&=  \sum\limits_{k \in \mathcal{I}^+} \Big(\sum\limits_{\ell \in \tilde{S}_{k}(r) }
|\hat{\beta}_{\ell}| - |\beta_{k}^{\star}|\Big)
+ \sum\limits_{k \in   \mathcal{I}^-}
\Big ( |\beta_{k}^{\star}|- \sum\limits_{\ell \in \tilde{S}_{k}(r) }
|\hat{\beta}_{\ell}|   \Big )\\ 
&\leq   \|\hat{\beta}\|_{\ell_1}  - \|\beta^{\star}\|_{\ell_1}  + 2
\sum\limits_{k \in \mathcal{I}^-}\Big (|\beta_{k}^{\star}|  - 
\sum\limits_{\ell \in
	\tilde{S}_{k}(r) } |\hat{\beta}_{\ell}| \Big )\\
&\leq \|\hat{\beta}\|_{\ell_1}  - \|\beta^{\star}\|_{\ell_1}  + 2 I_0.
\end{aligned}
\end{equation}
Then, use~\eqref{eq:bound_1_th2_0} and  \eqref{eq:kiwi} to obtain the
first inequality 
\eqref{eq:bound_th_abc} with $\mathcal{C}_3 =  \mathcal{C}_7 + 2
\mathcal{C}_4 $. This ends the proof of Theorem~\ref{th:bounds}.

\section{Construction of certificate functions}
\label{sec:proof_interpolating}

\subsection{Proof of Proposition \ref{prop:certificat_interpolating} (Construction of an interpolating certificate)}
\label{sec:proof_interpolating_part1}
This    section    is    devoted     to    the    proof    of    Proposition
\ref{prop:certificat_interpolating}.  We closely  follow the proof of
\cite{poon2018geometry}  taking into  account the  approximation of  the
kernel $\cK_T$ by the kernel $\cK_\infty$, which is measured through the
quantity $\mathcal{V}_T$ defined in \eqref{def:V_1}.

Let $T \in  \N$ and $s \in \N^*$. Recall Assumptions~\ref{hyp:g>0} (and thus \ref{hyp:reg-f} on the regularity of $\varphi_T$) and ~\ref{hyp:Theta_infini} on the regularity of the asymptotic kernel $\cK_{\infty}$ are in
force. Let $\rho \geq 1$, let  $r \in \left (0,1/\sqrt{2L_{0,2}} \right
)$  and $u_{\infty} \in \left (0,H_{\infty}^{(2)}(r,\rho) \right ) $
such that $ \ref{hyp:theorem_certificate_concavity},\,
\ref{hyp:theorem_certificate_separation},\,
\ref{hyp:theorem_certificate_metric}$ and
$\ref{hyp:theorem_certificate_approximation}$  of
Proposition~\ref{prop:certificat_interpolating} hold. 
We denote by $\norm{\cdot}_\op$ the operator norm associated to the
$\ell_\infty $ norm on $\R^\sparse$.

By assumption $\delta_\infty (u_{\infty},s)$ is finite.
Let $ \vartheta^\star = (\theta_1^\star, \ldots,
\theta_\sparse^\star)\in \Theta_{T,2 \rho_T \,
	\delta_{\infty}(u_{\infty},s) }^s$. We note $\cq^{\star}=
\{\theta^\star   _i,  \,   1\leq  i\leq   s\}$ the set of parameters of
cardinal $s$.
By   Lemma   \ref{lem:approx-delta},   we  have:
\[
\Theta^s_{T   ,  \rho_T   \delta_\infty   (u_{\infty},  s)}   \subseteq
\Theta^s_{T , \delta_T (u_T(s), s)}
\quad\text{where}\quad
u_T(s)=u_\infty +(s-1)  \mathcal{V}_T.
\]
Hence we have:
\begin{equation}
\label{eq:theta_set}
\vartheta^\star \in   \Theta^s_{T , \delta_T (u_T(s), s)}.
\end{equation}  
Set
\begin{equation}
\label{eq:def-mat-G}
\Gamma^{[i,j]}=\ck_T^{[i,j]}(\vartheta^\star)
\quad\text{and}\quad
\Gamma = \begin{pmatrix}
\Gamma^{[0,0]} & \Gamma^{[1,0]\top}\\
\Gamma^{[1, 0]} & \Gamma^{[1,1]}
\end{pmatrix}.
\end{equation}  

We deduce
from~\eqref{eq:def-delta-rs2} and~\eqref{eq:theta_set} that:
\begin{equation}
\label{eq:I_gamma001}
\norm{I - \Gamma^{[0,0]}}_{\op} \leq  u_T(s),
\quad
\norm{I - \Gamma^{[1,1]}}_{\op} \leq  u_T(s),
\quad
\norm{\Gamma^{[1,0]}}_{\op} \leq u_T(s)
\quad\text{and}\quad
\norm{\Gamma^{[1,0]\top}}_{\op} \leq u_T(s).
\end{equation} 

For
simplicity, for an  expression $A$ we write $A_T$  for $A_{\cK_T}$. 
Using this convention, recall the definition of the derivative operator
$\tD_{i; T}$ and write $\phi_T^{[1]}$ for $\tD_{1; T}[\phi_T]$. 

Let $\alpha=(\alpha_1, \ldots, \alpha_s)^\top$ and $ \coeff=(\coeff_1,
\ldots, \coeff_s)^\top$ be elements of $\R^s$. 
Let  $p_{\alpha,\coeff}$ be an element of $H_T$  defined by:
\begin{equation}
\label{eq:p_alpha_beta}
p_{\alpha,\coeff} = \sum\limits_{k=1}^{s} \alpha_{k}
\phi_{T}(\theta^{\star}_{k})
+ \sum\limits_{k=1}^{s} \coeff_{k} \, 
\phi^{[1]}_{T}(\theta^{\star}_{k}) ,
\end{equation}
and, using      \eqref{def:derivatives_kernel}       in
Lemma~\ref{lem:gT_consistent}, set   the  interpolating  real-valued  function
$\eta_{\alpha,\coeff}$ defined on $\Theta$ by:
\begin{equation}
\label{eq:def-h_alpha_beta}
\eta_{\alpha,\coeff}(\theta) =
\langle\phi_{T}(\theta),p_{\alpha,\coeff} \rangle_T =
\sum_{k=1}^{s} 
\alpha_{k}\,  \cK_T(\theta,\theta^{\star}_{k}) + \sum_{k=1}^{s}
\coeff_{k}\,  \cK_T^{[0,1]}(\theta,\theta^{\star}_{k}).
\end{equation}
By Assumption~\ref{hyp:g>0} on the regularity of $\varphi_T$ and the positivity of $g_T$ and
Lemma~\ref{lem:gT_consistent}, we get that the function $ \eta_{\alpha,\coeff} $ is of class
$\cc^3$  on  $\Theta$, and using~\eqref{eq:tDi}, we get that:
\begin{equation}
\label{eq:intervert2}
\eta_{\alpha, \coeff}^{[1]} := \tD_{1; T} [ \eta_{\alpha, \coeff}](\theta)=
\sum_{k=1}^{s} 
\alpha_{k}\,  \cK_T^{[1, 0]}(\theta,\theta^{\star}_{k}) + \sum_{k=1}^{s}
\coeff_{k}\,  \cK_T^{[1,1]}(\theta,\theta^{\star}_{k}).
\end{equation}

We  give   a  preliminary  technical  lemma. 

\begin{lem}
	\label{lem:ab-controle}
	Let $v = (v_1,\cdots,v_s)^\top \in  \{ -1,1\}^{s}$ be a sign
	vector. Assume that~\eqref{eq:I_gamma001} holds with
	$u_T(s)<1/2$. 
	Under Assumption~\ref{hyp:g>0},  there exist
	unique  $\alpha, \coeff \in \R^s$ such that: 
	\begin{align}
	\label{system1}
	\eta_{\alpha,\coeff}(\theta^{\star}_{k}) = v_{k} \in
	\{-1,1\} \quad \text{  and } \quad
	\eta_{\alpha,\coeff}^{[1]} (\theta^{\star}_{k}) = 0
	\quad \text{for} \quad 1 \leq k \leq s.
	\end{align}
	Furthermore, we have:
	\begin{equation}
	\label{eq:majo-ab}
	\norm{\alpha}_{\ve}\leq \frac{1-u_T(s)}{1-2u_T(s)}, \quad 
	\norm{\alpha- v}_{\ve} \leq \frac{u_T(s)}{1-2u_T(s)}
	\quad \text{and} \quad
	\norm{\coeff }_{\ve}\leq \frac{u_T(s)}{1-2u_T(s)}\cdot
	\end{equation}
\end{lem}

\begin{proof}[Proof of Lemma \ref{lem:ab-controle} ]
	Thanks to~\eqref{def:derivatives_kernel}, \eqref{eq:K-g}  and~\eqref{eq:intervert2},   we have:
	\begin{equation*}
	\left(	\eta_{\alpha,\coeff}(\theta^{\star}_{1}),\ldots,
	\eta_{\alpha,\coeff}(\theta^{\star}_{s}),
	\eta_{\alpha,\coeff}^{[1]}(\theta^{\star}_{1}),\ldots,
	\eta_{\alpha,\coeff}^{[1]}(\theta^{\star}_{s})\right)^\top
	= \Gamma  \begin{pmatrix} \alpha \\
	\coeff \end{pmatrix}.
	\end{equation*}
	Thus, solving \eqref{system1} is equivalent to solving,
	\begin{equation}
	\label{matrix_system1}
	\Gamma  \begin{pmatrix} \alpha \\
	\coeff \end{pmatrix}= \begin{pmatrix} v \\
	0_{s} \end{pmatrix}, 
	\end{equation}
	with $0_s$ the vector of size $s$ with all its components equal to zero.
	
	We  first show  that $\Gamma$  is non  singular so  that $\alpha$  and
	$\coeff$ exist and are uniquely defined.
	Using Lemma \ref{schur} based on the Schur complement, $\Gamma$ has an
	inverse        provided        that        $\Gamma^{[1,1]}$        and
	$              \Gamma_{SC}             :=\Gamma^{[0,0]}              -
	\Gamma^{[1,0]\top}[\Gamma^{[1,1]}]^{-1}\Gamma^{[1,0]}$     are     non
	singular.   We recall that if $M$ is a matrix
	such  that,  $\norm{I-M}_{\op}  <  1$,  then  $M$  is  non  singular,
	$M^{-1}       =        \sum\limits_{i\geq       0}(I-M)^{i}$       and
	$\norm{M^{-1}}_{\op} \leq \left (1- \norm{I-M}_{\op}\right)^{-1}$.
	\medskip

	Recall                 that                by                 assumption
	$u_T(s)\leq 1/2$.   Then, the second inequality in~\eqref{eq:I_gamma001} imply that  ${\norm{I - \Gamma^{[1,1]}}_{\op} <  1}$ and
	thus $\Gamma^{[1,1]}$ is non singular.   We now prove that $\Gamma_{SC}$
	is also non singular.  Using the triangle inequality we have:
	\begin{equation*}
	\begin{aligned}
	\norm{I - \Gamma_{SC}}_{\op}
	&= \norm{I - \Gamma^{[0,0]} +
		\Gamma^{[1,0]\top}[\Gamma^{[1,1]}]^{-1}\Gamma^{[1,0]}}_{\op}\\ 
	&\leq \norm{I - \Gamma^{[0,0]}}_{\op} +
	\norm{\Gamma^{[1,0]\top}[\Gamma^{[1,1]}]^{-1}\Gamma^{[1,0]}}_{\op}. 
	\end{aligned}
	\end{equation*}
	Let us bound the  terms on the right hand side  of the inequality above.
	To                                                                 bound
	$\norm{\Gamma^{[1,0]\top}   [\Gamma^{[1,1]}]^{-1}\Gamma^{[1,0]}}_{\op}$
	notice that:
	\begin{equation*}
	\begin{aligned}
	\norm{\Gamma^{[1,0]\top}[\Gamma^{[1,1]}]^{-1}\Gamma^{[1,0]}}_{\op} \leq ||\Gamma^{[1,0]}||_{\op}\norm{\Gamma^{[1,0]\top}}_{\op}\norm{[\Gamma^{[1,1]}]^{-1}}_{\op}.
	\end{aligned}
	\end{equation*}
	We have, thanks to \eqref{eq:I_gamma001} for the second inequality:
	\begin{equation}
	\label{eq:gamma__11-1}
	\norm{[\Gamma^{[1,1]}]^{-1}}_{\op} \leq  \frac{1}{1-\norm{I -
			\Gamma^{[1,1]}}_{\op} }\leq \frac{1}{1-u_T(s)} \cdot 
	\end{equation}
	Using~\eqref{eq:I_gamma001}, we get:
	\[
	||I - \Gamma_{SC}||_{\op} \leq u_T(s) + \frac{u_T(s)^{2}}{1-u_T(s)}
	= \frac{u_T(s)}{1-u_T(s)} \cdot
	\]
	By assumption, we have $u_T(s) \leq H_{\infty}^{(2)}(r,\rho) < 1/2$. Hence, we have  $\frac{u_T(s)}{1-u_T(s)} < 1$ and thus,  $\Gamma_{SC}$ is non singular.
	Furthermore, we get:
	\begin{equation}
	\label{eq:gamma_SC-1}
	||\Gamma_{SC}^{-1}||_{\op} \leq  \frac{1}{1-\norm{I -
			\Gamma_{SC}}_{\op} } \leq \frac{1-u_T(s)}{1-2u_T(s)}\cdot 
	\end{equation}
	As the matrices $\Gamma^{[1,1]}$ and $ \Gamma_{SC}$ are non singular, we
	deduce that the matrix $\Gamma$ is non singular.  \medskip
	
	We now give bounds related to $\alpha$ and ${ \coeff}$.
	The Lemma \ref{schur} on the Schur complement gives also that:
	\[
	\alpha = \Gamma_{SC}^{-1}v
	\quad \text{and} \quad
	{\coeff} = -[\Gamma^{[1,1]}]^{-1}\Gamma^{[1,0]}\Gamma_{SC}^{-1}v.
	\]
	Hence, we deduce that:
	\begin{equation*}
	\begin{aligned}
	\norm{\alpha}_{\ve}
	& \leq \norm{\Gamma_{SC}^{-1}}_{\op} \norm{v}_{\ve} \leq
	\frac{1-u_T(s)}{1-2u_T(s)}, \\
	\norm{{\coeff}}_{\ve}
	&\leq
	\norm{[\Gamma^{[1,1]}]^{-1}\Gamma^{[1,0]}\Gamma_{SC}^{-1}}_{\op}\norm{v}_{\ve}
	\leq
	\norm{[\Gamma^{[1,1]}]^{-1}}_{\op}\norm{\Gamma^{[1,0]}}_{\op}\norm{\Gamma_{SC}^{-1}}_{\op}
	\leq \frac{u_T(s)}{1-2u_T(s)},\\ 
	\norm{\alpha-v }_{\ve}
	&\leq  \norm{(\Gamma_{SC}^{-1}-I)}_{\op} \norm{v}_{\ve} \leq
	\norm{\Gamma_{SC} -I}_{\op} \norm{\Gamma_{SC}^{-1}}_{\op} \leq
	\frac{u_T(s)}{1-2u_T(s)}.
	\end{aligned}
	\end{equation*}
	This finishes the proof.
\end{proof}

We now fix  a sign vector $v = (v_1,\cdots,v_s)^\top  \in \{ -1,1\}^{s}$
and consider $p_{\alpha,\coeff}$  and $\eta_{\alpha,\coeff}$ with $\alpha$
and    $\coeff$    characterized    by   \eqref{system1}    from    Lemma
\ref{lem:ab-controle}.  Let  $\El\in \R^s$  be the  vector with  all the
entries equal to zero  but the $\ell$-th which is equal  to 1.
\medskip

\textbf{Proof of $\ref{it:as1-<1-c}$ from Assumption~\ref{assumption1}} \Put $C_F = {\varepsilon_{\infty}(r/\rho)}/{10}$. Let  $\theta \in \Theta_T $
such  that $\mathfrak{d}_T(\theta,\cq^\star)  > r$  (far region).  It is
enough to prove that $|\eta_{\alpha,\coeff}(\theta)| \leq 1 - C_{F}$. Let
$\theta^{\star}_{\ell}$ be one of  the elements of $\cq^{\star}$ closest
to $\theta$ in terms of the metric $\mathfrak{d}_T$.
Since $\vartheta^\star \in   \Theta_{T,2\rho_T\delta_{\infty}(u_{\infty},s) }^s$, we have, by the triangle inequality that for any $k\neq \ell$:
\begin{equation*}
2 \rho_T \, \delta_{\infty}(u_{\infty},s) <	\mathfrak{d}_T(\theta_\ell^\star,\theta_k^\star) \leq \mathfrak{d}_T(\theta_\ell^\star,\theta) + \mathfrak{d}_T(\theta,\theta_k^\star) \leq  2 \mathfrak{d}_T(\theta,\theta_k^\star).
\end{equation*} 
Hence,                              we                              have
$\vartheta^\star_{\ell,\theta}                                       \in
\Theta_{T,\rho_T\delta_{\infty}(u_{\infty},s)         }^s$,        where
$\vartheta^\star_{\ell,  \theta}$ denotes  the vector  $\vartheta^\star$
whose  $\ell$-th coordinate  has  been replaced  by
$\theta$. Then,  we
obtain        from        Lemma       \ref{lem:approx-delta}        that
$\Theta^s_{T   ,  \rho_T   \delta_\infty   (u_{\infty},  s)}   \subseteq
\Theta^s_{T , \delta_T (u_T(s), s)}$ and thus:
\begin{equation}
\label{eq:theta_l_in_Theta}
\vartheta^\star_{\ell,\theta} \in   \Theta^s_{T , \delta_T (u_T(s), s)}. 
\end{equation}
We       denote       by      $\Gamma_{\ell,       \theta}$       (resp.
$\Gamma_{\ell,   \theta}   ^{[i,j]}$)   the   matrix   $\Gamma$   (resp.
$\Gamma ^{[i,j]}$) in~\eqref{eq:def-mat-G}  where $ \vartheta^\star$ has
been  replaced by  $  \vartheta^\star_{\ell,\theta}$.  Notice the  upper
bounds~\eqref{eq:I_gamma001} also hold for $\Gamma_{\ell, \theta}$ because
of  \eqref{eq:theta_l_in_Theta}.  Recall we have Equalities \eqref{eq:formula_K_00} on the diagonal of the kernel $\cK_T$ and its derivatives. Elementary   calculations  give  with
$\eta_{\alpha,\coeff}$ from Lemma \ref{lem:ab-controle} that:
\begin{equation}
\label{eq:eta-a-v}
\eta_{\alpha,\coeff}(\theta)
= \El^\top \left( \Gamma^{[0, 0]}_{\ell, \theta} - I\right) \alpha + \cK_T(\theta,
\theta^\star_\ell) \alpha_\ell + \El^\top  \Gamma^{[1, 0]\top}_{\ell,
	\theta}  \coeff+ \cK_T^{[0, 1]} (\theta,
\theta^\star_\ell) \coeff_\ell.
\end{equation}
We deduce that:
\begin{equation}
\label{eq:|eta|}
|\eta_{\alpha,\coeff}(\theta)|
\leq  \norm{ \Gamma^{[0, 0]}_{\ell, \theta} - I}_{\op}
\norm{\alpha}_\ve
+ \norm{\alpha}_{\ve} |\cK_T(\theta,
\theta^\star_\ell)|+  \norm{\Gamma^{[1, 0]\top}_{\ell,
		\theta}}_\op \norm{\coeff}_\ve +  |\cK_T^{[0, 1]} (\theta,
\theta^\star_\ell)| \norm{\coeff}_{\ve}.
\end{equation}
Since $\theta$ belongs to the ``far region", we have by definition of $\varepsilon_{T}(r)$ given in \eqref{eq:def-e0}  that:
\begin{equation}
\label{eq:bound_kernel_2}
|\cK_T(\theta,\theta^{\star}_{\ell})|  \leq 1 - \varepsilon_{T}(r).
\end{equation}
The triangle inequality, the definitions \eqref{def:V_1} of $\mathcal{V}_T$ and \eqref{def:M_h_M_g} of $L_{1,0}$, give:
\begin{equation}
\label{eq:bound_kernel_3}
|\cK_T^{[0,1]}(\theta, \theta^{\star}_{\ell})| \leq L_{0,1} + \mathcal{V}_T.
\end{equation}
Then, using~\eqref{eq:I_gamma001} (which holds
for  $\Gamma_{\ell, \theta}$ thanks to~\eqref{eq:theta_l_in_Theta}), 
we get that:
\[
|\eta_{\alpha,\coeff}(\theta)|
\leq 1 -  \varepsilon_{T}(r) +  \frac{u_T(s)}{1- 2
	u_T(s)} \left(2 +  L_{1,0} + \mathcal{V}_T\right).
\]
Notice that the function $r \mapsto \varepsilon_{\infty}(r)$ is
increasing. Since  $\rho_T \leq \rho$, we get by Lemma \ref{lem:comp_epsilon} that:
\begin{equation}
\label{eq:epsilon}
\varepsilon_{T} (r)\geq  \varepsilon_{\infty}(r/\rho_T)- \mathcal{V}_T \geq \varepsilon_{\infty}(r/\rho)- \mathcal{V}_T. 
\end{equation} 
By assumption, we have $u_T(s) \leq H_{\infty}^{(2)}(r,\rho) \leq 1/4$.
Hence, we have $\frac{1}{1-2u_T(s)}\leq 2$.
We also have $\mathcal{V}_T \leq 1/2$.
Therefore, we get:
\[
|\eta_{\alpha,\coeff}(\theta)|
\leq  1 - \varepsilon_{\infty}(r/\rho) + \mathcal{V}_T +   u_T(s)  \left
( 5+ 2L_{1,0} \right) . 
\]
The assumption $u_T(s) \leq H_{\infty}^{(2)}(r,\rho)$ gives:
\begin{equation}
u_T(s)  \leq \frac{8}{10  \, (5 + 2 L_{1,0})} \varepsilon_{\infty} (r/\rho)\cdot
\end{equation}
The   assumption   $\mathcal{V}_T   \leq   H_{\infty}^{(1)}(r,\rho)$   gives
$\mathcal{V}_T \leq  \varepsilon_{\infty} (r/\rho)/10$.  Hence,
we                                                                  have
$|\eta_{\alpha,\coeff}(\theta)|\leq                                    1-
\frac{\varepsilon_{\infty}(r/\rho)}{10}$.
Thus, Property $\ref{it:as1-<1-c}$  from Assumption \ref{assumption1}
holds with $C_F ={\varepsilon_{\infty}(r/\rho)}/{10}$.  
\medskip

\textbf{Proof of $ \ref{it:as1-<1} $ from Assumption  \ref{assumption1}}  \Put  $C_{N} = {\nu_{\infty}(\rho r)}/{180}$. Let $\theta \in \Theta_T$ such that $\mathfrak{d}_T(\theta,\cq^\star) \leq r$. Let $\ell \in \{1,\cdots,s\}$ such that  $\theta\in
\mathcal{B}_T(\theta_\ell^{\star},r)$ (``near region"). Thus, it is enough to prove that $|\eta_{\alpha,\coeff}(\theta)| \leq 1 -
C_{N}\, \mathfrak{d}_T(\theta_\ell^{\star},\theta)^{2}$.
This will be done by using Lemma  \ref{quadratic_decay} to obtain a quadratic decay on $\eta_{\alpha,\coeff}$ from a bound on its second Riemannian  derivative.
\medskip

Recall  that
the function $\eta_{\alpha,\coeff}$ is twice continuously
differentiable. Set $\eta_{\alpha,\coeff}^{[2]} =
\tilde{D}_{2;T}[\eta_{\alpha,\coeff}]$. Differentiating~\eqref{eq:intervert2}
and using that $\ck_T^{[2,0]}(\theta, \theta)=-1$ and
$\cK_T^{[2,1]}(\theta,\theta)=0$, see~\eqref{eq:formula_K_00}, we deduce that:
\begin{equation}
\label{eq:def_D2}
\eta_{\alpha,\coeff}^{[2]}(\theta)
=   \El^\top(I+\Gamma^{[2, 0]}_{\ell, \theta})\alpha
+\cK_T^{[2,0]}(\theta,\theta^{\star}_{\ell}) \El^\top\alpha+ 
\El^\top\Gamma^{[2, 1]}_{\ell, \theta}\coeff+
\cK_T^{[2,1]}(\theta,\theta^{\star}_{\ell})\El^\top \coeff . 
\end{equation}
Since   $v = (v_1,\cdots,v_s)^\top \in  \{ -1,1\}^{s}$ is a sign vector,
we get:
\begin{equation}
\label{eq:def_D2-next}
\eta_{\alpha,\coeff}^{[2]}(\theta) -   v_{\ell}\cK_T^{[2,0]}(\theta,\theta^{\star}_{\ell})
= 
\El^\top(I + \Gamma^{[2, 0]}_{\ell, \theta})\alpha
+\cK_T^{[2,0]}(\theta,\theta^{\star}_{\ell})\El^\top (\alpha- v)
+ \El^\top\Gamma^{[2, 1]}_{\ell, \theta}\coeff
+   \cK_T^{[2,1]}(\theta,\theta^{\star}_{\ell})\El^\top \coeff 
. 
\end{equation}
The triangle inequality and the definition of $\mathcal{V}_T$ give:
\begin{equation}
\label{eq:bound_L_V}
| \cK_T^{[2,0]}(\theta,\theta^{\star}_{\ell}) | \leq L_{2,0} + \mathcal{V}_T \quad \text{and} \quad | \cK_T^{[2,1]}(\theta,\theta^{\star}_{\ell})  | \leq L_{2,1} + \mathcal{V}_T,
\end{equation}
where $ L_{2,0}$ and $ L_{1,2}$ are defined in \eqref{def:M_h_M_g}.
We deduce from \eqref{eq:theta_l_in_Theta}, the definition of $\delta_T$
in~\eqref{eq:def-delta-rs2}
and~\eqref{eq:def-MT} that:
\begin{equation}
\label{eq:bound_u}
\norm{I + \Gamma^{[2, 0]}_{\ell, \theta}}_\op \leq u_T(s) \quad \text{and}  \quad \norm{\Gamma^{[2, 1]}_{\ell, \theta}}_\op \leq u_T(s) .
\end{equation}
We deduce from~\eqref{eq:def_D2-next} that:
\begin{align*}
|\eta_{\alpha,\coeff}^{[2]}(\theta) -
v_{\ell}\cK_T^{[2,0]}(\theta,\theta^{\star}_{\ell})|
&  \leq
\left \|\alpha \right \|_{\ve} \norm{I + \Gamma^{[2, 0]}_{\ell,
		\theta}}_\op
+  \norm{\alpha-v}_\ve  (L_{2,0}+ \mathcal{V}_T)
+ \norm{\coeff}_\ve \left(\!  \norm{\Gamma^{[2, 1]}_{\ell, \theta}}_\op
+ L_{2,1} + \mathcal{V}_T\!\!\right)\\
&\leq   \frac{ u_T(s)}{1-2u_T(s)}  (1+ L_{2,0} + L_{2,1} +
2\mathcal{V}_T). 
\end{align*}
By assumption, we have $u_T(s) \leq H_{\infty}^{(2)}(r,\rho) \leq 1/4$. Hence, we have $\frac{1}{1-2u_T(s)}\leq 2$. Furthermore, we have by assumption $\mathcal{V}_T \leq H_{\infty}^{(1)}(r,\rho) \leq 1/2$ and $u_T(s) \leq H_{\infty}^{(2)}(r,\rho)$. 
In particular, we have:
\[
u_T(s) \leq \frac{8}{ 9(2L_{2,0} + 2L_{2,1} + 4)}
\nu_{\infty}(\rho r).
\]
Therefore, we obtain:
\begin{equation}
\label{eq:bound_quadratic_decay}
|\eta_{\alpha,\coeff}^{[2]}(\theta) -
v_{\ell}\cK_T^{[2,0]}(\theta,\theta^{\star}_{\ell})|
\leq   \frac{8}{9}\nu_{\infty}(\rho r) .
\end{equation}
\medskip

We now  check that the hypotheses of Lemma \ref{quadratic_decay}-$\ref{lem:decay_point_ii}$ hold in order to obtain a quadratic decay on $\eta_{\alpha,\coeff}$ from the bound \eqref{eq:bound_quadratic_decay}. First recall that $\eta_{\alpha,\coeff}$ is twice continuously differentiable and have the interpolation properties \eqref{system1}.
By the triangle inequality and since by assumption $\mathcal{V}_T \leq L_{2,0}$ we have:
\begin{equation*}
\underset{ \Theta_T^2}{\sup} | \cK_T^{[2,0]} | \leq L_{2,0} + \mathcal{V}_T \leq 2L_{2,0}.
\end{equation*} 
Then, Lemma \ref{lem:comp_epsilon} ensures that
for any $\theta,\theta'$ in $\Theta_T$ such that
$\mathfrak{d}_T(\theta,\theta') \leq  r$ we have: 
\begin{equation*}
-\cK_T^{[2,0]}(\theta,\theta') \geq \nu_{\infty}(r \rho_T) - \mathcal{V}_T\geq \nu_{\infty}(\rho r) - \mathcal{V}_T\geq  \frac{9}{10}\nu_{\infty}(\rho r),
\end{equation*}
where we used that that the function $r \mapsto \nu_{\infty}(r)$ is
decreasing and $\rho_T \leq \rho$ for the second inequality and that
$\mathcal{V}_T \leq H_{\infty}^{(1)}(r,\rho) \leq \nu_{\infty}(\rho r) /10$
for the last inequality.

Set $\delta = \frac{8}{9} \nu_{\infty}(\rho r) $,  $\varepsilon = \frac{9}{10} \nu_{\infty}(\rho r)$, $L = 2 L_{2,0}$. As $r < L^{-\frac{1}{2}}$ and  $\delta< \varepsilon$,
we apply Lemma \ref{quadratic_decay}-$\ref{lem:decay_point_ii}$ and get for $\theta\in
\mathcal{B}_T(\theta_\ell^{\star},r)$:
\begin{equation*}
|\eta_{\alpha,\coeff}(\theta)| \leq 1 -  \frac{\nu_{\infty}(\rho r)}{180}\, \mathfrak{d}_T(\theta,\theta_{\ell}^{\star})^{2}.
\end{equation*}

\medskip

\textbf{Proof   of $ \ref{it:as1-ordre=2}$  from   Assumption
	\ref{assumption1}}  \Put $C_N'  =  (5L_{2,0} +  L_{2,1}  + 4)/8$.  Let
$\theta           \in            \Theta_T$           such           that
$\mathfrak{d}_T(\theta,\cq^\star) \leq r$. Let $\ell \in \{1,\cdots,s\}$
such   that  $\theta\in   \mathcal{B}_T(\theta_\ell^{\star},r)$  (``near
region").            We            shall           prove            that
$   |   \eta_{\alpha,\coeff}(\theta)   -   v_{\ell}|   \leq   C_{N}'   \,
\mathfrak{d}_T(\theta_\ell^{\star},\theta)^{2}$.  \medskip

Let            us            consider            the            function
$f  : \theta  \rightarrow \eta_{\alpha,\coeff}(\theta)  - v_{\ell}  $. We
will bound the second covariant derivative $f^{[2]}=\tD_{2;T}[f]$ of $f$ and apply
Lemma \ref{quadratic_decay}-$\ref{lem:decay_point_i}$ on  $f$ to prove the
property  $\ref{it:as1-ordre=2}$ for  $\eta_{\alpha,\coeff}$.  Notice  that
$f$    is   twice    continuously   differentiable.    By
construction, see
\eqref{system1},    we    have    $f(\theta^\star_\ell)    =    0$    and
$f^{[1]}(\theta^\star_\ell) =  0$.
Since $f^{[2]} = \eta_{\alpha,\coeff}^{[2]}$, we deduce from
\eqref{eq:def_D2}, the bounds \eqref{eq:bound_L_V} that:
\[
|f^{[2]}(\theta)|\leq
\norm{\alpha}_\ve \norm{I+\Gamma^{[2, 0]}_{\ell, \theta}}_\op
+ \norm{\alpha}_\ve (L_{2,0}+ \mathcal{V}_T)
+  \norm{\coeff}_\ve   \norm{\Gamma^{[2, 1]}_{\ell, \theta}}_\op
+  \norm{\coeff}_\ve (L_{2,1} + \mathcal{V}_T)
.  
\]
Using
\eqref{eq:bound_u},  and the bounds on $\alpha$ and
${\coeff}$ from Lemma \ref{lem:ab-controle}, we get: 
\[
|f^{[2]}(\theta)|\leq 
\frac{1-u_T(s)}{1-2u_T(s)}(L_{2,0} + \mathcal{V}_T+ u_T(s)) +
\frac{u_T(s)}{1-2u_T(s)} (L_{2,1} + \mathcal{V}_T + u_T(s)).
\]
Since  $u_T(s) \leq H_{\infty}^{(2)}(r,\rho)\leq {1}/{6}$ and
$\mathcal{V}_T\leq H_{\infty}^{(1)}(r,\rho) \leq 1/2$, we get:
\[
|f^{[2]}(\theta)|
\leq\frac{5}{4} L_{2,0} + 
\frac{1}{4}  L_{2,1} + 1 .  
\]
We get 
thanks to Lemma \ref{quadratic_decay}-$\ref{lem:decay_point_i}$ on the
function $f$ that for any $\theta\in 
\mathcal{B}_T(\theta_\ell^{\star},r)$:
\begin{equation*}
|\eta_{\alpha,\coeff}(\theta) - v_{\ell}|  \leq  \frac{1}{8}\,
\left (5L_{2,0} + L_{1,2} + 4 \right )\,
\mathfrak{d}_T(\theta,\theta_{\ell}^{\star})^{2}. 
\end{equation*}

\medskip

\textbf{Proof of $ \ref{it:norm<c}$ from  Assumption
	\ref{assumption1}} \Put $C_  B=2$.
Recall  the definition  of
$p_{\alpha,\coeff}$  in \eqref{eq:p_alpha_beta}.  Elementary calculations
give  using the  definitions of  $\Gamma^{[0,0]}$, $\Gamma^{[1,1]}$  and
$\Gamma^{[1,1]}$ in~\eqref{eq:def-mat-G}:
\begin{align*}
\norm{p_{\alpha,\coeff}}_{T}^2
&\leq  2
\norm{\sum\limits_{k=1}^{s} \alpha_{k} 
	\phi_{T}(\theta^{\star}_{k})}_T^2
+ 2\norm{\sum\limits_{k=1}^{s} \coeff_{k} \,
	\phi_{T}^{[1]}(\theta^{\star}_{k}) }_T^2 \\
&= 2\alpha^\top\Gamma^{[0,0]}\alpha + 2   \coeff^\top\Gamma^{[1,1]} \coeff\\
& \leq 2 \norm{\alpha}_{\ell_1}\norm{\alpha}_{\ve}\,
\norm{\Gamma^{[0,0]} }_{\op}  +
2\norm{\coeff}_{\ell_1}\norm{\coeff}_{\ve}\, \norm{\Gamma^{[1,1]} }_{\op}.
\end{align*}
Using that $\norm{I}_\op=1$ and~\eqref{eq:I_gamma001}, we get that:
\[
\norm{\Gamma^{[0,0]} }_{\op}\leq (1+u_T(s))
\quad\text{and}\quad
\norm{\Gamma^{[1,1]} }_{\op}\leq (1+u_T(s)).
\]
By assumption we have $u_T(s) \leq H_{\infty}^{(2)}(r,\rho) \leq \frac{1}{6}$. 
We deduce that:
\[
\norm{p_{\alpha,\coeff}}_{T}^2   
\leq  2 (1+u_T(s)) \frac{(1- u_T(s))^2 + u_T(s)^2}{(1- 2u_T(s))^2}
s\leq  4 s.
\]
This gives:
\begin{equation}
\label{eq:majo-pab}
\norm{p_{\alpha,\coeff}}_{T}\leq  2 \sqrt{s}.
\end{equation}

We proved that  $\ref{it:as1-<1}$-$\ref{it:norm<c}$  from  Assumption~\ref{assumption1} stand. By assumption we also have that for all $\theta\neq \theta'\in \cq^{\star}: \dT(\theta, \theta')> 2 \, r$, therefore  Assumption~\ref{assumption1} holds.

This finishes the proof of
Proposition~\ref{prop:certificat_interpolating}.


\subsection{Proof of Proposition \ref{prop:certificat2} (Construction of an interpolating derivative certificate)}
\label{sec:proof_derivative}

This   section    is   devoted   to    the   proof   of    Proposition
\ref{prop:certificat2} and is close to Section \ref{sec:proof_interpolating_part1}. Let $T\in \N$  and $\sparse \in \N^*$. Recall
Assumptions~\ref{hyp:g>0} (and  thus ~\ref{hyp:reg-f} on  the regularity
of  $\varphi_T$) and  ~\ref{hyp:Theta_infini} on  the regularity  of the
limit   kernel   $\ck_{\infty}$  are   in   force.   Set
$u'_{\infty} \in (0, 1/6)$.  We denote by $\norm{\cdot}_\op$ the operator
norm  associated  to  the  $\ell_\infty  $  norm  on  $\R^\sparse$.   By
assumption    $\delta_\infty    (u'_{\infty},s)$     is    finite.     Let
$   \vartheta^\star   =  (\theta_1^\star,   \ldots,   \theta_\sparse^\star)\in
\Theta_{T,2  \rho_T  \,  \delta_{\infty}(u_{\infty}',s)  }^s$.   We  note
$\cq^{\star}=  \{\theta^\star  _i,  \,  1\leq  i\leq  s\}$  the  set  of
parameters           of           cardinal           $s$.            Let
$\alpha=(\alpha_1,          \ldots,         \alpha_s)^\top$          and
$  \coeff=(\coeff_1,  \ldots, \coeff_s)^\top$  be  elements  of $\R^s$.
Recall $p_{\alpha, \coeff}$,
$\eta_{\alpha, \coeff}$ and  $\eta_{\alpha,\coeff}^{[1]}=\tD_{1; T}
[\eta_{\alpha,\coeff}]$ given by~\eqref{eq:p_alpha_beta},
\eqref{eq:def-h_alpha_beta} and~\eqref{eq:intervert2}.

\medskip

The next lemma is similar to Lemma \ref{lem:ab-controle},
but notice that in Lemma  \ref{lem:ab-controle_2} the function $\eta_{\alpha,\coeff}$
vanished on  $\cq^\star$ and  has a derivative  that interpolates  a sign
vector,   whereas  in  Lemma  \ref{lem:ab-controle} it  is  the  opposite.

Recall the definition of $\DT$ from \eqref{def:V_1} and define  $u_T'(s)=u_\infty' +(s-1)  \mathcal{V}_T$. We remark that \eqref{eq:I_gamma001} holds with $u_T(s)$ replaced by $u_T'(s)$ because of \eqref{eq:theta_set}.
\begin{lem}
	\label{lem:ab-controle_2}
	Let  $v  =   (v_1,\cdots,v_s)^\top  \in  \{  -1,1\}^{s}$   be  a  sign
	vector. Assume that~\eqref{eq:I_gamma001} holds with $u_T(s)$ replaced by
	$u'_T(s)<1/2$. 
	Under Assumption~\ref{hyp:g>0},  there exist
	unique  $\alpha, \coeff \in \R^s$ such that:
	\begin{align}
	\label{system1_bis}
	\eta_{\alpha,\coeff}(\theta^{\star}_{k}) &= 0  \quad \text{  and } \quad 
	\eta_{\alpha,\coeff}^{[1]}(\theta^{\star}_{k}) = v_{k} \quad \text{for} \quad 1 \leq k \leq s.
	\end{align}
	Furthermore, we have:
	\begin{equation}
	\label{eq:ab-controle_2}
	\begin{aligned}
	& \norm{\alpha}_\ve\leq \frac{u_T'(s)}{1-2u_T'(s)}  \quad \text{and} \quad
	\norm{\coeff}_\ve \leq \frac {1-u_T'(s)}{1-2u_T'(s)} .
	\end{aligned}
	\end{equation}
\end{lem}

\begin{proof}
	Thus, with   $0_s$ the vector of size $s$ with all its components equal to zero and $ \Gamma$  defined by  \eqref{eq:def-mat-G}, Equation~\eqref{system1_bis} is  equivalent to:
	\begin{equation}
	\label{matrix_system22}
	\Gamma \begin{pmatrix} \alpha \\
	\coeff \end{pmatrix}= \begin{pmatrix} 0_{s} \\
	v\end{pmatrix} \cdot
	\end{equation}
	According to the proof of Lemma~\ref{lem:ab-controle}, the matrices
	$\Gamma_{SC}  =\Gamma^{[0,0]} -
	\Gamma^{[1,0]\top}[\Gamma^{[1,1]}]^{-1}\Gamma^{[1,0]}$,
	$\Gamma^{[1,1]}$ and $\Gamma$ are non singular. Thus  the vectors
	$\alpha$ and $\coeff$ exist and are uniquely determined
	by~\eqref{matrix_system22}. From Lemma~\ref{schur}, we deduce that:
	\[
	\alpha =
	-\Gamma_{SC}^{-1}\Gamma^{[1,0]\top}[\Gamma^{[1,1]}]^{-1}v \quad
	\text{and} \quad \coeff = \left ( I +
	[\Gamma^{[1,1]}]^{-1}\Gamma^{[1,0]}\Gamma_{SC}^{-1}\Gamma^{[1,0]\top}
	\right )[\Gamma^{[1,1]}]^{-1}v  . 
	\]
	Using       \eqref{eq:gamma_SC-1},       \eqref{eq:I_gamma001}       and
	\eqref{eq:gamma__11-1} and replacing $u_T(s)$ by $u_T'(s)$,    we    easily     obtain    the     inequalities
	\eqref{eq:ab-controle_2}.
\end{proof}

We fix the sign vector $v = (v_1,\cdots,v_s)^\top \in  \{ -1,1\}^{s}$
and consider $p_{\alpha,\coeff}$ and $\eta_{\alpha,\coeff}$ given
by~\eqref{eq:p_alpha_beta} and~\eqref{eq:def-h_alpha_beta}, with
$\alpha$ and $\coeff$ given by Lemma~\ref{lem:ab-controle_2}.
\medskip

\textbf{Proof of $ \ref{it:as2-ordre=2}$ from Assumption
	\ref{assumption2}} \Put $c_N =  (L_{0,2} + L_{2,1}
+ 7)/8$. We define  the function  $f: \theta \mapsto \eta_{\alpha,\coeff}(\theta) -
v_{\ell}\,
\operatorname{sign}(\theta-\theta_\ell^{\star})\mathfrak{d}_T(\theta,\theta_\ell^{\star})$
on $\Theta$. To prove the Property $\ref{it:as2-ordre=2}$, we will bound
the second covariant derivative of $f$,  that is $f^{[2]} := \tilde
D_{2;T}[f]$, and apply Lemma
\ref{quadratic_decay}-$\ref{lem:decay_point_i}$. 
Recall  $\mathfrak{d}_T(\theta,\theta_\ell^{\star})= |G_T(\theta) -
G_T(\theta_\ell^\star)|$ with $G_T$ a primitive of $\sqrt{g_T}$, and
thus $f(\theta)=\eta_{\alpha, \coeff}(\theta) - v_\ell(G_T(\theta)-
G_T(\theta_\ell^\star))$. We deduce that $f$ is twice continuously
differentiable on $\Theta$; and elementary calculations give $f^{[2]}=
\eta_{\alpha, \coeff}^{[2]}$.

Let        $\theta         \in        \Theta_T$ and let 
$\theta^{\star}_{\ell}$ be one of the elements of $\cq^{\star}$ 
closest to $\theta$ in terms of the metric $\mathfrak{d}_T$. Recall the notations      $\Gamma_{\ell,       \theta}$       (resp.
$\Gamma_{\ell,   \theta}   ^{[i,j]}$) and $  \vartheta^\star_{\ell,\theta}$ from the proof of Proposition \ref{prop:certificat_interpolating}.  Since   $f^{[2]}=  \eta_{\alpha,  \coeff}^{[2]}$,   we  deduce
from~\eqref{eq:def_D2} that:
\begin{equation}
\label{eq:bound_eta}
|f^{[2]}(\theta)|
\leq  \norm{I+\Gamma^{[2, 0]}_{\ell, \theta}}_\op \norm{\alpha}_\ve
+ \norm{\alpha}_\ve|\cK_T^{[2,0]}(\theta,\theta^{\star}_{\ell})| + 
\norm{\coeff}_\ve \norm{\Gamma^{[2, 1]}_{\ell, \theta}}_\op+
\norm{\coeff}_\ve  |\cK_T^{[2,1]}(\theta,\theta^{\star}_{\ell})| .
\end{equation}
Notice that \eqref{eq:theta_l_in_Theta} holds with $u_T(s)$ replaced by $u_T'(s)$. Using~\eqref{eq:bound_L_V}      and~\eqref{eq:bound_u}      and      the
bounds~\eqref{eq:ab-controle_2}  on $\alpha$  and  ${\coeff}$ from  Lemma
\ref{lem:ab-controle_2}, we get:
\[
|f^{[2]}(\theta)|
\leq \frac{u_T'(s)}{1-2u_T'(s)} (L_{2,0} + \mathcal{V}_T + u_T'(s)) +
\frac{1-u_T'(s)}{1-2u_T'(s)} (L_{2,1} + \mathcal{V}_T + u_T'(s)) . 
\]
By assumption,  we have $u'_T(s) \leq 1/6$ and $\mathcal{V}_T \leq 1 $. Hence, we obtain:
\[ 
|f^{[2]}(\theta)|
\leq \frac{1}{4}L_{2,0} + \frac{5}{4}L_{2,1}  + \frac{7}{4}\cdot
\]

Since $f(\theta^\star_\ell)=0$ and $f^{[1]}(\theta^\star_\ell)=0$ as
well, using  Lemma~\ref{quadratic_decay}~$\ref{lem:decay_point_i}$, we
get, with $c_N = (L_{2,0} + 5L_{2,1}  + 7)/8$:
\[
\big |\eta_{\alpha,\coeff}(\theta)-v_{\ell}\,
\operatorname{sign}(\theta-\theta_\ell^{\star}) \,
\mathfrak{d}_T(\theta,\theta_\ell^{\star}) \big |
=  |f(\theta) |\leq c_N \, \mathfrak{d}_T(\theta,\theta_{\ell}^{\star})^{2}.
\]
\medskip

\textbf{Proof of  $\ref{it:as2-<1-c} $  from Assumption
	\ref{assumption2}}  \Put $c_F = (5 L_{1,0}+ 7)/4 $. Let  $\theta \in
\Theta_T $, we shall prove that $|\eta_{\alpha,\coeff}(\theta)| \leq c_{F}$. Let 
$\theta^{\star}_{\ell}$ be one of the elements of $\cq^{\star}$ 
closest to $\theta$ in terms of the metric $\mathfrak{d}_T$.
We deduce from~\eqref{eq:eta-a-v} that:
\[
|\eta_{\alpha,\coeff}(\theta)| 
\leq  \norm{\alpha}_\ve \norm{\Gamma^{[0, 0]}_{\ell, \theta} - I}_\op +
\norm{\alpha}_\ve |\cK_T(\theta, 
\theta^\star_\ell)| +  \norm{\coeff}_\ve \norm{ \Gamma^{[1, 0]\top}_{\ell,
		\theta}}_\op +  \norm{\coeff}_\ve|\cK_T^{[0, 1]} (\theta,
\theta^\star_\ell)|.
\]
Using~\eqref{eq:I_gamma001}, \eqref{eq:formula_K_00},
\eqref{eq:bound_kernel_3} and  the bounds~\eqref{eq:ab-controle_2}
on $\alpha$ and ${\coeff}$ from Lemma \ref{lem:ab-controle_2}, we get:
\[
|\eta_{\alpha,\coeff}(\theta)| \leq \frac{u_T'(s)}{1-2u_T'(s)} \left (
1+ u_T'(s) \right) +\frac{1-u_T'(s)}{1-2u_T'(s)} \left( L_{1,0} + \mathcal{V}_T +
u_T'(s) \right ). 
\]
By assumption, we have $u_T'(s) \leq 1/6$, and thus
$\frac{1}{1-2u_T'(s)}\leq 3/2$. Since  $\mathcal{V}_T \leq 1$, we obtain:
\[
|\eta_{\alpha,\coeff}(\theta)| \leq  \frac{5}{4}L_{1,0} +
\frac{7}{4}\cdot
\]
\medskip

\textbf{Proof    of  $\ref{it:as2-<c}$    from    Assumption~\ref{assumption2}} \Put $c_ B=2 $.  Using  very similar arguments as in the
proof of~\eqref{eq:majo-pab}  (taking care that  the upper bound  of the
$\ell_\infty$     norm    of     $\alpha$    and     $\coeff$    are     given
by~\eqref{eq:ab-controle_2})            we           also            get
$\norm{p_{\alpha,\coeff}}_{T}\leq 2 \sqrt{s}$.  \medskip

We proved that  $\ref{it:as2-ordre=2}$-$\ref{it:as2-<1-c} $ from  Assumption~\ref{assumption2} stand for any $\theta \in \Theta_T$. Hence  Assumption~\ref{assumption2} holds for  any positive $r$  such that for all
$\theta\neq \theta'\in \cq^{\star}: \dT(\theta, \theta')> 2 \, r$.

This finishes the proof of Proposition \ref{prop:certificat2}.

\section{Auxiliary Lemmas} \label{app:C}

We recall in the next section some basic results on  the Fréchet derivative
and the Bochner integral. Then, we provide the proofs of the intermediate results. 

\subsection{The Fréchet derivative and the Bochner integral} 
The Fréchet derivative and Bochner integrals are defined for Banach space 
valued functions, but we shall only consider the case of Hilbert space valued
functions.

Let $(H, \langle \cdot, \cdot \rangle)$ be an Hilbert space and let
$\Theta$ be an interval of $\R$. We note $\norm{\cdot}$ the norm associated to the scalar product. A function $f$ from $\Theta$ to $H$ is  
Fréchet differentiable at  $\theta\in \Theta$ if it is continuous at
$\theta$ and there exists an
element $\partial_\theta f\in H$ such that:
\[
\underset{h \rightarrow 0;\, \theta+h\in \Theta}{\lim} \quad  \norm{\frac{f(\theta + h) -
		f(\theta)}{h}  - \partial_{\theta}f(\theta) }=0 .
\]
The       derivative      of       $f$       is      the       function
$\partial_\theta f: \theta \mapsto \partial_{\theta}f(\theta) $ defined
on  $\Theta$  when  it  exists.   We  also  define  by  recurrence  the
derivative $\partial_{\theta}^{i}  f$ of order  $i \in \N^*$ of  $f$ as
the derivative  of $\partial^{i-1}_\theta f$, with  the convention that
$\partial^0 _\theta f=f$,  and say that $f$ is of  class $\cc^i$ if the
derivatives  $\partial_{\theta}^{j}  f$  exist and  are  continuous  on
$\Theta$  for $j\in  \{0, \ldots,  i\}$.  The  standard differentiating
rules for composition, addition and multiplication apply to the Fréchet
derivative. We refer to \cite{lang}  for a complete presentation of the
subject.     By   definition,    if    $f$    is   differentiable    at
$\theta\in \Theta$, then we have for all $g\in H$ that:
\begin{equation}
\label{eq:deriv}
\partial_{\theta}\left \langle f(\theta),  g\right \rangle =
\left \langle \partial_\theta f(\theta),  g \right \rangle.
\end{equation}

\medskip

The  Bochner  integral  extends  the Lebesgue  integral.   We  refer  to
\cite[Chapter~1]{arendt_book} and
\cite[Section~11.8]{aliprantis}  for  further  details  on  the  Bochner
integral.  We  endow the interval  $\Theta\subset \R$ with  its usual
Borel  sigma field  inherited from  the Borel  sigma field  on $\R$ and a measure $\mu$.   A
function $f$  from $\Theta$ to $H$  is strongly measurable if  it is the
limit      of     simple      functions     or      equivalently,     see
\cite[Lemma~11.37]{aliprantis},           if           the           map
$\theta  \mapsto \langle  f(\theta), g  \rangle$ is  measurable for  all
$g\in H$ and $f(\theta)$ lies for $\mu$-almost every $\theta \in \Theta$
in a  closed separable subspace of  $H$.  In particular if  the function
$f$   is    continuous,   then   it   is    strongly   measurable,   see
\cite[Corollary~1.1.2]{arendt_book}.   If  $f$ is  strongly  measurable,
then the norm $\norm{f}$ is a measurable function from $\Theta$ to $\R$,
see  \cite[Lemma~11.39]{aliprantis}.  Then  a  function  $f$ defined  on
$\Theta$ (endowed  with the Lebesgue  measure) is Bochner  integrable if
and only if  it is strongly measurable and if  $\norm{f}$ is integrable;
in              which             case,              we             have
$\norm{\int  f(\theta)\, \rd  \theta} \leq  \int\norm{ f(\theta)}\,  \rd
\theta$, see \cite[Theorem~11.44]{aliprantis}  (which is easily extended
from   finite    measure   to   $\sigma$-finite   measure,    see   also
\cite[Theorem~1.1.4]{arendt_book}  in this  direction).  We  remark that
the fundamental  theorem of calculus  is still valid in  this framework,
see  \cite[Proposition 1.2.2]{arendt_book}.  In  particular,  if $f$  is
continuous and Bochner integrable  on $\Theta$ and $\theta_0\in \Theta$,
then, we have:
\begin{equation}
\label{eq:deriv-prim}
F'(\theta)=f(\theta)
\quad\text{where}\quad
F(\theta)=\int_{\theta_0}^\theta f(q) \, \rd q.
\end{equation}
As  a particular case of  \cite[Lemma~11.45]{aliprantis}, if
$f$ is  Bochner integrable on $\Theta$,  then for all $g\in  H$, we have
that:
\begin{equation}
\label{eq:integ}
\int_\Theta \langle f(\theta), g \rangle \, \rd \theta=  \langle
\int_\Theta  f(\theta)\,\rd \theta, g \rangle. 
\end{equation}

\subsection{Tail bounds for suprema of Gaussian processes}
In  order to  prove Theorems  \ref{maintheorem} and  \ref{th:bounds}, we
provide   in  Lemma~\ref{lem:bound_azais_final}   a   bound  with   high
probability   of  the   supremum  of   a  Gaussian   process  given   by
$\theta \mapsto \left  \langle w_T, h(\theta) \right  \rangle_T $, where
$w_T$  is   a  noise  process   and  $h$   is  a  function  from $\Theta$, an interval  of $\R$, to the Hilbert space $(H_T, \langle \cdot, \cdot \rangle_T)$. The next lemma is
in the  spirit of \cite[Proposition~4.1]{azais2009level} (where  one assumes that 
the Gaussian  process has unitary variance); its  proof is given at  the end of
this   section   and   relies    on   Lemma~\ref{lem:azaiz-v2}.

We denote by $\mathfrak{d}_T$ the  Riemannian metric associated to the
kernel $\cK_T$, see also
Section~\ref{sec:Rieman}. Recall definitions~\eqref{eq:def-cov-deriv}
and~\eqref{eq:def-tD} and set $f^{[1]}(\theta)=\tD_{1, T}[f](\theta)=
\partial_\theta f(\theta)/\sqrt{g_T(\theta)}$ with $g_T$ defined
in~\eqref{eq:formula-gT-2}.

\begin{lem}
	\label{lem:bound_azais_final}
	Let $T\in  \N$ be fixed. Suppose that Assumptions~~\ref{hyp:reg-f}
	and~\ref{hyp:g>0} 
	hold. 
	Let  $h$ be a function of class $\cc^1$ from
	$\Theta_T$ to $H_T$, with $\Theta_T$ a sub-interval of $\Theta$.  Assume there exist  finite 
	constants $C_1$ and $C_2$ such that for all $\theta \in \Theta_T$:
	\begin{equation}
	\label{eq:max-gauss}
	\norm{h(\theta)}_T \leq  C_1
	\quad \text{and} \quad
	\norm{h^{[1]} (\theta)}_T\leq C_2.  
	\end{equation}
	Let $w_T$  be an $H_T$-valued 
	Gaussian	noise such that Assumption \ref{hyp:bruit} holds, and consider the Gaussian process
	$X=\left(X(\theta)=\langle h(\theta), w_T \rangle_T, \theta\in
	\Theta\right)$. Then,  we have for $u > 0 $:
	\begin{equation}
	\label{eq:bound_azais_final}
	\P \left ( \sup_{\theta\in \Theta_T}  |X(\theta)| \geq u
	\right ) \leq c \cdot \left ( \sigma \frac{|\Theta_T|\sqrt{ \Delta_T}}{u
	}\vee 1 \right )\,  \expp{-u^2/(4 
		\sigma^2 \Delta_T C_1^2)},
	\end{equation}
	where $|\Theta_T|$ denotes  the Riemannian  length 
	of             the  interval            $\Theta_T$             and
	$c  =  2C_2 + 1$.
\end{lem}

We first state a technical lemma.

\begin{lem}
	\label{lem:azaiz-v2}
	Let    $I\subset   \R$ be an interval.    Assume    that
	$X=(X(\theta),  \,  \theta \in I)$  is  a real  centered
	Gaussian process  with Lipschitz  sample  paths.   Then,  for
	all $u > 0$ and an arbitrary $\theta_0\in I$, we have:
	\begin{equation}
	\label{eq:lem_azais}
	\P\left(\sup_I  \,  X\geq u\right) \leq \frac{1}{ u} 
	\int_I \sqrt{\Var(X'(\theta))} \,  \expp{-u^2/(4
		\Var(X(\theta)))}  \, \rd \theta  
	+ \frac{1}{2}\expp{-u^2 /(2 \Var(X(\theta_0)))}.
	\end{equation} 
\end{lem}

\begin{proof}
	We  first  start  with  a  general remark  on  Lipschitz  functions  on
	$\R$.  Let $f$  be  a  real-valued Lipschitz  function  defined on  an
	interval      $I\subset      \R$.      Let     $b>a$      and      set
	$f_{a,b}=  \min(\max(f,a),   b)$.  The  function  $f_{a,b}$   is  also
	Lipschitz and, thanks to \cite[Theorem 3.3 p107]{eg:meas-theo}, we get
	that               $f'_{a,b}=f'=0$               a.e.               on
	$\{x\in   I\,  \colon\,   f(x)=  a   \text{   or  }   b\}$  and   thus
	$f'_{a,b}= f'\, \ind_{\{f\in  (a,b)\}}$ a.e. on $I$. We
	deduce that:
	\[
	\sup f_{a,b} - \inf f_{a,b}\leq  \int _I
	|f'_{a,b}(x)|\, \rd x=  \int _I |f'(x)|\, \ind_{\{f(x)\in (a,b)\}}\,
	\rd x. 
	\]
	Using this inequality,  we obtain that  for any $x_0\in I$:
	\begin{equation}
	\label{eq:ineq-Lip}
	\int _a^b \ind_{\{\sup_I f>t\}} \, \rd t=\int _a^b \ind_{\{\sup_I
		f_{a,b}>t\}} \, \rd t=\sup f_{a,b} - a\leq  (b-a) \ind_{\{f(x_0)\geq a
		\}}+  \int _I |f'(x)|\, \ind_{\{f(x)\in (a,b)\}}\,
	\rd x . 
	\end{equation}  
	Then, applying  Inequality~\eqref{eq:ineq-Lip} to the function $X$ and
	taking the expectation, we get, with $M=\sup _I X$, $a=u>0$,
	$b=u+\varepsilon$, $\varepsilon>0$ and $x_0=\theta_0$:
	\begin{equation}
	\label{eq:majo-intM}
	\int_u^{u+\varepsilon} \P(M\geq  t)\, \rd t\leq  \varepsilon
	\P(X(\theta_0) \geq u) + \int_I
	\E\left[ |X'(\theta)|\ind_{\{u<X(\theta)<u+\varepsilon\}}\right] \, \rd \theta.
	\end{equation}

	The random variable $X(\theta_0) $ is a centered Gaussian variable  and therefore we have:
	\begin{equation}
	\label{eq:majo-intM-X0}
	\P\left ( X(\theta_0) \geq u\right ) =
	\int_{u}^{+\infty}\frac{\expp{-x^2/(2
			\Var  (X(\theta_0)))}}{\sqrt{2\pi
			\Var(X(\theta_0))}}  \, \rd x
	\leq   \frac{1}{2} \expp{-u^2 /2 \Var(X(\theta_0))},
	\end{equation}
	where we used for the inequality that  $\int_u^{+\infty}\expp{-t^2} \rd t\leq
	\frac{\sqrt{\pi}}{2}\, \expp{-u^2}$ holds for  $u>0$, see
	\cite[Formula~7.1.13]{abramowitz1964handbook}. Notice that
	\eqref{eq:majo-intM-X0} trivially holds if $\Var(X(\theta_0))=0$ as
	$u>0$. 
	
	We now give a bound of the second term in the right hand-side
	of~\eqref{eq:majo-intM}. Since $(X',X)$ is also a Gaussian process, we can write:
	\[
	X'(\theta) = \alpha_\theta X(\theta)+  \beta_\theta G,
	\]
	where $G$ is a standard Gaussian random variable independent of
	$X(\theta)$ and:
	\[
	\alpha_\theta = \frac{\E[X'(\theta)X(\theta)] }{\Var(X(\theta))}
	\quad \text{and} \quad
	\beta_{\theta}^2= \Var(X'(\theta))  - \alpha_\theta^2 \Var(X(\theta))
	,
	\]
	with the convention that $\alpha_\theta=0$ if $\Var(X(\theta))=0$. 
	We get $  |X'(\theta)|\leq   |\alpha_\theta X(\theta)|+ |\beta_\theta|\,
	|G|$. 
	Since $G$ is independent of $X(\theta)$ and $u>0$, we deduce that:
	\[
	\E\left[  |X'(\theta)|\ind_{\{u<X(\theta)<u+\varepsilon\}}
	\right]\leq \left(|\alpha _\theta| (u+\varepsilon)  +
	\sqrt{\frac{2}{\pi}}\,  |\beta_\theta|\right) \,
	\P(u<X(\theta)<u+\varepsilon).
	\]
	
	Letting $\varepsilon$ goes to 0 in \eqref{eq:majo-intM},  using
	\eqref{eq:majo-intM-X0} the right continuity of the cdf of $M$ and the
	monotonicity of the density $ p_{X(\theta)}(u)$ of the law of $X(\theta)$, we deduce that:
	\begin{equation}
	\label{eq:ineq-majoM2}
	\P(M\geq  u)\leq   \frac{1}{2} \expp{-u^2 /2 \Var(X(\theta_0))}
	+ \int _I \left(|\alpha_\theta| u + \sqrt{\frac{2}{\pi}}\,
	|\beta_\theta|\right)\,  p_{X(\theta)}(u) \, \rd
	\theta,
	\end{equation}
	where  by  convention $  p_{X(\theta)}(u)$  is  taken  equal to  $0$  if
	$\Var(X(\theta))=0$.  We  now bound  the second  term of  the right-hand
	side of \eqref{eq:ineq-majoM2} in two steps.
	Using                     that
	$        \beta_{\theta}^2        \leq       \Var(X'(\theta))$        and
	the inequality 
	$\expp{-x^2} \leq \expp{-x^2/2}/ \sqrt{2}\, x$ for $x>0$, we get that:
	\begin{equation}
	\label{eq:bound_N1}
	\sqrt{\frac{2}{\pi}}\, |\beta_\theta| \,  p_{X(\theta)}(u)
	\leq \inv{\pi}\,   \frac{\sqrt{\Var(X'(\theta))  }}{u } \expp{-u^2/4 \Var(X(\theta))} .
	\end{equation}
	Thanks to the Cauchy-Schwarz inequality, we get $|\alpha_\theta|\leq
	\sqrt{\Var(X'(\theta))}/  \sqrt{\Var(X(\theta))} $. Using also 
	the inequality  $ \expp{-x^2} \leq  3\expp{-x^2/2}/ 4 x^2$
	for $x>0$, we get that:
	\begin{equation}
	\label{eq:bound_N2}
	|\alpha_\theta| u \,  p_{X(\theta)}(u)
	\leq   \frac{3}{4}\sqrt{\frac{2}{\pi}}\, \frac{\sqrt{\Var(X'(\theta))  }}{ u } \expp{-u^2/4 \Var(X(\theta))} .
	\end{equation}
	
	Notice  that \eqref{eq:bound_N1}  and \eqref{eq:bound_N2}  hold also  if
	$\Var(X(\theta))=0$. Using that $\frac{3}{4}\sqrt{\frac{2}{\pi}} +
	\frac{1}{\pi}\simeq 0.92\leq 1$, we    deduce   \eqref{eq:lem_azais}   from
	\eqref{eq:ineq-majoM2}, \eqref{eq:bound_N1} and \eqref{eq:bound_N2}.
\end{proof}

\begin{proof}[Proof of Lemma~\ref{lem:bound_azais_final}]
	We first consider the case $\Theta_T=[\theta_0, \theta_1]$ and let
	$\gamma: [0, 1] \rightarrow [\theta_0, \theta_1]$ be a minimizing  path
	with respect to the Riemannian metric $\mathfrak{d}_T$ (see
	Remark~\ref{rem:Rieman-gen}); in particular we have $|\gamma'(s)|
	\sqrt{g_T(\gamma(s))} = \mathfrak{d}_T(\theta_0, \theta_1)$.    
	Thanks to~\eqref{eq:deriv}, the Gaussian
	process $\tilde X=(\tilde X(s)=X(\gamma(s)), s\in [0,1])$ is of
	class $\cc^1$ on   $s\in [0, 1]$, with derivative  $\tilde X'(s) = \gamma'(s)\, X'(\gamma(s))= \gamma'(s)\, \langle \partial_\theta h
	(\gamma(s)), w_T \rangle_T$. Then, according to Lemma
	\ref{lem:azaiz-v2}, Inequality~\eqref{eq:lem_azais} holds. 
	By Assumption  \ref{hyp:bruit}, we have for all $\theta\in \Theta_T$:
	\[
	\Var(X(\theta))) \leq \sigma^2 \Delta_T \norm{h(\theta)}_T^2 \leq
	\sigma^2 \Delta_T C_1^2
	\quad \text{and} \quad
	\frac{\Var(X'(\theta)))}{g_T(\theta)} \leq \sigma^2 \Delta_T \norm{
		h^{[1]} (\theta)}_T^2 \leq \sigma^2 \Delta_T C_2^2.
	\]
	Plugging those  bounds in Inequality~\eqref{eq:lem_azais} with  $|\gamma'(s)|
	\sqrt{g_T(\gamma(s))} = \mathfrak{d}_T(\theta_0, \theta_1)$,  we obtain:
	\begin{align*}
	\P \left ( \sup_{[\theta_0, \theta_1]} X  \geq u \right )
	&\leq  \frac{1}{u}
	\sqrt{\sigma^2 \Delta_T }C_2 \,  \expp{-u^2/(4
		\sigma^2 \Delta_T C_1^2)} 	\int_0^1 |\gamma'(s)| \sqrt{g_T(\gamma(s))}  \, \rd s + \frac{1}{2} \expp{-u^2 /(2
		\sigma^2 \Delta_T C_1^2)} \\
	&\leq \left ( C_2 + \frac{1}{2}\right) \left ( \sigma
	\frac{\mathfrak{d}_T(\theta_0, \theta_1)\sqrt{ \Delta_T}}{u }\vee 1 \right )\,  \expp{-u^2/(4 
		\sigma^2 \Delta_T C_1^2)}  .
	\end{align*}
	Since $\P \left ( \sup_{[\theta_0, \theta_1]} |X| \geq u \right ) \leq 2
	\, \P \left ( \sup_{[\theta_0, \theta_1]} X \geq u \right )$, we obtain that~\eqref{eq:bound_azais_final} holds for $\Theta_T$ a bounded closed interval. Then, use monotone convergence and the continuity of $X$ to get~\eqref{eq:bound_azais_final} for any interval $\Theta_T$.   
\end{proof}

\subsection{Schur complement}
The following Lemma is a classical result on the Schur complement.
\begin{lem}[Schur complement]
	\label{schur}
	Let  $M  \in  \R^{n  \times  n}$   be  a  matrix  composed  of  blocks
	$A  \in  \R^{(n-k)  \times  (n-k)}$,  $B  \in  \R^{(n-k)  \times  k}$,
	$C \in \R^{k \times (n-k)}$, $D \in \R^{k \times k}$:
	\[
	M = \begin{pmatrix} A & B \\
	C & D\end{pmatrix}
	\]
	Assume that $D$  and $S_1 = A-BD^{-1}C$ are non singular. Then,  the system:
	\begin{equation}
	\label{eq:system_schur}
	M  \begin{pmatrix}x\\y\end{pmatrix} = \begin{pmatrix}a\\b\end{pmatrix}.
	\end{equation}
	with $x \in  \R^{n-k}$, $y \in \R^{k}$, $a \in  \R^{n-k}$ and $b \in
	\R^{k}$, has a unique solution given by:
	\[
	x = S_1^{-1}a - S_1^{-1}B D^{-1}b
	\quad\text{and}\quad
	y = D^{-1}b - D^{-1}CS_1^{-1}a + D^{-1}CS_1^{-1}B D^{-1}b.
	\]
\end{lem}


\subsection{Proofs of Lemmas in Section~\ref{sec:riemannian_metric}}
\label{sec:proof-sec4}

\begin{proof}[Proof of Lemma~\ref{lemma:expansion}]
	For simplicity,  we remove the  subscript $\cK$ and for  example write
	$f^{[1]}=\tD_1[f]=D_1[f]/\sqrt{g}$.    Recall
	that $G$, a  primitive of  $\sqrt{g}$, is continuous
	increasing  and thus  induces a  one-to-one map  from $\Theta$  to its
	image. Following  Remark~\ref{rem:Rieman-gen}, we consider the
	minimizing   path   $\gamma:   [0,   1]  \rightarrow   \Theta$   from
	$\theta_0$  to   $\theta$  defined   by
	$\gamma_s=G^{-1}(as+b)$, with
	$b=G(\theta_0)$ and $a=G(\theta) -  G(\theta_0)$. Thus, we have 
	$  \mathcal{L}(\gamma)=\mathfrak{d}(\theta,\theta_0) $.   The
	minimizing  path 
	from  $\theta_0$  to  $\theta$  has   constant  speed  thus  equal  to
	$\mathfrak{d}(\theta_0,  \theta)$.  From  the explicit  expression  of
	$\gamma$, we get in fact  that $\dot{\gamma}_t \sqrt{g(\gamma_t)} = A$
	for             $t\in            [0,             1]$,            where
	$A=\operatorname{sign}(\theta-\theta_0)\,
	\mathfrak{d}(\theta,\theta_0)$.  Thus, we have:
	\begin{equation}
	\label{eq:deriv-1}
	f(\theta)- f(\theta_0)
	= f(\gamma_1) - f(\gamma_{0}) 
	= \int_{0}^{1} \dot{\gamma}_t\, f'(\gamma_t) \ \rd t
	= A \int_{0}^{1} \tD_{1}[f](\gamma_t) \ \rd t
	= A \int_{0}^{1} f^{[1] }(\gamma_t) \ \rd t,
	\end{equation}
	where we used~\eqref{eq:deriv-prim} and that the derivative of $f\circ
	\gamma_t$ is $\dot{\gamma}_t\, f'\circ \gamma_t$ 
	for the second equality and the definition of $\tD_{1}[f]$ as well as
	the equality  $\dot{\gamma}_t \sqrt{g(\gamma_t)}  = A$ for the
	last. 
	\medskip
	
	Using~\eqref{eq:deriv-1}  for $f$ and $\theta$ replaced by $f^{[1]}$ and
	$\gamma(t)$ for some $t\in [0, 1]$, we get thanks
	to~\eqref{eq:tDi} that:
	\[
	f^{[1] }(\gamma_t)
	= f^{[1]  }(\theta_0) + \tilde A \int_{0}^{1}
	f^{[2]}(\tilde \gamma_s) \ \rd s,
	\]
	where   $\tilde    \gamma$   is   a   geodesic    from   $\theta_0$   to
	$\gamma_{t}$ and $\tilde A=\dot{\tilde \gamma}_s \sqrt{g(\tilde \gamma_s)} $.
	Since  $\gamma$ is  itself  a  geodesic, we  deduce  that
	$\tilde \gamma_{s}=\gamma_{st}$, and thus $\tilde A=tA$. Plugging this
	in \eqref{eq:deriv-1}, we get:
	\[
	f(\theta)- f(\theta_0)
	= A\, f^{[1]  }(\theta_0)
	+ A^2 \int_{[0, 1]^2}
	f^{[2]}( \gamma_{st}) \   t\, \rd t \,\rd s  = A\, f^{[1]  }(\theta_0)
	+ A^2 \int_{0}^{1} (1-r)\, 
	f^{[2]} ( \gamma_{r}) \ \rd r.     
	\]
	This gives \eqref{eq:expansion}.
\end{proof}

\begin{proof}[Proof of Lemma~\ref{lem:gT_consistent}]
	Recall that by Assumption \ref{hyp:g>0} the function $\phi_T$ is
	$\mathcal{C}^3$. According to~\eqref{eq:deriv}, we have that  for any $i,j \in
	\left\{0,\ldots, 3\right\}$ 	and any $\theta,\theta' \in \Theta$: 
	\begin{equation}
	\label{eq:intervertion_kernel}
	\partial_{\theta,\theta'}^{i,j} \left \langle \phi_T(\theta),
	\phi_T(\theta')\right \rangle_T =  \left \langle
	\partial_\theta^i \phi_T(\theta), \partial_{\theta'}^j
	\phi_T(\theta')\right \rangle_T. 
	\end{equation}
	This and~\eqref{eq:def-cov-deriv}, \eqref{eq:def-tD},
	\eqref{eq:def-cov-deriv-2} and~\eqref{eq:def-tD2} readily
	imply~\eqref{def:derivatives_kernel}.       The      first  equality
	of~\eqref{eq:formula_K_00}    comes from 	Cauchy-Schwarz's
	inequality. The second is   clear.  
	We also have:
	\begin{equation}
	\label{eq:deriv-1phi}
	\langle  \partial_\theta \phi_T(\theta),
	\phi_T(\theta)  \rangle_T=\inv{2}\, \partial_\theta  \norm{\phi_T(\theta)}^2=0
	\end{equation}
	Since the right hand-side is also equal to $\sqrt{g_T(\theta)}\,
	\cK_T^{[1,0]}(\theta,\theta)$ thanks to~\eqref{def:derivatives_kernel},
	we get the third    equality
	of~\eqref{eq:formula_K_00}. 
	Taking the derivative with respect to $\theta$ in~\eqref{eq:deriv-1phi}
	yields $g_T(\theta)=  \langle
	\partial_\theta\phi_T(\theta),  \partial_\theta \phi_T(\theta)\rangle= -
	\langle 
	\partial^2_\theta\phi_T(\theta),    \phi_T(\theta)\rangle$.
	Thanks to~\eqref{eq:deriv1-3}, we get 
	$\partial_\theta^2 \phi_T= g_T \tD_{2, T}[ \phi_T] + (1/2 g_T) g'_T
	\partial_\theta \phi_T$. Using~\eqref{def:derivatives_kernel}    and~\eqref{eq:deriv-1phi} again,  we
	deduce that $ \langle 
	\partial^2_\theta\phi_T(\theta),    \phi_T(\theta)\rangle=
	g_T(\theta)\, \cK_T^{[2,0]}(\theta,\theta)$.  This gives
	the  fourth equality  of~\eqref{eq:formula_K_00}.  Eventually,  we deduce
	from~\eqref{def:derivatives_kernel}, \eqref{eq:deriv1-3}
	and~\eqref{eq:def-tD} that:
	\[
	g_T(\theta)^{3/2}\,  \cK_T^{[2, 1]}(\theta, \theta)=\langle \partial^2_\theta
	\phi_T(\theta), \partial_\theta \phi_T(\theta) \rangle -
	\inv{2}\frac{g'_T(\theta)}{g_T(\theta)} \langle \partial_\theta
	\phi_T(\theta), \partial_\theta \phi_T(\theta) \rangle.
	\]
	Then, use that $g'_T(\theta)=2 \langle \partial^2_\theta
	\phi_T(\theta), \partial_\theta \phi_T(\theta) \rangle$ to deduce that
	$\cK_T^{[2, 1]}(\theta, \theta)=0$. 
\end{proof}

\subsection{Control on $f$ from its derivatives$f^{[2]}$} 
The proof of the next lemma is similar to the proof of
\cite[Lemma~2]{poon2018geometry} and is left to the reader. 
Recall from~\eqref{eq:formula_K_00}
that $\cK_T^{[2,0]}(\theta,\theta)=-1$ on $\Theta$. 
\begin{lem}
	\label{quadratic_decay}
	Suppose Assumptions~\ref{hyp:reg-f} and~\ref{hyp:g>0} on the dictionary hold. Let $f$ be   a
	real valued function defined on an interval $\Theta$ of class $\cc^2$. Let
	$\theta_0 \in \Theta$. Set for $i=1, 2$, $f^{[i]} =
	\tilde{D}_{i;T}[f]$ (see~\eqref{eq:def-tD}).
	\begin{propenum}
		\item \label{lem:decay_point_i} Assume $f(\theta_{0})  =0$,
		$f^{[1]}(\theta_{0}) =0 $ and that there exist $\delta > 0$ and $r>0$
		such that for any $\theta \in \mathcal{B}_T(\theta_{0},r)$:
		\begin{equation}
		\label{hyp:D2}
		|f^{[2]}(\theta)| \leq 2\delta.
		\end{equation}
		Then, we have $|f(\theta) | \leq \delta\,
		\mathfrak{d}_T(\theta,\theta_{0})^{2}, \text{ for any } \theta \in
		\mathcal{B}_T(\theta_{0},r)$. 
		
		\item  \label{lem:decay_point_ii}  Let  $\Theta_T\subset \Theta$  be  an
		interval and  suppose that $ L\geq \sup_{\Theta_T^2}|\cK_T^{[2,  0]}|$ is
		finite     and    there     exist    $\varepsilon     >    0$     and
		$r     \in     (0,L^{-\frac{1}{2}})$     such     that     for     any
		$\theta                \in               \mathcal{B}_T(\theta_{0},r)$,
		$  -\cK_T^{[2,0]}(\theta,\theta_{0}) \geq  \varepsilon$.  Assume  that
		$\mathcal{B}_T(\theta_{0},r)\subset                         \Theta_T$,
		$f(\theta_{0}) = v  \in \{-1;1\}$, $f^{[1]}(\theta_{0}) =0  $ and that
		there  exists   $\delta\in  (0,   \varepsilon)$  such  that   for  any
		$\theta \in \mathcal{B}_T(\theta_{0},r)$:
		\begin{equation}
		\label{hyp:D2-K}
		|f^{[2]}(\theta)- v\cK_T^{[2,0]}(\theta,\theta_{0})| \leq \delta.
		\end{equation}
		Then,                               we                              have
		$|f(\theta)|  \leq 1                           -
		\frac{(\varepsilon-\delta)}{2}\mathfrak{d}_T(\theta,\theta_{0})^{2},
		\text{ for any } \theta \in \mathcal{B}_T(\theta_{0},r)$.
	\end{propenum}
\end{lem}

\subsection{Proof of Lemma~\ref{lem:rates0}}
\label{sec:proof-lem-gauss}
We keep the notations from Section~\ref{sec:model-hyp}. 
In order to prove that the constants $c_0$, $c_1$ and $c_2$ do not depend on
the scaling factor $\sigma_0$, we shall rewrite  $\RT$ and $\DT$ defined
in~\eqref{eq:def-rho} and~\eqref{def:V_1}  using a change of scale. To
do so, we  define $\varphi^0(\theta) =
k(\cdot - \theta)$ for $\theta \in \Theta$; the grid
${t_1^0,\cdots,t_T^0}$ where $t_j^0 = t_j / \sigma_0$; the Hilbert space
$L^2(\lambda_T^0)$ with $\lambda_T^0 = \Delta_T\,\sigma_0^{-1} \sum_{j=1}^T
\delta_{t_j^0}$, endowed with its natural scalar product noted $\left
\langle \cdot , \cdot \right \rangle_{\lambda_T^0}$ and norm
$\norm{\cdot}_{\lambda_T^0}$; the parameter space $\Theta_T^0 = [a_T
(1-\epsilon)\sigma_0^{-1},b_T(1-\epsilon) \sigma_0^{-1}]$. Since the scaling
factor $\sigma_0$ is fixed, the measures $(\lambda_T^0, T\geq 2)$ converge vaguely
towards the Lebesgue measure $\lambda_\infty$  on $\R$. We shall also consider  another
kernel: 
\begin{equation*}
\cK_T^0 (\theta,\theta')= \left \langle \phi_T^0(\theta) , \phi_T^0(\theta') \right \rangle_{\lambda^0_T} \quad \text{  with } \quad \phi^0_T = \varphi^0 / \norm{\varphi^0}_{\lambda^0_T},
\end{equation*}
and the limit kernel $\cK_\infty^0(\theta,\theta') = \left \langle \phi_\infty^0(\theta) , \phi_\infty^0(\theta') \right \rangle_{\infty}$ with $\phi^0_\infty = \varphi^0 / \norm{\varphi^0}_{\infty}$.  For any $T \in  \N \cup \{+\infty\}$, the kernel $\cK_T^0$ is  of
class $\cc^{3,3}$ on $\Theta^2$ and for $i, j\in \{0, \ldots, 3\}$ and
$\theta, \theta'\in \Theta$, we have:
\begin{equation*}
\cK_T^{[i,j]}(\theta,\theta')  = \cK_T^{0[i,j]} \left (\frac{\theta}{\sigma_0},\frac{\theta'}{\sigma_0} \right) \quad \text{  and } \quad \frac{1}{\sigma_0^2} g_{\ck_T^0}\left (\frac{\theta}{\sigma_0} \right ) = g_{\ck_T}(\theta) .
\end{equation*}
We can now rewrite  $\RT$ and $\DT$ by using a change of scale and we get:
\begin{equation*}
\RT=\max \left(\sup_{\Theta_T^0} \sqrt{\frac{ g_{\ck_T^0}}{ g_{\ck_\infty^0}
}},\sup_{\Theta_T^0} \sqrt{\frac{ g_{\ck_\infty^0} }{ g_{\ck_T^0} }} \right),
\end{equation*}
and
\begin{equation*}
\DT=\max( \DT^{(1)}, \DT^{(2)})
\quad\text{with}\quad
\DT^{(1)}=\max_{i,j\in \{0, 1, 2\} }\, \sup_{(\Theta_T^0)^2} |
\cK_T^{0[i,j]} - \cK_\infty^{0[i,j]}|
\quad\text{and}\quad
\DT^{(2)}=\sup_{\Theta_T^0} |h_{\cK_{T}^0} - h_{\cK_{\infty}^0}|. 
\end{equation*}
Thus, bounding $\RT$ and $\DT$ amounts to controling the proximity between the kernels $\cK^0_T$ and $\cK_\infty^0$. \medskip

First, we provide an upper bound for any $i,j \in \{0,\cdots,3\}$ of:
\begin{equation}
\label{eq:bound_no_normed}
B_{i,j}(T)=\sup_{\theta,\theta' \in \Theta_T^0} \left | \left
\langle  \partial_\theta^{i} \varphi^0(\theta),
\partial_\theta^{j} \varphi^0(\theta') \right \rangle_{\lambda_T^0}  -
\left \langle \partial_\theta^{i} \varphi^0(\theta),
\partial_\theta^{j} \varphi^0(\theta') \right
\rangle_{\infty}\right |.   
\end{equation}
Notice that:
\[
\partial^i_\theta \partial^j_t \varphi^0(\theta, t)=(-1)^j \,  k ^{(i+j)} 	(\theta-t).
\]
In what follows, we shall use at most three derivatives in $\theta$ and one derivative in $t$, so that $i+j\leq 4$ in the above formula.  Recall the polynomials $P_i$ are defined as $k^{(i)}=P_i \, k$ and set $M= \max_{0\leq  i\leq  4} \sup |P_i|\, \sqrt{k}$. It is
elementary to get that for $\theta, \theta' \in \R$:
\[
\left | (\Delta_T / \sigma_0) \sum\limits_{k=1}^{T}
\partial_\theta^{i}\varphi^0(\theta, t_k^0)\partial_\theta^{j}
\varphi^0(\theta', t_k^0)  - \int_{a_T/\sigma_0}^{b_T/\sigma_0}
\partial_\theta^{i}\varphi^0(\theta, t)\partial_\theta^{j}
\varphi^0(\theta', t)  \, \rd t \right | \leq 4
\sqrt{\pi}\, 
\Delta_T M^2  \sigma_0^{-1}.
\]
We have for $\theta, \theta' \in \Theta_T^0$ that:
\begin{align*}
\left | \int_{\R \setminus [a_T/\sigma_0, b_T/\sigma_0]} \partial_\theta^{i}
\varphi^0(\theta, t)\partial_\theta^{j} \varphi^0(\theta', t) \, \rd t
\right |
& \leq  \left | \int_{b_T/\sigma_0}^{+\infty}  \partial_\theta^{i}
\varphi^0(\theta, t)\partial_\theta^{j} \varphi^0(\theta', t) \,  \rd t
\right |  +  \left | \int_{-\infty}^{a_T/\sigma_0}  \partial_\theta^{i}
\varphi^0(\theta, t)\partial_\theta^{j} \varphi^0(\theta', t) \,  \rd t
\right |\\
& \leq   2M^2 \int_{\epsilon
	b_T/\sigma_0}^{+\infty} k(t) \, \rd t\\
& \leq  2\sqrt{\pi}\, M^2 \expp{- \epsilon^2
	b_T^2/ 2 \sigma_0^2}, 
\end{align*}
where we used that  $2\int_u^{+\infty}\expp{-t^2} \rd t\leq
\sqrt{\pi}\, \expp{-u^2}$ for $u>0$, see formula 7.1.13 in
\cite{abramowitz1964handbook}. We deduce that:
\[
B_{i,j}(T) \leq  4
\sqrt{\pi}\, 
\Delta_T M^2\sigma_0^{-1} + 2\sqrt{\pi}\, M^2 \expp{- \epsilon^2
	b_T^2/2 \sigma_0^2}
\leq 2\sqrt{\pi}\, M^2 \gamma_T,
\]
with $\gamma_T = 2 \Delta_T\sigma_0^{-1}+ \sqrt{\pi}\, \expp{-\epsilon^2
	b_T^2/2\sigma_0^2}$.

Similar        arguments        as        above        yield        that:
\[
\sup_{\theta  \in  \Theta_T^0}  \left  |  \norm{\varphi^0(\theta)}_{\lambda_T^0}^2  -
\norm{\varphi^0(\theta)}_\infty  ^2 \right  |  \leq  \gamma_T. 
\]
so                                                                  that
$\norm{\varphi^0(\theta)}_{\lambda_T^0}^2\geq  \sqrt{\pi} - \gamma_T$  for
all   $\theta\in  \Theta_T^0$.    It   is  then   easy   to  deduce   that
$  \sup_{ \Theta_T^0}  |g_{\ck_T^0}-g_{\ck_{\infty}^0}  |$  is bounded  by  a constant  times
$\gamma_T$ when $\gamma_T$ is smaller than a universal finite constant.  Up to
taking $\gamma_T$ smaller than some universal finite constant, this
and  the  fact that  $g_{\ck_\infty^0}  =1/2$  give the  second  part
of~\eqref{eq:lem:rates0}.  Then use formulae  for the derivatives of the
kernels, see~\eqref{eq:def-KT}  and~\eqref{eq:def-tD}, to get  the first
part of~\eqref{eq:lem:rates0}.

\end{appendix}

%

\end{document}